\documentclass[11pt,english]{article}
\usepackage[T1]{fontenc}
\usepackage[utf8]{inputenc}
\usepackage{calrsfs}
\usepackage{amsmath}
\usepackage{amsthm}
\usepackage{amssymb}
\usepackage{dsfont}
\usepackage{bm}
\usepackage{graphicx}
\usepackage{booktabs} 
\usepackage{caption}
\usepackage{subcaption}
\usepackage{comment}
\usepackage{microtype} 
\usepackage{verbatim}
\usepackage[title]{appendix}
\usepackage{algorithm}
\usepackage{algpseudocode}
\usepackage{enumerate}
\usepackage{dutchcal}
\usepackage{mathtools}
\usepackage{cases}
\usepackage{natbib}
\usepackage{empheq}
\usepackage[dvipsnames]{xcolor}
\definecolor{Gray}{gray}{0.94}
\usepackage[align=center,shadow=true,shadowsize=6pt,nobreak=true,framemethod=tikz,skipabove=9.5pt,skipbelow=9pt,innertopmargin=5pt,innerbottommargin=5pt,innerleftmargin=5pt,innerrightmargin=5pt,leftmargin=1.5pt,rightmargin=1.5pt]{mdframed}
\usetikzlibrary{shadows,calc,decorations.pathreplacing,arrows.meta,positioning}

\makeatletter
\numberwithin{equation}{section}
\numberwithin{figure}{section}
\theoremstyle{plain}
\newtheorem{thm}{\protect\theoremname}
\theoremstyle{definition}
\newtheorem{defn}{\protect\definitionname}
\theoremstyle{plain}
\newtheorem{lemma}{\protect\lemmaname}

\newtheorem{observation}{\protect\observationname}

\theoremstyle{remark}
\newtheorem*{rem*}{\protect\remarkname}
\newtheorem{assumption}{\protect\assumptionname}
\theoremstyle{plain}

\theoremstyle{plain}
\newtheorem{claim}{\protect\claimname}

\@ifundefined{date}{}{\date{}}
\usepackage{color}
\usepackage{colortbl}
\usepackage{hyperref}
\hypersetup{
    colorlinks,
    allcolors=[rgb]{0.3,0.1,0.9}
}

\usepackage[lined,boxed,ruled, linesnumbered, algo2e]{algorithm2e}

\usepackage[margin=1in]{geometry}

\usepackage{thmtools}
\usepackage{thm-restate}
\usepackage{amsfonts}
\usepackage{tikz}
\usepackage{enumitem}
\usepackage{wrapfig}
\allowdisplaybreaks

\makeatother

\usepackage{babel}
\providecommand{\definitionname}{Definition}
\providecommand{\lemmaname}{Lemma}
\providecommand{\observationname}{Observation}
\providecommand{\remarkname}{Remark}
\providecommand{\theoremname}{Theorem}
\providecommand{\corollaryname}{Corollary}
\providecommand{\remarkname}{Remark}
\providecommand{\assumptionname}{Assumption}
\providecommand{\conjecturename}{Conjecture}
\providecommand{\claimname}{Claim}

\begin{document}
\global\long\def\norm#1{\left\Vert #1\right\Vert }%
\global\long\def\R{\mathbb{R}}%
\global\long\def\eps{\epsilon}%
 \global\long\def\N{\mathbb{N}}%
\global\long\def\Rn{\mathbb{R}^{n}}%
\global\long\def\tr{\mathrm{Tr}}%
\global\long\def\diag{\mathrm{diag}}%
\global\long\def\Diag{\mathrm{Diag}}%
\global\long\def\C{\mathbb{C}}%
\global\long\def\conv{\mathrm{conv}}%
\global\long\def\var{\mathrm{var}}%
\global\long\def\dist{\mathrm{dist}}%
\global\long\def\E{\mathbb{E}}%
\global\long\def\vol{\mathrm{vol}}%
\global\long\def\argmax{\mathrm{argmax}}%
\global\long\def\argmin{\mathrm{argmin}}%
\global\long\def\sign{\mathrm{sign}}%
\global\long\def\bd{\mathrm{bd}}%
\global\long\def\R{\mathbb{R}}%
\global\long\def\Z{\mathbb{Z}}%
\global\long\def\ham{\mathrm{Ham}}%
\global\long\def\e#1{ \exp\left(#1\right)}%
\global\long\def\Var{\mathrm{Var}}%
\global\long\def\dint{{\displaystyle \int}}%
\global\long\def\step{\delta}%
\global\long\def\Ric{\mathrm{Ric}}%
\global\long\def\P{\mathbb{P}}%
\global\long\def\len{\text{\text{len}}}%
\global\long\def\lspan{\mathrm{span}}%
\global\long\def\KL{\mathcal{D}_{\mathrm{KL}}}%
\global\long\def\TV{\mathcal{D}_{\mathrm{TV}}}%
\newcommand{\santosh}[1]{{{\bf \color{blue}{SANTOSH: #1}}}}
\newcommand{\Xinyuan}[1]{{{\bf \color{red}{XINYUAN: #1}}}}
\newcommand{\adam}[1]{{{\bf \color{brown}{[Adam: #1]}}}}
\newcommand{\squeezeup}{\vspace{-2.5mm}}
\newcommand{\eg}{\emph{e.g.}}
\newcommand{\ie}{\emph{i.e.}}
\global\long\def\kurt{\mathrm{Kurt}}%
\global\long\def\Eop{\mathop{\mathbb{E}}}%
\global\long\def\ctext{\mathcal{c}}
\global\long\def\ltext{\mathcal{l}}
\global\long\def\size{\mathrm{s}}%
\global\long\def\Xcal{\mathcal{X}}
\global\long\def\cond{\;\middle|\;}
\newcommand{\Mod}[1]{\ (\mathrm{mod}\ #1)}
\global\long\def\Scal{\mathcal{S}}%
\global\long\def\Exp{\mathop{\mathbb{E}}\limits}%
\global\long\def\Pcal{\mathcal{P}}
\global\long\def\Ccal{\mathcal{C}}
\global\long\def\Tcal{\mathcal{T}}
\global\long\def\id{\mathcal{\text{id}}}
\global\long\def\dcal{\mathcal{d}}
\global\long\def\bp{\bar{p}}
\global\long\def\bq{\bar{q}}
\newcommand{\bvec}[1]{\mathbf{#1}}  
\newcommand{\bmat}[1]{\boldsymbol{#1}}  
\global\long\def\bfF{\mathbf{F}}
\global\long\def\bfG{\mathbf{G}}
\global\long\def\txtin{\text{in}}
\global\long\def\txtout{\text{out}}
\global\long\def\pred{\text{Pred}}
\global\long\def\barv{\bar{v}}
\global\long\def\barx{\bar{x}}
\global\long\def\bary{\bar{y}}
\global\long\def\txtcc{\text{cc}}
\global\long\def\lp{\text{LP}}
\newcommand{\ind}[1]{\mathbf{1}\left(#1\right)}
\global\long\def\Qcal{\mathcal{Q}}

\setcounter{tocdepth}{2}

\bibliographystyle{alpha}

\title{
Provable Long-Range Benefits of Next-Token Prediction
}

\author{Xinyuan Cao \\ Georgia Tech\\\texttt{xcao78@gatech.edu} \and Santosh S. Vempala \\ 
Georgia Tech
\\\texttt{vempala@gatech.edu}}

\maketitle

\begin{abstract}
Why do modern language models, trained to do well on next-word prediction, appear to generate coherent documents and capture long-range structure? Here we show that next-token prediction is provably powerful for learning longer-range structure, even with common neural network architectures.  Specifically, we prove that optimizing next-token prediction over a Recurrent Neural Network (RNN) yields a model that closely approximates the training distribution: for held-out documents sampled from the training distribution, no algorithm of bounded description length limited to examining the next $k$ tokens, for any $k$, can distinguish between $k$ consecutive tokens of such documents and $k$ tokens generated by the learned language model following the same prefix. We provide polynomial bounds (in $k$, independent of the document length) on the model size needed to achieve such $k$-token indistinguishability, offering a complexity-theoretic explanation for the long-range coherence observed in practice. 

\end{abstract}

\tableofcontents

\thispagestyle{empty} 
\setcounter{page}{0}     

\section{Introduction}
Large Language Models (LLMs) exhibit surprising learning properties across a variety of different tasks. 
Most publicly released LLMs are \textit{autoregressive}, meaning that those LLMs are trained to generate text sequentially, learning to predict a single token at a time conditioned on the preceding tokens. The training process optimizes the parameters of a model on a corpus of sequential data by minimizing the expected log loss of next-token prediction. Once trained, LLMs are used to generate the documents (strings) given prompts (prefixes) by repeatedly generating the next token.

Generating a document by repeatedly predicting the next token is \textit{statistically} equivalent to full document generation. Any distribution over documents can also be represented by the conditional distribution over the next token or by the probabilities of entire documents (see Def. \ref{def:lm}). Thus, from an information-theoretic point of view, i.e., ignoring computation, maximizing next-token log likelihood is equivalent to maximizing whole-document log likelihood (see Lemma~\ref{lemma:ntp_mle}). This equivalence further implies that a language model (LM) can be used to ``complete any thought,'' i.e., given an arbitrary text prompt, one can repeatedly sample the next tokens conditioned on previously generated tokens, progressively extending the prompt to a full document.

This equivalence no longer holds when one considers \textit{computational efficiency}. For instance, accurately completing prompts like ``The prime factors of 49990139543 are \ldots'' requires the ability to factor large integers, which is often assumed to be intractable for numbers that are the product of two large random primes. While a non-autoregressive polynomial-time generator exists for such strings \citep{kalai2003generating}, by first picking random primes, multiplying them, then outputting the corresponding document, autoregressive LLMs will fail to complete these prompts once the integers exceed a certain threshold (see Appendix \ref{sec:appendix_autoregressive} for a simple proof). 

This computational limitation, which depends on model size, training complexity, etc., potentially leads to not capturing the essential structure of the training distribution (e.g., rules of a context-free grammar \citep{deletang2022neural} or logical reasoning \citep{qian2022limitations,mirzadeh2024gsm}).  
It is remarkable that interesting long-range structure is effectively learned by autoregressive LLMs. For example, GPT4 rarely generates grammatically incorrect sentences \citep{fang2023chatgpt}. The main motivation of this paper is to investigate the long-range implications of next-token prediction. We aim to rigorously explain the effectiveness of next-token prediction for document generation.

More precisely, let us assume that we have a model with small expected log loss for next-token prediction on some training distribution. Without a priori knowing the ``structure'' of the input distribution, can we say something meaningful about the model having learned latent structure? 
As a concrete step towards this goal, here we consider next-$k$-token generation: 

\begin{center}
    {\em How accurately can a model trained on next-token prediction generate the next $k$ tokens?} 
\end{center}

Alternatively, could there exist a way to distinguish the output of the model from the training distribution by looking at $k$ output tokens at a time, given that looking at only one token gives no statistical distinction? Such distinguishers are quite natural --- e.g., looking at $k=4$ or more tokens in trigram models would easily distinguish them from training distributions of grammatically correct text.
The same idea extends to modern LLM evaluation. Many current tasks measure performance using various functions such as correctness \cite{hendrycks2020measuring, hendrycks2021measuring}, factuality \cite{zhao2023felm}, or logical coherence \cite{qi2025large}. These functions can be viewed as distinguishers --- they evaluate whether a bounded window of tokens (a given prefix and generated text for the specific task) satisfies a specific property satisfied by the training distribution.

We use the following notion of a {\em distinguisher}: a Boolean function that determines if the sequence generated from an LM satisfies some (fixed) property of the training distribution (see Def.~\ref{def:distinguisher}). This verification task is often significantly more computationally efficient than text generation.
For example, consider the generation task of the Boolean satisfiability (SAT) problem, where the prompt is a Boolean formula in Conjunctive Normal Form, and the generated tokens are an assignment of variables to make the formula true. Although finding such an assignment is NP-hard, an efficient distinguisher exists that simply verifies whether the proposed solution makes the formula true. The goal of a training algorithm is to produce an LM that cannot be distinguished from the training distribution.

The main finding of this paper (Theorem~\ref{thm:main2}) is that an LLM obtained by simply minimizing the next-token log loss generates output that is indistinguishable from the training distribution. Specifically, given the same prompts, no distinguisher of up to a given size can distinguish the learned model from the training distribution, even when examining $k$ tokens at a time. We show this result for Recurrent Neural Networks (RNNs), a standard family of neural network architectures that is natural for sequential generation tasks and computationally quite general.

Our main theorem can be viewed as a complexity-theoretic treatment of LMs and loss minimization. Yao's classical theorem relates distinguishers and next-bit predictors~\citep{yao1982theory}: the existence of a $k$-bit distinguisher, i.e., a circuit that can distinguish a given distribution from the uniform distribution over $k$-bit sequences, implies the existence of a next-bit predictor with nontrivial success probability. In the setting of LMs, we show that {\em the existence of distinguishers implies that next-token log loss can be reduced}, thereby improving next-token prediction.

We argue that even if the training loss is not zero, simply minimizing it (approximately) ensures that the output LM cannot be distinguished from the training data by any bounded-size RNN over a generation window of size $k$.

To state our main results precisely, we first introduce some definitions.

\subsection{Language models and distinguishers}
\label{sec:intro_lm_distinguisher}

We consider language models that generate strings of length $n$ over a finite alphabet $\Sigma$ of tokens. 
Let $\Scal:=\Sigma^{<n}$ be all possible prefixes. For a string $s$, we denote its length as $|s|$, and its $i$-th token as $s_i$. We denote its substrings $s_{i:j} = s_is_{i+1}\cdots s_{j-1},
s_{i:} = s_{i:|s|+1}, s_{:i}=s_{1:i}$. 

\begin{defn}[Language Model]\label{def:lm}
    An autoregressive Language Model (LM) $p$ is a function that outputs next-token probabilities of a distribution over strings, i.e., it maps a string $s$ and a token $y$ to the probability that $y$ is the next token after $s$: $p:(\Sigma \cup \{\emptyset\}) \times \Scal \rightarrow [0,1]$. We write $p(y\mid s)$ for the probability of token $y$ following the prefix $s$. 
\end{defn}

Every LM $p$ corresponds to a \textit{text distribution} $\bp:|\Sigma|^n \rightarrow [0,1]$, such that $\sum_{x\in \Sigma^n}\bp(x)=1$, that assigns probability $\bp(x)$ to each text $x=x_1x_2\cdots x_n$. In terms of conditional probability, each LM corresponds to a text distribution: $\bp(x)=\prod_{i=1}^{n} p(x_{i}\mid x_{:i})$.

Let $\Delta(\Sigma^n)$ be all LMs over $\Sigma^n$.
We assume that $p$ represents the \textit{training next-token distribution}, i.e., the distribution that generates the training documents,
and that $q$ denotes \textit{the learned LM}. Let $\bar{p}$ and $\bar{q}$ be the corresponding text distributions. The goal is to learn an LM $q$ using data drawn from $p$. 

The {\em training objective} is the standard \textit{next-token loss}, defined as 
\begin{align}\label{eq:Ldef}
    L(q) = -\Exp_{x\sim \bp}\left[\frac{1}{n}
    \sum_{i=1}^n 
    \log q(x_i \mid x_{:i})
    \right]
\end{align}

Given some prefix $s$, the task of an LM is to complete $s$ to a full document (string of length $n$) by sequentially generating new tokens, where each token is conditioned on the concatenation of $s$ and all tokens generated so far.

A distinguisher $d$ operates on a full document, and $d_i(x)$ is meant to distinguish whether the given completion of $x_{:i}$, namely $x_{i:}$ is from the underlying next-$k$-token  distribution or not.

\begin{defn}[Next-$k$-token Distinguisher]\label{def:distinguisher}
    A next-$k$-token distinguisher is a function $d:[n]\times \Sigma^n \rightarrow \{0,1\}$. We require that $d_i(x):=d(i,x)$ depends only on the prefix $x_{:i}$ and the $k$-token window $x_{i:i+k}$, i.e., $d(i,x)=d(i, x_{:i+k})$ for all $i\in[n], x\in\Sigma^n$. The \textbf{advantage} of $d$ in distinguishing between two text distributions $\bp$ and $\bq$ is defined as
    \begin{align}
    \label{eq:def_advantage}
        a(d,\bp,\bq):=\Eop\limits_{y\sim \bp}\left[
    \frac{1}{n}\sum_{i=1}^n
    \left(
    \Eop\limits_{x\sim \bq}
    \left[
    d_i(x) \mid x_{:i}=y_{:i}
    \right]
    -d_i(y)
    \right)
    \right].
    \end{align}
    Two text distributions $\bp$ and $\bq$ are $\mathbf{\eps}$-\textbf{indistinguishable }by a next-$k$-token distinguisher $d$ if its advantage's absolute value $\left|a(d,\bp,\bq)\right|\leq \eps$.
We say that $\bp$ and $\bq$ are $\varepsilon$-indistinguishable to next-$k$-token distinguishers of size ${\cal d}$ if 
\[
\sup_{d: |d|\le {\cal d}} \; \left|a(d,\bp,\bq)\right|\le \varepsilon,
\]
where the supremum ranges over all next-$k$-token distinguishers $d$ of size at most ${\cal d}$.
\end{defn}

This definition of `advantage' quantifies the distinguisher's ability to capture differences between the training distribution $\bp$ and the learned LM $\bq$. Consider any distinguisher that assigns $0$ to valid (prefix, $k$-token) pairs, in particular to such pairs from the training distribution, and assigns $1$ when some implicit pattern or rule is violated. Then, for a given prefix $y_{:i}$ sampled from the training distribution $\bp$, the inner term compares    
the distinguisher's value   $\left[ d_i(x) \mid x_{:i}=y_{:i}
    \right]$  
    on the learned LM $\bq$'s $k$-token completion,
    with $d_i(y)$ 
    the training distribution $\bp$'s $k$-token completion. 
Thus, a positive advantage measures the difference of a particular function between the learned LM and the training distribution over a $k$-token window, averaged over all prefixes from the true data.
We can assume that the advantage is nonnegative without loss of generality; if not, we can simply consider the distinguisher with all outputs complemented.

We note that being sufficiently close in KL divergence implies there is no $\eps$-distinguisher. Since a distinguisher is a binary function of a prefix and a $k$-token string, its advantage, i.e., the difference between expectations under $\bp$ and $\bq$, is the difference in probability that $\bp$ and $\bq$ assign to that (prefix, $k$-token) pair. This gap is at most the total variation distance between the prefix-conditioned distributions, and using Pinsker's Inequality, one can show the following (see Appendix~\ref{sec:appendix_indistinguish} for a rigorous derivation): for any next-$k$-token distinguisher $d$ and text distributions $\bp, \bq$,
\[
a(d,\bp,\bq)
    \leq \sqrt{
    \frac{k}{2n}\KL(\bp\|\bq)}.
\]

\subsection{Main result}
\label{sec:intro_results}
We consider the class of LMs representable by Recurrent Neural Networks (formally defined in Def.~\ref{def:rnn}), where recurrent connections allow information to persist over time steps. An RNN maintains a \textit{hidden node set}, a subset of nodes that serves as the `working memory', in that the network's future outputs can be computed from the current values of the hidden node set and the subsequent inputs, without revisiting the past input sequence.

The \textit{size} of the RNN refers to the total number of nodes, while the\textit{ hidden node set size} refers specifically to the number of nodes in the hidden node set. The \textit{RNN-time} corresponds to the number of steps during which node values are actively updated. 
For an RNN $Q$, we denote its size as $|Q|$, its hidden node set size as $|H_Q|$, and RNN-time as $\Tcal_Q$.

A priori, it is unclear whether there exists a bounded-size RNN whose output is indistinguishable from a given training distribution.
Even if an indistinguishable RNN exists, it is conceivable that finding one might be intractable. 
In practice, one simply uses Gradient Descent to minimize the next-token loss over a small set of model architectures. 
For instance, Llama 3 includes LMs with 8B, 70B, and 405B parameters \cite{dubey2024llama}, each trained separately by optimizing the next-token loss.

Our main result shows that \textit{minimizing next-token log loss yields an indistinguishable language model}. Specifically, we first select a size parameter $j_0$ uniformly from a large enough range, and define the network (RNN) size as $N_1=c_1 \cdot j_0^2$ and the hidden node set size as $H_1=c_2 \cdot j_0$ for fixed constants $c_1, c_2\in\N$. We then train an LM $q_1$ of size $N_1$ and hidden node set size $H_1$ by minimizing the next-token loss.  Next, we consider a slightly larger model with network size $N_2=c_1\cdot (j_0+1)^2$ and hidden node set size $H_2=c_2\cdot (j_0+1)$, and train an LM $q_2$ by minimizing the next-token loss. If the loss for $q_2$ does not decrease by at least $\eps^2/4k$, we output LM $q_1$. Otherwise, we repeat the process with incrementally larger sizes, until the decrease in loss is less than $\eps^2/4k$. The procedure is formally stated in Algorithm~\ref{algo:min_loss_prac}; with large probability, we only need to try two models.

\begin{restatable}[Minimizing Next-token Loss Yields an Indistinguishable LM]{thm}{thmmainminloss}
\label{thm:main2}
For any $0<\eps<1,  k,\tau,\dcal\in\N$, for alphabet size $|\Sigma|=O(1)$,
with probability at least $0.9$, by trying only two model sizes and minimizing next-token loss, we can output an LM $q$ with the following properties: 
\begin{enumerate} 
\item The model $q$ is $\eps$-indistinguishable from the training  distribution $p$ for any next-$k$-token distinguisher $d:[n]\times \Sigma^n \rightarrow\{0,1\}$ realized by an RNN of size $|d|\leq \dcal$ and RNN-time $\Tcal(d) \leq \tau$.
\item The model $q$ has size $O\left(
        \frac{k^2}{\eps^4}(\dcal+k)
        \right)$
    and RNN-time  
    $
    \tau \cdot (k2^k)^{O\left(\frac{k}{\eps^2}\right)}
    $.
\end{enumerate}

\end{restatable}
These guarantees naturally extend to arbitrary alphabet sizes, as shown in the proof in Section~\ref{sec:main_results}.
In words, the theorem says that the standard training scheme --- fix a model architecture and minimize next-token loss --- reaches a model that is guaranteed to be indistinguishable from the training distribution. The model's size is polynomial in $k$, the distinguisher window length, $\dcal$, the distinguisher size bound and $1/\eps$, the inverse of the distinguisher advantage, while remaining independent of the document length $n$; as long as our distinguisher has bounded window size, the LM remains indistinguishable regardless of the length of the generated document. This result addresses the issue of accumulated error~\cite{ranzato2015sequence, bengio2015scheduled, arora2022exposure}, which can increase with the length of the generated document due to suboptimal next-token loss. 

An important point that we emphasize here is that our result for loss minimization {\em does not require knowledge of any distinguisher --- the guarantee relies only on the existence of the distinguisher}.  

To illustrate the theorem, consider training data consisting of mathematical computations, such as multi-step arithmetic, algebraic simplifications or logical deductions, whose correctness can be checked by a small RNN and thus correct and incorrect derivations can be distinguished. In this setting, Theorem~\ref{thm:main2} implies that simply minimizing the next-token loss with a model of size cubic in $k$ and linear in the distinguisher size $\dcal$ suffices to yield an LM that can generate correct derivations (i.e., indistinguishable by any distinguisher of size $\dcal$) for any output of length at most $k$ tokens. This indicates how statistically trained next-token prediction can reliably reproduce structured multi-step reasoning. It also suggests the benefits of scaling up model sizes --- to make the learned LM indistinguishable from the training distribution on larger windows.

\paragraph{Bit Size.}
To ensure that our constructions do not implicitly use very large computations, e.g., by requiring the maintenance of numbers whose lengths grow as functions of the input size, we also bound the memory required per node in Theorem~\ref{thm:main_bit}. Specifically, we show that next-token loss minimization yields an $\eps$-indistinguishable language model of polynomial size with the number of bits maintained at each node bounded by $O\left(
        b_D + \frac{k^3}{\eps^4}
        +\frac{k}{\eps^2}\log\tau
        \right)$, where $\tau$ is the distinguisher RNN-time bound and $b_D$ is a bound on the number of bits used by distinguisher RNNs.

\subsection{Technical overview}

Our proof is based on two main conceptual ideas. The first is that, given a next-$k$-token distinguisher $d$ with some advantage for an LM $q$, one can \textit{boost} $q$, i.e., modify it, so that the KL divergence between $\bp$ and $\bq$ decreases by at least $a^2(d,\bp,\bq)n/(4k)$. We first prove this for general LMs and distinguishers in Lemma~\ref{lemma:lm_update_decrease_kl}. We then show that such a boosting step can also be done for LMs and distinguishers implemented by RNNs in Lemma~\ref{lemma:rnn_boosting_construction}. 

The second idea connects one-step boosting to indistinguishability through loss minimization. Specifically, Lemma~\ref{lemma:self_boosting_criticizer} shows that if a model's loss can be reduced via distinguisher-based boosting, then the model obtained by \textit{minimizing log loss under size constraints} must already be indistinguishable from the training distribution. We call this `self-boosting', since it is achieved by pure loss minimization, without explicit knowledge of a distinguisher.

Improving the next-token loss is equivalent to improving the KL divergence, $\KL(\bp||\bq) = \Eop\limits_{x\sim \bp} \log \frac{\bp(x)}{\bq(x)}$. 
Specifically, for any LMs $q,q'\in\Delta(\Sigma^n)$, Lemma~\ref{lemma:ntp_kl} implies that
\[
\KL(\bp \|\bq) - \KL(\bp \|\bq') = n\left(
L(q)-L(q')
\right).
\]

\paragraph{Boosting using distinguishers.}
The idea for improving a model using a distinguisher is to apply a reweighting step based on a next-$k$-token distinguisher with non-negligible advantage. The advantage of a distinguisher defined in Equation~\eqref{eq:def_advantage} is computed as the average advantage over all length-$k$ blocks. However, we cannot reweight all length-$k$ blocks simultaneously because of their overlaps. Conversely, reweighting only the single best block would yield an insufficient improvement in the KL divergence, making the number of boosting steps depend on the document length. So instead we group all length-$k$ blocks into $k$ disjoint sets based on their starting offset in $[0,k-1]$. We then identify the most advantageous set of disjoint blocks among these $k$ options, indexed by the offset $i_0^*\in[0,k-1]$, and apply an update to each length-$k$ block in that set, as illustrated in Figure~\ref{fig:lemma1_selfboost}. As we will see shortly, the resulting improvement in KL divergence allows us to bound the total number of boosting steps in the proof of the main theorem. 

\begin{restatable}[Boosted Text Distribution]{lemma}{lemmalmupdate}
\label{lemma:lm_update_decrease_kl}
Let $k\in [n]$, $\bp,\bq\in\bar{\Delta}(\Sigma^n)$, and $d$ be a next-$k$-token distinguisher with advantage $\alpha:=a(d,\bp,\bq)$. Then there exists $i_0^*\in [0,k-1]$ such that the text distribution $\bq'$ defined below satisfies:
    \[
    \KL(\bp\| \bq')\leq \KL(\bp\|\bq)-\frac{\alpha^2 n}{4k}.
    \]
The text distribution $\bq' \in \bar{\Delta}(\Sigma^n)$ is defined as: 
\[
\bq'(x)=q'(x_{:i_0^*+1}) 
\prod_{i\in I}q'(x_{i+1:i+k+1}\mid x_{:i+1}),
\]
where $I:=\{ i\in [n] \mid i \equiv i_0^* \Mod{k} \}$ and $q'$ is 
    \begin{align}
    \label{eq:q_update_k_token}
        q'(x_{:i_0^*+1}):=q(x_{:i_0^*+1}); \quad
    \forall i\in I:\quad 
    q'(x_{i+1:i+k+1}\mid x_{:i+1})\propto q(x_{i+1:i+k+1}\mid x_{:i+1})e^{-\alpha d_{i+1}(x)}.
    \end{align}
    For all $i\in I$, the normalization constants of $q'$ are:
    \[
    Z(s):=\Eop\limits_{x\sim \bq}\left[
    e^{-\alpha d_{|s|+1}(x)} \mid x_{:|s|+1}=s
    \right] \text{ with }
    q'(x_{i+1:i+k+1} \mid x_{:i+1}) := \frac{q(x_{i+1:i+k+1}\mid x_{:i+1})e^{-\alpha d_{i+1}(x)}}{Z(x_{:i+1})}.
    \]
\end{restatable}

\newlength{\jCoordLen}
\pgfmathsetlength{\jCoordLen}{1.5cm} 
\newlength{\kCoordLen}
\pgfmathsetlength{\kCoordLen}{3.0cm} 
\newlength{\startCoordLen}
\pgfmathsetlength{\startCoordLen}{0.5cm} 

\newlength{\qblockwidth}
\setlength{\qblockwidth}{\jCoordLen}
\addtolength{\qblockwidth}{-\startCoordLen} 

\newlength{\distinguisherblockwidth}
\setlength{\distinguisherblockwidth}{\kCoordLen} 

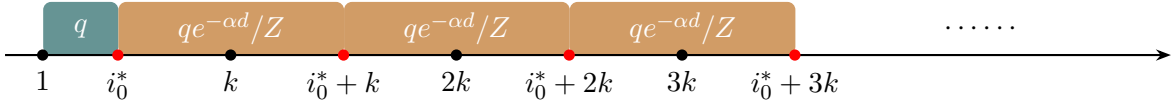
\begin{figure}[htbp] 
\centering 

\begin{tikzpicture}[
    dot/.style={circle,fill,inner sep=1.5pt},
    q_block/.style={
        rectangle,
        fill=teal!60!black!60!white,
        rounded corners=3pt,
        minimum height=0.7cm,
        minimum width=\qblockwidth,
        text=white,
        anchor=south west
    },
    distinguisher_block/.style={
        rectangle,
        fill=orange!70!black!60!white,
        rounded corners=3pt,
        minimum height=0.7cm,
        minimum width=\distinguisherblockwidth, 
        text=white,
        anchor=south west
    },
    axis/.style={thick,-{Stealth[length=2mm]}}
]

    \coordinate (start) at (\startCoordLen,0); 
    \coordinate (j) at (\jCoordLen,0);       
    \coordinate (k_point) at (\kCoordLen,0); 
    
    \coordinate (j_plus_k) at ($(j) + (\kCoordLen, 0)$); 
    \coordinate (j_plus_2k) at ($(j_plus_k) + (\kCoordLen, 0)$); 
    \coordinate (j_plus_3k) at ($(j_plus_2k) + (\kCoordLen, 0)$); 
    
    \coordinate (two_k) at (2*\kCoordLen, 0); 
    \coordinate (three_k) at (3*\kCoordLen, 0); 
    
    \node[q_block] at (start) {{$q$}};
    
    \node[distinguisher_block] at (j) {{$qe^{-\alpha d}/Z$}};
    \node[distinguisher_block] at (j_plus_k) {{$qe^{-\alpha d}/Z$}};
    \node[distinguisher_block] at (j_plus_2k) {{$qe^{-\alpha d}/Z$}};

    \node at (13, 0.35cm) {$\cdots\cdots$};

    \draw[axis] (0,0) -- (15.5,0); 

    \node[dot,label=below:{$1$}] at (start) {};
    \node[dot,label=below:{$k$}] at (k_point) {};
    \node[dot,label=below:{$2k$}] at (two_k) {};
    \node[dot,label=below:{$3k$}] at (three_k) {};
    
    \node[dot,label=below:{$i_0^*$},red] at (j) {};
    \node[dot,label=below:{$i_0^*+k$},red] at (j_plus_k) {};
    \node[dot,label=below:{$i_0^*+2k$},red] at (j_plus_2k) {};
    \node[dot,label=below:{$i_0^*+3k$}, red] at (j_plus_3k) {};
    
\end{tikzpicture}
\caption{Illustration of the boosting construction for $q'$ in Lemma~\ref{lemma:lm_update_decrease_kl}. The axis is the index of the text, ranging from $1$ to $n$.
The new model $q'$ behaves identically to the original $q$ until $i_0^*$. After that, it repeatedly applies a reweighting $qe^{-\alpha d}/Z$ over subsequent length-$k$ blocks, starting at $i_0^*+1$.
} 
\label{fig:lemma1_selfboost} 

\end{figure} 

This lemma shows that a distinguisher with non-negligible advantage can be used to improve an LM, making the updated distribution over documents closer to the true distribution in KL divergence. Its proof is inspired by \cite{alvarez2022gans}, who considered next-token prediction; we consider next-$k$-token distinguishers rather than next-token distinguishers.

The boosted full-document model in Lemma~\ref{lemma:lm_update_decrease_kl} also implies a boosted next-token model, which we describe next. Its proof is in Section~\ref{subsec:proof_lemmas78910}. 

For any two strings $s,x\in\Sigma^{*}$, we use $s\cdot x$ to denote their concatenation, and write $q(s\mid x)=\prod_{i=1}^{|s|} q(s_i\mid x\cdot s_{:i})$ for the conditional probability of the string $s$ given the prefix $x$.

\begin{restatable}[Boosted Next-token Probability]{lemma}{lemmaselfboostnexttoken}
    \label{lemma:self_boost_next_token_prob}
    Let $k\in [n]$, $p,q\in\Delta(\Sigma^n)$, and $d$ be a next-$k$-token distinguisher with advantage $\alpha:=a(d,\bp,\bq)$. Then there exists $i_0^*\in [0,k-1]$ such that the model $q'$ has next-token conditional probability satisfying:
    \[
    \KL(\bp\| \bq')\leq \KL(\bp\|\bq)-\frac{\alpha^2 n}{4k}.
    \]
    The model $q'\in \Delta(\Sigma^n)$ is defined as follows:
    \begin{align}
        \label{eq:self_boost_next_token_prob}
        q'(x_i \mid x_{:i})=\begin{cases}
            \displaystyle q(x_i \mid x_{:i}) & \text{if }i\leq i_0^*;\\[0.1in]
            \displaystyle \frac{
            \sum\limits_{s\in\Sigma^{k}}q(s\mid x_{:i_0+1})\exp\left(-\alpha d_{i_0+1}(x_{:i_0+1}\cdot s)\right)\cdot \ind{s_{:i-i_0+1}=x_{i_0+1:i+1}}
            }{
            \sum\limits_{s\in\Sigma^{k}}q(s\mid x_{:i_0+1})\exp\left(-\alpha d_{i_0+1}(x_{:i_0+1}\cdot s)\right)\cdot \ind{s_{:i-i_0}=x_{i_0+1:i}}
            }
            & \begin{aligned}
                &\text{if }i=i_0+r_0,\\
                & i_0\in I,\\
                &1\leq r_0\leq k,
            \end{aligned}
        \end{cases}
    \end{align}
    where $I:=\{ i\in [n] \mid i \equiv i_0^* \Mod{k} \}$. 
\end{restatable}

\paragraph{RNN implementation of boosting.}
Now we turn to how such a boosted model could be implemented. For this, we consider a family of sufficiently powerful model architectures and ensure that the boosted model is not much larger than the original while remaining in the same family. 
Intuitively, each boosting step incorporates information from the distinguisher and can therefore increase the model size. 
Roughly speaking, as long as the number of boosting steps is bounded and the model size grows slowly, this procedure should yield an indistinguishable LM that is not too large.

When we restrict to a natural family of architectures, in our case RNNs, we have to compute the probability update defined in Equation~\eqref{eq:self_boost_next_token_prob} with an RNN. For a prefix $x_{:i_0+1}$, we can enumerate and sum over all length-$k$ strings $s\in\Sigma^k$. The difficulty is to compute the next-token probability for each candidate $s\in\Sigma^k$, \textit{without overwriting} the RNN's internal state, which we need for each string in the enumeration.

To build the RNN that realizes $q'$, we break down the boosted next-token probability \eqref{eq:self_boost_next_token_prob} into simpler components that can be implemented by RNNs.

For any $i\in[n]$, define $i_0(i)$ as the greatest element in $I$ that is less than or equal to $i$. For any length-$k$ string $s \in \Sigma^k$, define the functions
\begin{gather*}
    \displaystyle  f_1(s,i) = q\left(s \mid x_{:i_0(i)+1}\right),\quad f_2(s,i) = \exp\left(-\alpha d_{i_0(i)+1}\left(x_{:i_{0}(i)+1} \cdot s\right)\right), \\
    \displaystyle  g_1(s,i) = \ind{s_{:i-i_0(i)+1}=x_{i_0(i)+1:i+1}},\quad
    g_2(s,i) = \ind{s_{:i-i_0(i)}=x_{i_0(i)+1:i}}.
\end{gather*}
Then the next-token probability $q'$ can be rewritten as
\begin{align}
    \label{eq:ntp_q'_fg}
    q'(x_i \mid x_{:i}) =
\begin{cases}
    \displaystyle q(x_i \mid x_{:i}) & \text{ if } i\leq i_0^*;\\
    \displaystyle \frac{\sum_{s\in \Sigma^k} f_1(s,i)f_2(s,i)g_1(s,i)}{\sum_{s\in \Sigma^k} f_1(s,i)f_2(s,i)g_2(s,i)} & \text{ if }i\geq i_0^*+1.
\end{cases}
\end{align}
The RNN implementation of $q'$ is nontrivial, as it requires \textbf{enumerating} all length-$k$ strings and \textbf{synchronizing} the computation of four functions $f_1,f_2,g_1,g_2$ across time steps. Each step in this iteration needs to use the internal state of the RNN implementing $q$ after seeing $x_{:i_0(i)+1}$. 

Our first, conceptually simpler, solution to address the internal state maintenance problem is described in Section~\ref{subsec:naive}. Here, the boosted model replicates the original model so that one copy maintains the RNN state after processing a prefix $x_{:i+1}$ and the other copy iteratively processes each length-$k$ continuation $s$. This approach roughly doubles the model size at each boosting step, leading to an \textit{exponential growth in the model size} with the number of boosting steps.

To get an efficient construction, we leverage the \textit{hidden node set} of an RNN, which captures all the state information needed to process the remaining input string.  
Instead of having two copies of the original RNN, we use two copies of the hidden node set.
The first copy maintains the state after processing a prefix $x_{:i_0(i)+1}$. The second is used as ``scratch space'' ---  we copy the state of the first hidden node set, then process a length-$k$ continuation token by token. 
The first copy of the hidden node set and the $\Theta(k)$ nodes needed for the enumeration of length-$k$ strings constitute the hidden node set of the boosted RNN (it does not double!). 
Lemma~\ref{lemma:rnn_boosting_construction} shows that in each boosting step using the distinguisher $D$, the hidden node set size increases by only $\Theta(k)+|H_D|$, and the model size increases by $\Theta(k)+|H_Q|+|H_D|+|D|$. Consequently, the hidden node set grows linearly and the model size grows quadratically with the number of boosting steps.

\begin{restatable}[RNN Boosting]{lemma}{lemmarnn}
\label{lemma:rnn_boosting_construction}
For any language model $q$ representable by an RNN $Q$, if there exists a next-$k$-token distinguisher $d_i(x)$ with advantage $\alpha$, representable by an RNN $D$, then there exists a language model $q'$, representable by an RNN $Q'$, with size $|Q'|=|Q|+|H_Q|+|D|+|H_D|+7k+25$, hidden node set size $|H_{Q'}|=|H_Q|+|H_D|+6k+17$
and RNN-time $\Tcal=\Tcal_{Q'} =(|\Sigma|^k+1)k(\max\{\Tcal_Q,\Tcal_D\}+4)$ such that
the next-token loss decreases by at least $\alpha^2/4k$, i.e., $L(q')-L(q) \leq -\alpha^2/4k$.
\end{restatable}

The proof of Lemma~\ref{lemma:rnn_boosting_construction}, in Section~\ref{sec:proof_lemma3}, builds on the structural flexibility of RNNs, enabling efficient synchronization of multiple RNNs and systematic enumeration of candidate continuations while preserving a compact hidden node set, as we specify in Section~\ref{sec:syn_enu}.

\paragraph{Self-boosting by loss minimization.}
Lemma~\ref{lemma:rnn_boosting_construction} establishes that if a model can be distinguished, it can be boosted to improve loss. However, the boosting step requires explicit knowledge of the distinguisher to build the new model and it is unclear how to find a distinguisher or determine whether one exists. To address this, we argue that it suffices to simply minimize the next-token loss. 
If the loss of the model cannot be reduced with a modest increase in its size, then there is no distinguisher (of size up to a desired bound) that has a significant advantage. 

To avoid having to know a distinguisher, we formulate a more general ``self-boosting'' principle in Lemma~\ref{lemma:self_boosting_criticizer}. This lemma abstracts the role of the distinguisher by introducing a function $c(q):\mathcal{Q}\rightarrow \R_+\bigcup\{0\}$, which measures the \textit{effective loss improvement} obtainable from boosting a model $q\in \mathcal{Q}$. It connects this generalized boosting to loss minimization: if a model can be boosted to achieve smaller loss at the cost of controlled increases in size, time and auxiliary parameters (e.g., hidden node set size), then we can construct a sequence of hyperparameter settings so that for all but a small subset of these settings, simply minimizing the loss under the hyperparameter constraints guarantees that the minimizer $\hat{q}$ satisfies $c(\hat{q})\leq \eps$ for any $c$. This implies $\eps$-indistinguishability for our setting.

\begin{restatable}[Model Self-boosting and Loss Minimization]{lemma}{lemmaselfboost}
\label{lemma:self_boosting_criticizer}
    Let $\Qcal$ be a set of models, where each $q\in\Qcal$ has 
    size $|q|\in \N$, time $\Tcal(q)\in\N$ and a value $h(q)\in \R$. There exists a $q_1 \in \Qcal$ with $|q_1|=1$, $\Tcal(q_1)=1$ and $h(q_1)=h_0$ for some $h_0\in\R$. Let $\mathcal{C}$ be a set of functions $c: \Qcal \rightarrow \R_+\cup \{0\}$ and $\beta,\gamma,\delta,\theta,\zeta \geq 0$. Suppose there is a loss function $L: \Qcal \rightarrow \R_+\cup \{0\}$ such that, for every $q \in \Qcal, c\in \Ccal$, there exists an $q' \in \Qcal$ with 
    \begin{gather*}
        |q'| \leq |q|+h(q)+\beta,\quad
        \Tcal(q')\leq \gamma\Tcal(q)+\delta,\quad
        h(q')\leq \theta h(q)+\zeta,\quad
        L(q')\leq L(q)-c(q).
    \end{gather*}
    Define three sequences $\{N_i\}_{i\geq 1},\{\Tcal_i\}_{i\geq 1}, \{H_i\}_{i\geq 1}$ as follows. 
    \begin{gather}
    \label{eq:lemma_boost_NHtau_sequence}
        N_1 = 1, N_{i+1}=N_i+H_i+\beta;\\
        \Tcal_1 = 1, \Tcal_{i+1}=\gamma \Tcal_i+\delta;\\
        H_1 = 1, H_{i+1}=\theta H_i+\zeta.
    \end{gather}
    Then there exists a set $B_\eps\subset \{N_i\}_{i\geq 1}$, such that
    for any $\eps > 0$, for all $N_j \in\{N_i\}_{i\geq 1} \setminus B_\eps$, every $\hat{q} \in\Qcal$ which minimizes $L(q)$ over $C_j:= \{q\in\Qcal \mid |q| \leq N_j,\Tcal(q) \leq \Tcal_j ,h(q)\leq H_j\}$ satisfies
        \[
        \max_{x\in \Ccal}c(\hat{q}) \leq \eps.
        \]
        Specifically, $B_\eps := \{N_j\in \{N_i\}_{i\geq 1}\mid 
        L_{j+1}<L_{j}-\eps
        \}$ for $L_j:=\min\limits_{q\in C_j} L(q)$. Also, $|B_\eps| \leq L(q_1)/\eps$.
 \end{restatable}

In the context of RNN-based LMs, Lemma~\ref{lemma:self_boosting_criticizer} can be applied with $|q|$ denoting the RNN size, $\Tcal(q)$ the RNN-time, and $h(q)$ the hidden node set size. The function $c(q)$ corresponds to the effective KL divergence reduction achievable by the next-$k$-token distinguisher. Based on Lemma~\ref{lemma:rnn_boosting_construction}, we set $c(q)=\alpha^2n/4k$ when there is a distinguisher with advantage $\alpha$. Lemma~\ref{lemma:rnn_boosting_construction} also specifies the parameters $\beta,\gamma,\delta,\theta,\zeta$ to achieve a KL improvement of $c(q)$. As a result, Lemma~\ref{lemma:self_boosting_criticizer} guarantees that, apart from a small set of hyperparameter choices (those for which the KL decreases by more than $\eps^2n/4k$ when moving to the next hyperparameter choice), minimizing the next-token loss, or equivalently minimizing KL divergence, under hyperparameter constraints yields a model $\hat{q}$ satisfying $\max_c c(\hat{q})\leq \eps^2n/4k$. Equivalently, the advantage $\alpha$ is at most $\eps$.

We then show that the number of ``bad'' hyperparameter choices is small. The algorithm increases the hyperparameters until the KL divergence reduces by less than $\eps^2n/4k$. The initial KL-divergence of the trivial one-node RNN that outputs a uniform distribution is $\KL(\bp, \bq_1)=n\log|\Sigma|$. Thus, the number of hyperparameter choices where the loss decreases by more than $\eps^2 n/4k$ is bounded by $\KL(\bp,\bq_1)/(\eps^2n/4k)=4k\log|\Sigma|/\eps^2$. 
Since we select the first set of hyperparameters uniformly from a much larger set (of size $40k\log|\Sigma|/\eps^2$), with probability at least 0.9, purely minimizing loss under the hyperparameter constraint is sufficient to find an indistinguishable model. We achieve this guarantee without ever needing to know, construct, or apply an explicit distinguisher.

The proofs of Lemma~\ref{lemma:self_boosting_criticizer} and Theorem~\ref{thm:main2} are in Section~\ref{sec:main_results}.

\paragraph{Bounded bit size.}
Our main result establishes indistinguishability guarantees with bounded model size, but does not address the sizes of numbers used in the RNN computation. Real-world neural networks have to operate with space constraints, such as bounded bit size at each node. It is conceivable that the complexity of processing long inputs is being transferred to the sizes of numbers needed for the RNN's execution. To ensure that loss minimization gives us truly bounded-space constructions, we extend our analysis to quantify how limited precision in computations affects indistinguishability. We define the bit size of an RNN as the maximum number of bits required to represent each node's value in fixed-point arithmetic. We prove that the $\eps$-indistinguishability achieved by minimizing next-token loss also holds with bounded bit size.

Theorem~\ref{thm:main_bit} shows that, with bit size polynomial in the window size $k$ and the bit size bound of the distinguisher family considered (and $1/\eps$), one can still obtain an $\eps$-indistinguishable LM. Notably, the bit size also does not depend on the document length.
The proof closely parallels that of Theorem~\ref{thm:main2}: we construct a boosted model with improved loss and controlled increases in model size and bit size, formalized in Lemma~\ref{lemma:rnn_selfboost_bitsize} --- an analog of Lemma~\ref{lemma:rnn_boosting_construction}.  Lemma~\ref{lemma:self_boosting_criticizer} can be applied again and ensures that loss minimization alone suffices to recover indistinguishability under model and bit-size constraints.

\paragraph{Future directions.}
While our main theorem provides a polynomial bound on the size of an indistinguishable RNN, the RNN-time can be exponentially high. This limitation appears to be inevitable for certain hard distributions. For example, going back to the factoring example, 
consider texts of the form `$m=p_1 p_2\cdots p_s$', where $m$ is a uniformly random $n$-bit positive integer, followed by its prime factorization in decreasing order. A polynomial-sized distinguisher can verify correctness by verifying primality of the proposed factors and multiplying them. However, the completion problem of generating the factorization given the prefix `$m=$' is believed to be computationally intractable. On the other hand, there are problems for which the ``counting'' step in the self-boosting can be done efficiently (without full enumeration, e.g., by Dynamic Programming or MCMC methods) and thus potentially result in smaller RNN-times for suitably structured classes of training distributions. 

It would also be interesting to analyze the sample complexity of training distributions for learning via loss minimization. One could derive bounds based on the VC dimension of models of bounded size; however, directly using the structure of the training distribution might lead to tighter bounds. 

\subsection{Related work}

\paragraph{Next-Token Prediction.}

Long after its inception in Shannon's pioneering work on the statistical structure of language \cite{shannon1948mathematical,shannon1951prediction}, next-token prediction has been central to the design and training of modern LLMs, which are typically trained to minimize the next-token log loss on massive training corpora \cite{brown2020language, achiam2023gpt}. This approach has been shown to be highly effective not only in natural language processing \cite{shlegeris2022language}, but also in multimodal domains \cite{zhao2023survey, wang2024emu3}.

Despite these successes, empirical studies document snowballing failure models \cite{ranzato2015sequence,bengio2015scheduled}, in which small prediction errors at each step accumulate, producing globally incoherent outputs. These phenomena are observed in compositional tasks
\cite{dziri2024faith}, path-finding \cite{momennejad2024evaluating}, arithmetic, and story generation \cite{bubeck2023sparks}, and are exacerbated by limitations of the teacher-forcing training scheme \cite{bachmann2024pitfalls} in a path-star graph. 

Motivated by these challenges, there has been growing theoretical interest in understanding the expressive power of next-token predictors implemented by RNNs and transformers. \cite{malach2023auto} showed that next-token predictors can approximate any efficiently computable function by a Turing machine. From a statistical perspective, recent work provides generalization guarantees for transformer-based language models trained with next-token prediction \cite{ligeneralization}. From a computational perspective, related work also investigates when neural architectures can implement nontrivial algorithmic computations; for example, recurrent convolutional networks have been shown to be able to represent succinct algorithms \citep{goel2022recurrent}.

\paragraph{Distinguishability.}
Distinguishability is a well-studied notion in theoretical computer science, particularly in cryptography \cite{kunzli2005distinguishing, goldreich2005foundations} and the theory of pseudorandomness \cite{yao1982theory,nisan1994hardness}. In these settings, a distinguisher is an algorithm that attempts to tell two distributions apart, and indistinguishability under a class of distinguishers defines the notion of computational equivalence between distributions. Yao's classic theorem \cite{yao1982theory} connects this definition to prediction by showing that the existence of a $k$-bit distinguisher (from the uniform distribution) implies a next-bit predictor with nontrivial success probability. Our work adopts this complexity-theoretic viewpoint, using computationally bounded distinguishers to formalize the proximity of language models.

In machine learning, distinguishers are widely used to assess and guide model behaviors in fields such as generative models \cite{goodfellow2020generative}, adversarial robustness \cite{fawzi2018analysis}, language models \cite{alvarez2022gans}, and reinforcement learning \cite{ho2016generative,ouyang2022training}. For instance, in GANs \cite{goodfellow2020generative}, a `discriminator' network is trained to distinguish between real and generated samples, thereby driving the generator to produce more realistic outputs. Similarly, in Reinforcement Learning (RL) frameworks like GAIL \cite{ho2016generative} and RLHF \cite{ouyang2022training}, a discriminator or reward model is trained to distinguish `expert' (human-preferred) outputs from the agent's generated outputs. The resulting signal --- how `distinguishable the agent's policy is --- is then used as a reward to update the policy. Closer to our setting, \cite{alvarez2022gans} gave a polynomial-time reduction from likelihood maximization to next-token distinguishability for $n$-gram models and neural networks with a softmax output layer. 
In contrast, our work focuses on the more general case of next-$k$-token distinguishability in LMs implemented by RNNs. Crucially, while previous work relies on explicitly training distinguishers, our main result (Theorem~\ref{thm:main2}) demonstrates a `self-boosting' property (Lemma~\ref{lemma:self_boosting_criticizer}), showing that simply minimizing the next-token loss is, by itself, sufficient to drive down the advantage of a distinguisher, yielding an indistinguishable model without ever needing to explicitly train or know a distinguisher.

\paragraph{Benefits of Loss Minimization.}
Loss minimization is the standard framework for training machine learning models. The idea of connecting model boosting to loss minimization or likelihood maximization originated in \cite{friedman2000additive, lebanon2001boosting}. Since then, several theoretical results have shown that loss minimization yields desirable statistical properties. For example, minimizing discrimination error is directly related to the total variation distance \cite{hashimoto2019unifying}, and minimizing loss leads to multicalibrated models \cite{hebert2018multicalibration, garg2024oracle, blasiok2023loss}. Our work builds on this foundation by showing that minimizing next-token loss leads to statistically indistinguishabe models.

\subsection{Preliminaries}
\label{sec:prelim}

\paragraph{Language Models.}

For $m\in\N$, we denote the sets of strings,
\[
\Sigma^* = \bigcup_{i=0}^\infty \Sigma^i, \quad
\Sigma^{<m}:= \bigcup_{i=0}^{m-1}\Sigma^i, \quad
\Sigma^{\leq m}:= \bigcup_{i=0}^{m}\Sigma^i 
\]

We use $s\cdot z$ to denote the concatenation of two strings $s,z\in\Sigma^*$.

Let $\bar{\Delta}(\Sigma^n)$ be all text distributions over $\Sigma^n$, i.e., joint distributions over length-$n$ strings. 
Let $\Delta(\Sigma^n)$ be all LMs over $\Sigma^n$ assigning conditional next-token probabilities. 
Every LM $q\in\Delta(\Sigma^n)$ corresponds to a \textit{text distribution} $\bq:\Sigma^n \rightarrow [0,1]$, such that $\sum_{x\in \Sigma^n}\bq(x)=1$, that assigns probability $\bq(x)$ to each string (text) $x=x_1x_2\cdots x_n$:
\[
\bq(x):= q(x_1 \mid \eps)\cdot q(x_2 \mid x_1)\cdot q(x_3 \mid x_1x_2)\cdots q(x_n \mid x_1x_2\cdots x_{n-1}),
\]
where $\eps$ is the empty string.
Conversely, one can compute the next-token probabilities using text probabilities in general. However, this requires an infinite sum (or exponentially large sum of $|\Sigma|^n$ if the texts are bounded to length $n$). Specifically, for $y\in\Sigma$ and $s\in\Scal$:
\[
q(y\mid s) = \frac{\sum_{t\in \Sigma^{n-|s|-1}}\bq(s\cdot y\cdot t)}{\sum_{t\in \Sigma^{n-|s|}} \bq(s\cdot t)}.
\]
For any LM $q$ and its corresponding text distribution $\bq$, we might use $q$ and $\bq$ interchangeably as they have a one-to-one correspondence.

Generally, for any $s\in \Sigma^{\leq n}$, we write the marginal probability
\[
\bq(s) :=\sum_{z\in \Sigma^{n-|s|}}\bq(s\cdot z).
\]
Given $s,z\in\Sigma^{\leq n}$ with $|s|+|z|\leq n$, we denote the conditional probability
\[
q(z|s) := 
\prod_{i=1}^{|z|}q(z_i\mid s\cdot z_{:i})
=
\mathbb{P}_{x\sim \bq}\left(
x_{|s|+1:|s|+|z|+1}=z\mid x_{:|s|+1}=s
\right)
=
\frac{
\sum_{z'\in \Sigma^{n-|s|-|z|}}\bq(s\cdot z\cdot z')
}{
\sum_{z'\in\Sigma^{n-|s|}
}
\bq(s\cdot z')
}.
\]

\paragraph{Loss Functions.}

We first show that minimizing the next-token loss is equivalent to maximizing the log-likelihood of the data.
\begin{lemma}[Next-token Loss and Maximum Log-Likelihood]
\label{lemma:ntp_mle}
    For a fixed distribution $\bp$, the document length times next-token loss of $\bq$ is the negative of its log likelihood. That is,
    \[
    nL(q) = -\Eop\limits_{x\sim \bp} \left[\log \bq(x)\right].
    \]
\end{lemma}
\begin{proof}
    By definition,
    \begin{align*}
        nL(q)
        =-\int \bp(x) \cdot \sum_{i=1}^n \log q(x_i \mid x_{:i})\,dx
        = -\int \bp(x) \log \bq(x)\,dx
        =-\Eop\limits_{x\sim \bp} \left[\log \bq(x)\right].
    \end{align*}
\end{proof}

To further analyze this object, we can re-express it using Shannon entropy and the KL divergence. Specifically, For a distribution $\bp$, its entropy $H(\bp)$ is defined as
    \[
    H(\bp)=-\Eop\limits_{x\sim \bp}\log \bp(x).
    \]

\begin{lemma}[Next-token Loss and KL Divergence]
\label{lemma:ntp_kl}
    For a fixed distribution $\bp$, the document length times next-token loss of $\bq$ and the KL divergence between $\bp$ and $\bq$ differ by the entropy of $\bp$. That is,
    \[
     nL(q) - \KL(\bp||\bq)=H(\bp)
    \]
\end{lemma}
\begin{proof}
    \begin{align*}
        nL(q) - \KL(\bp||\bq)
        = -\Eop\limits_{x\sim \bp} \left[\log \bq(x)\right] - \Eop\limits_{x\sim \bp} \log \frac{\bp(x)}{\bq(x)}
        = -\Eop\limits_{x\sim \bp} \left[\log \bp(x)\right]
        =H(\bp)
    \end{align*}
\end{proof}

\paragraph{Recurrent Neural Networks.}
\label{subsec:llm_nn}


A Recurrent Neural Network (RNN) is a type of neural network (NN) designed for processing sequential data, characterized by a recurrent architecture where each node maintains a value at each time step. This cyclic and stateful process contrasts with feed-forward networks, such as \textit{Multi-Layer Perceptrons} (MLPs), which process data in a single forward pass, with each node computing a value based on the values of nodes in the previous layer. An RNN achieves recurrence by sharing parameters (weights) across time steps, updating each node's value based on the values of its connected nodes from the previous step. The total number of time steps for a computation is the RNN's \textit{RNN-time}, denoted as $\Tcal$. Thus, the computation of an RNN over $\Tcal$ time steps is that of an MLP with $\Tcal$ layers, where layers share the same weights and the number of nodes in each layer is the number of RNN nodes (see Fig.~\ref{fig:def_rnn}).

An RNN processes an input stream $x_1,x_2,\cdots,x_n$ sequentially. 
The \textit{hidden node set} is a subset of nodes whose values are determined solely by the past input sequence and their own previous values. The values of the hidden node set, together with the future input sequence, are sufficient to compute the output of the RNN.

\paragraph{Notation.} We use uppercase letters to denote RNNs and sets of nodes, and lowercase letters to denote individual nodes in the network. The value of a node $v$ at time $t$ is denoted by $v^t$, and the set of values of a set of nodes $V$ at time $t$ is denoted by $V^t$. We use lowercase letters to denote individual functions, and bold lowercase letters to denote sets of functions. The \textit{incoming neighbors} of a node $v$, denoted as $N(v)$, are the set of vertices $u\in V$ with edges to $v$:
            $N(v) = \{u \in V \mid (u,v)\in E\}.$ This includes any self-loop $(v,v)$ if it exists. 

\paragraph{Transition Functions.} RNN computation is based on transition functions. Each function $f: \R^n\rightarrow \R$ maps a set of node values $\{x_1,\cdots x_n\}$ to a single node value $y$. Given input $x = (x_1,x_2,\cdots x_n)^\top \in \R^n$ and a vector $w \in \R^n$,
transition functions are compositions of a constant number of the following elementary functions:
\begin{enumerate}
\item ReLU: $\sigma(x; w) = \max(0, w^\top x)$
\item Product: $\phi(x) = \prod_{i=1}^n x_i$
\item Reciprocal: $\psi(x;w) = 1/w^\top x$ where the denominator must be nonzero.
\end{enumerate}
We show formally in Appendix~\ref{subsec:tool_lemmas} that common functions such as the indicator function and Boolean logic operations are transition functions.

\begin{figure}[t]
    \centering
    \includegraphics[width=0.9\linewidth]{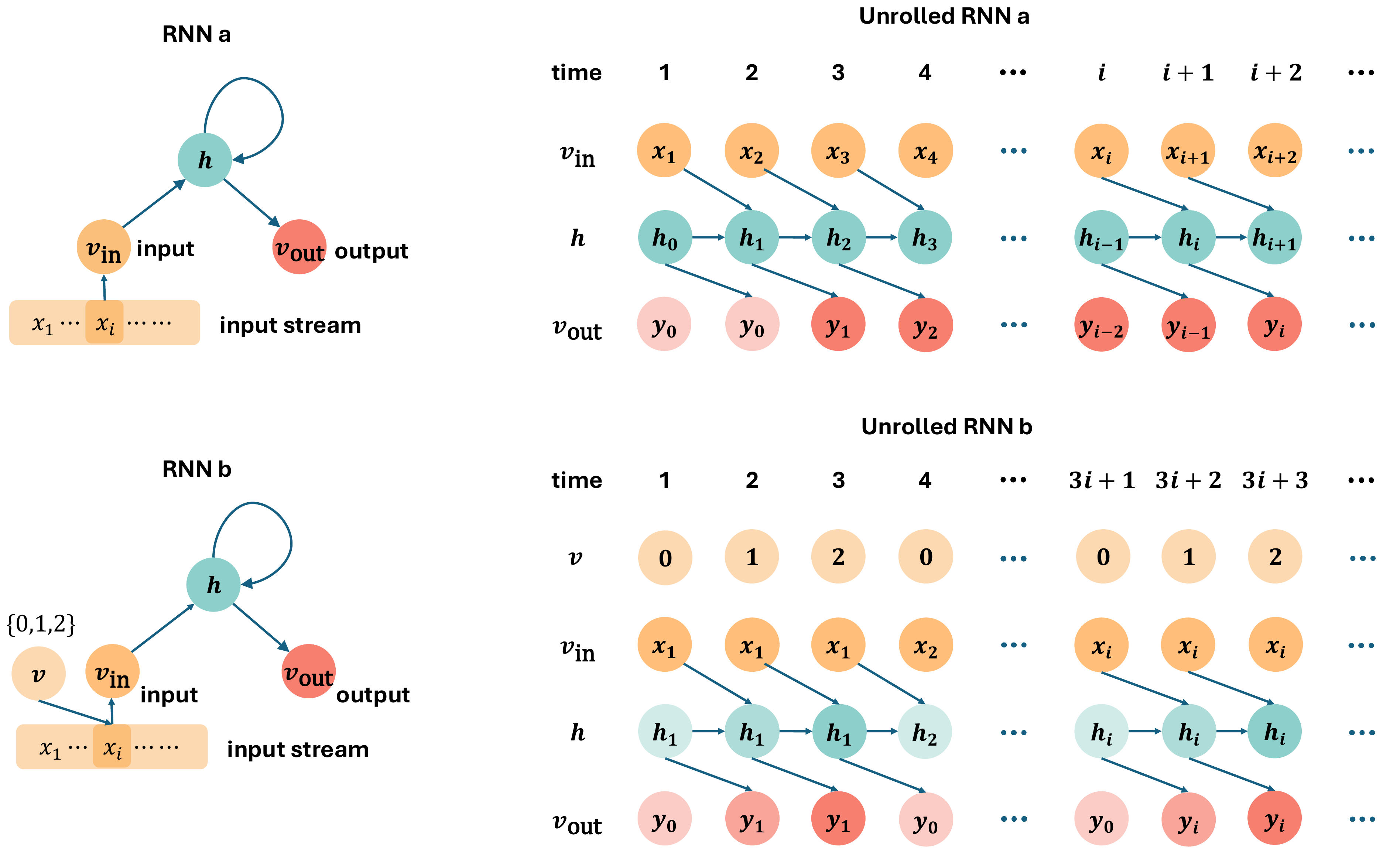}
    \caption{Two examples of RNNs, and their corresponding unrolled feedforward networks. In both RNNs, $h$ is the hidden node set. The subscript indicates the input index.
    RNN $a$ receives a new input $x_i$ at each time step $i$. Thus, the output corresponding to the input $x_{:i+1}$ is computed at time $i+2$. RNN $b$ receives a new input $x_i$ and holds it for three consecutive time steps $(t=3i,3i+1,3i+2)$. This is managed by a control node $v\in\{0,1,2\}$. Thus, the output corresponding to the input $x_{:i+1}$ is computed at time $3i+3$.
    }
    \label{fig:def_rnn}
\end{figure}

\begin{defn}[Recurrent Neural Network (RNN)]
\label{def:rnn}
    An RNN $Q$ is described by a tuple:
    \[
    \left(G_Q = \left(V_Q, E_Q\right),  V_{Q,\txtin}, v_{Q,\txtout}, \bm{f}_Q, \Tcal_Q, g_Q, \left(H_Q, \bm{f}_{Q,H}, \psi_{Q,H}\right)\right).
    \]
    \begin{enumerate}
        \item \textbf{A directed graph} $G_Q=(V_Q,E_Q)$, where each node $v\in V_Q$ is associated with a real-valued state that evolves over time. 
        \item \textbf{Size:} The size of the RNN is the number of nodes in the graph $G_Q$, denoted by $|Q| = |V_Q|$.
        \item \textbf{Bit-size}:
        The bit-size of the RNN is the maximum number of bits needed to encode the value stored in each node at any time step, denoted by $\langle Q\rangle=1+\langle Q\rangle_I+\langle Q\rangle_F$. Formally, we fix a signed fixed-point representation with one sign bit, $\langle Q\rangle_I$ integer bits, and $\langle Q\rangle_F$ fractional bits. Each real number $r$ stored in a node is represented as $r=\sign(r)\left(r_I+r_F\right)$, where $\sign(r)\in\{+1,-1\}$ is the sign of $r$, $r_I$ is its integer part in the range $[1,2^{\langle Q\rangle_I}]$, and $r_F\in [0,1)$, its fractional part, is a multiple of $2^{-\langle Q\rangle_F}$.
        \item \textbf{Input node set} $V_{Q,\txtin} \subset{V_Q}$: A subset of nodes whose values are set externally by an input sequence. The input sequence is indexed by a pointer that can be advanced by the RNN.
        \item \textbf{Output node} $v_{Q,\txtout} \in V_Q$: A designated node whose value is the output of the RNN.
        \item \textbf{Transition functions} $\bm{f}_Q = \{f_v\}_{v \in V_Q \setminus V_\txtin}$.  At each time step $t$, the value of a non-input node $v$
        is updated based on the values of its incoming neighbors at time $t-1$.
        \begin{align*}
            \label{eq:rnn_def_update_new}
            v^t = f_v(\{u^{t-1} \mid u \in N(v)\}) \text{ where }f_v\text{ is a transition function}.
        \end{align*}
        \item \textbf{RNN-time} $\Tcal_Q \in \N$: The number of time steps for which the RNN is computed, i.e., nodes are updated.
        \item \textbf{RNN function} $g_Q:\left(\R^{|V_\txtin|} \right)^*\rightarrow \R$: The function computed by the RNN, mapping the input sequence to the output node. That is, for any $t\geq 1$,
        $v^{t+\Tcal_Q}_\txtout = g_Q(V_\txtin^{1}, V_\txtin^{2},\cdots, V_\txtin^{t})$.
        \item  \textbf{Hidden node set} $(H_Q,\bm{f}_{Q,H}, \psi_{Q,H})$: A subset of nodes  $H_Q = \{h_{1},h_{2},\cdots, h_{|H_Q|}\} \subset V_Q$ defined by two properties:
            \begin{itemize}
                \item Each hidden node set value $h^t$ is computed using a transition function $f_{Q,H_j}:\R^{|V_\txtin|+|H_Q|} \rightarrow \R$ of only the input node values $V_\txtin^{t-1}$ and the hidden node set $H^{t-1}$ and not the rest of the RNN nodes. 
                \[
                h^t = f_{Q,H_j}(V_{\txtin}^{t-1}, H_Q^{t-1}), \quad \forall h \in H_Q.
                \]
                \item At any time $t$, the RNN output can be computed from just its hidden node set values at time $t$ and the remaining input sequence $V^{t:n+1}_\txtin$. That is, there exists a function $\psi_{Q,H}:\left(\R^{|V_\txtin|}\right)^*\times \R^{|H_Q|} \rightarrow \R$ such that
                \begin{align}
                    \label{eq:def_rnn_hidden_node}
                    \psi_{Q,H}( V_\txtin^{t:n+1}, H_Q^{t})= g_Q(V_\txtin^1,V_\txtin^2,\cdots,V_{\txtin}^n).
                \end{align}
            \end{itemize}
        \end{enumerate}
\end{defn}

To optimize the parameters of an RNN, the network is \textit{unrolled} for $\Tcal$ steps, effectively transforming it into a feedforward network with $\Tcal+1$ layers. This unrolled form allows for the use of standard backpropagation techniques to compute gradients and update parameters across time steps (a procedure known as backpropagation through time (BPTT) \cite{werbos2002backpropagation}).

In this work, we focus on the class of next-$k$-token distinguishers (Definition~\ref{def:distinguisher}) that are implemented by RNNs. We refer to such a model as a Distinguisher RNN. This model is an RNN that takes an index $i$ and a string $x$ as input and outputs $d_i(x) \in \{0,1\}$.

\section{Boosting an RNN}
\label{sec:proof_lemma}
The main goal of this section is to prove Lemma~\ref{lemma:rnn_boosting_construction}, which provides an \textit{efficient construction} of a boosted LM based on RNNs.

We begin with a simpler construction of the boosting RNN in Section~\ref{subsec:naive}, which results in an exponential increase in model size with the number of boosting steps (roughly doubling with each boosting). We then present a more efficient construction in Section~\ref{subsec:construction_plan}. The efficient construction relies on \textit{synchronized enumeration} within an RNN, described in Section~\ref{sec:syn_enu}. The proofs of lemmas for boosted next-token probability and its components, including Lemma~\ref{lemma:rnn_boosting_construction}, are deferred to Section~\ref{subsec:proof_lemmas78910}.

\subsection{A simple construction}
\label{subsec:naive}

To implement the boosted next-token conditional distribution in Equation~\eqref{eq:ntp_q'_fg}, we need to construct RNN components that realize $f_1,f_2,g_1,g_2$. For the indicator functions $g_1,g_2$, we use $k$ nodes to store the string $x_{i_0(i)+1:i+1}$, and another $k$ nodes to enumerate length-$k$ strings (Lemma~\ref{lemma:transition_function}). Since indicator functions are themselves transition functions by Lemma~\ref{lemma:transition_function}, we can construct an RNN of size $O(k)$ to implement $g_1$ and $g_2$.

Next, we consider the functions $f_1$ and $f_2$: 
\[
f_1(s,i) = q(s \mid x_{:i_0(i)+1}) = \prod_{r=1}^k q(s_r \mid x_{:i_0(i)+1}\cdot s_{:r}), \qquad f_2(s,i) = \exp(-\alpha d(i_0(i)+1, x_{:i_0(i)+1}\cdot s)).
\]
By Lemma~\ref{lemma:transition_function}, product and exponential are transition functions. Let $q$ and $d$ be implemented by RNNs $Q$ and $D$. For a fixed string $s\in\Sigma^k$ and index $i\in[1,n]$, an RNN can compute $f_2$. To compute $f_1$,  we can use an RNN to compute the $k$ conditional probabilities $q(s_r \mid x_{:i_0(i)+1}\cdot s_{:r})$ by sequentially processing the tokens of $s$. The core difficulty, however, is that this sequential computation must be performed for all $|\Sigma|^k$ possible strings $s\in\Sigma^k$. Let $s^{(j)}$ denote the $j$-th string in $\Sigma^k$ for $1\leq j\leq |\Sigma|^k$. 
We need to compute all conditional probabilities in Table~\ref{tab:conditional_probs}, row by row. 
\begin{table}[ht]
\centering
\begin{tabular}{c| c c c c}
\toprule
$s^{(j)}$ & $r=1$ & $r=2$ & $\cdots$ & $r=k$ \\
\midrule
$s^{(1)}$ & $q(s^{(1)}_1 \mid x_{:i_0(i)+1})$ & $q(s^{(1)}_2 \mid x_{:i_0(i)+1}\cdot s^{(1)}_{:2})$ & $\cdots$ & $q(s^{(1)}_k \mid x_{:i_0(i)+1}\cdot s^{(1)}_{:k})$ \\
$s^{(2)}$ & $q(s^{(2)}_1 \mid x_{:i_0(i)+1})$ & $q(s^{(2)}_2 \mid x_{:i_0(i)+1}\cdot s^{(2)}_{:2})$ & $\cdots$ & $q(s^{(2)}_k \mid x_{:i_0(i)+1}\cdot s^{(2)}_{:k})$ \\
$\vdots$ & $\vdots$ & $\vdots$ &$\cdots$ & $\vdots$ \\
\bottomrule
\end{tabular}
\caption{Conditional probabilities for each sequence $s^{(j)}$ and each string index $r$.}
\label{tab:conditional_probs}
\end{table}
This computation order highlights a key computational challenge: each row $j$ must be computed by processing $s^{(j)}$ independently, starting from the same prefix $x_{:i_0(i)+1}$. When switching from the string $s^{(j)}$ to $s^{(j+1)}$, the computation must restart from the state corresponding to the prefix $x_{:i_0(i)+1}$, not from a state ``contaminated'' by the suffix $s^{(j)}$. Therefore, we must preserve the hidden state of the RNN $Q$ after processing the prefix $x_{:i_0(i)+1}$ and reuse it as the starting point for each of the $|\Sigma|^k$ suffix computations.
A straightforward solution is an RNN that contains two copies of the RNN $Q$:
\begin{itemize}
    \item One copy serves as memory, storing the hidden state $\tilde{H}_Q$ after processing the prefix $x_{:i_0(i)+1}$.
    \item The other copy loads the state $\tilde{H}_Q$ and uses it to sequentially process the tokens of $s$,  computing $q\left(s_r\mid x_{:i_0(i)+1}\cdot s_{:r}\right)$ for $r=1,2,\cdots,k$.
\end{itemize}

The same strategy can be applied when computing $f_2$ for all $|\Sigma|^k$ strings. We can then sum up the products $f_1(s,i)f_2(s,i)g_1(s,i)$ and $f_1(s,i)f_2(s,i)g_2(s,i)$ over all possible strings $s\in\Sigma^k$, and their ratio yields the updated next-token conditional probability according to Equation~\eqref{eq:ntp_q'_fg}.

The constructed RNN $Q'$ has size $|Q'| = 2|Q|+2|D| + O(k)$, roughly doubling with each boosting step.

\subsection{An efficient construction}
\label{subsec:construction_plan}

There are two main hurdles to constructing an RNN more efficiently:
\begin{itemize}
    \item \textbf{Model size}. 
    To keep the model size smaller, we copy only the \textit{hidden node set} of the RNN $Q$, rather than the entire RNN. This captures the necessary state information while avoiding exponential growth in model size. By itself, this is not enough --- if we double the size of the hidden node set, the RNN size will still grow exponentially. To avoid this, we must ensure that the hidden node set of the constructed RNN $Q'$ remains sufficiently small and not double, so that the overall size of $Q'$ scales polynomially with the number of boosting steps.

    \item \textbf{Synchonization of RNN components}.  In Equation~\eqref{eq:ntp_q'_fg}, the updated distribution involves summing over all strings $s$, combining the outputs of the functions $f_1$, $f_2$, $g_1$, and $g_2$. Since these functions are computed by separate RNN modules, we must carefully synchronize their computations to ensure that all four RNNs evaluate their respective function on the same string $s$ at the same time step.

    To do this, we keep track of the maximum of the RNN-times for computing $f_1$, $f_2$, $g_1$, and $g_2$, denoted as $\Tcal$, and each component that finishes early holds its state until $\Tcal$ time steps are reached. We will use counters (Claim~\ref{claim:lem_proof_counters}) for tracking, which are also RNNs. We show in Lemma~\ref{lemma:rnn_load_run_hold} that the flexibility of RNNs enables precise, step-by-step control of when to ``load'' a new state, compute its transition function and update state, or hold its state to be the same as at the previous time step.

    The RNN receives one new token $x_i$ from the input stream every $\Tcal$ steps. We are only concerned with the output at time steps that are multiples of $\Tcal$. Specifically, we require that the output node $v_{\txtout}^{\Tcal \cdot i} = g(x_{:i+1})$ for each $i\in N$.
    We also relax the second requirement of its hidden state set $H$. We need the hidden node set sufficiency property \eqref{eq:def_rnn_hidden_node} not in all time steps, but only for the times that are multipliers of $\Tcal$. Formally, we require the existence of a function $\psi_H:\left(\R^{|V_\txtin|}\right)^* \times \R^{|H|}\rightarrow\R$ such that for any $i\in[1,n]$,
    \[
    \psi_H\left(x_{i:n+1}, H^{(\Tcal - 1)\cdot i+1}\right) = g\left(x_{:n+1}\right).
    \]
    Note that this still conforms to the definition of RNNs (Definition~\ref{def:rnn}).

\end{itemize}

\paragraph{Construction Plan.}
\begin{itemize}
    \item We construct RNNs that compute the functions $f_1(s,i)$, $f_2(s,i)$,  $g_1(s,i)$ and $g_2(s,i)$, synchronizing output for each input index $i \in [1,n]$ and each string $s \in\Sigma^k$ (Lemma~\ref{lemma:rnn_prod_q}, Lemma~\ref{lemma:rnn_exp_distinguisher} and Lemma~\ref{lemma:rnn_indicator_first_digit}).
    \item We compute $q'(x_i\mid x_{:i})$ by summing the products $f_1(s,i)f_2(s,i)g_1(s,i)$ and $f_1(s,i)f_2(s,i)g_2(s,i)$ for each string $s$, and then take the ratio of these two products.
\end{itemize}

We will show that the size of the hidden state of $Q'$ is bounded. Specifically, $|H_{Q'}| \leq |H_Q|+|H_D|+6k+17$. This will lead to an efficient construction of the RNN $Q$.
For convenience, we consider $k$-digit binary strings instead of $k$-digit base-$|\Sigma|$ strings. 
Let $z^{(j)}, 1\leq j\leq 2^k$ denote $k$-digit binary strings in increasing order.

The following lemmas compute the functions $f_1,f_2,g_1$ and $g_2$ respectively.

\begin{restatable}[Computing $k$-token Probability $f_1$]{lemma}{lemmafone}
    \label{lemma:rnn_prod_q}
    Let $n \in\N$ be the document length, $k,i_0^*\in N$ with
    $i_0^*\in[0,k-1]$.
    Let $Q$ be an RNN for language model $q$, which receives a new input token from the input stream every $\Tcal$ steps. That is, for any input stream $x_1,x_2,\cdots,x_n$, for any $1\leq i\leq n$,
    the output node $v_{Q,\txtout}^{i\Tcal_Q} = q(x_i \mid x_{:i})$.
    Let $\tau\in\N$ be an integer such that $\tau \geq \Tcal_Q+4$. Then there exists an RNN $U$ with RNN-time $\Tcal_U = (2^k+1)k\tau$ such that 
        \begin{itemize}
            \item For input index $1\leq i\leq i_0^*$, it outputs the same next-token conditional probability as $q$,
            \[
            v_{U,\txtout}^{i\Tcal_U-1} =q(x_i\mid x_{:i}).
            \]
            \item For input index $i>i_0^*$, its output node computes the conditional probability $q$ over all possible length $k$ binary strings, conditioned on the input prefix $x_{:{i_0(i)}}$. That is, for $1\leq j\leq 2^k$,
            \[
            v_{U,\txtout}^{(i-1)\Tcal_U + jk\tau-1} = q(z^{(j)} \mid x_{:i_0(i)+1}).
            \]
        \end{itemize}
    The RNN U has a size of $|U| = |Q|+|H_Q|+2k+7$ and a hidden node set size of $|H_U| = |H_Q|+2k+6$.
\end{restatable}

\begin{restatable}[Computing Exponential of Weighted Distinguisher $f_2$]{lemma}{lemmaftwo}
    \label{lemma:rnn_exp_distinguisher}
    Let $n \in\N$ be the document length and $k,i_0^*\in\N$ with $i_0^*\in[0,k-1]$. Let $d:[n]\times\Sigma^n \rightarrow \{0,1\}$ be a next-$k$-token distinguisher implemented by an RNN $D$. Let $\tau \in \N$ be an integer such that $\tau \geq \Tcal_D+2
    $. 
    Let $x_1,x_2,\cdots,x_n$ be the input stream.
    Then there exists an RNN $W$ with RNN-time $\Tcal_W = (2^k+1) k\tau$ such that for any $i_0^*< i\leq n$,
    \[
    v_{W,\txtout}^{(i-1)\Tcal_W + jk\tau-1} =\exp\left(-\alpha
        d\left(i_0(i)+1, x_{:i_0(i)}\cdot z^{(j)}\right)\right).
    \]
    The RNN $W$ has a size of $|W| =|D|+|H_D|+2k+7 $ and a hidden node set size of $|H_W| = |H_D|+2k+6$. 
\end{restatable}

\begin{restatable}[Computing Indicator Functions $g_1$ and $g_2$]{lemma}{lemmag}
    \label{lemma:rnn_indicator_first_digit}
    Let $n \in\N$ be the document length.
    Let $\tau \geq 4$ be an integer. 
    Let $x_1,x_2,\cdots,x_n$ be the input stream.
    Then there exists an RNN $O$ with RNN-time $\Tcal_O =(2^k+1)k \tau $ such that for any $i_0^*< i\leq n$, there exists two nodes $v_1,v_2$ such that
    \[
    v_1^{(i-1)\Tcal_O + jk\tau-1} =
        \ind{z^{(j)}_{:i-i_0(i)+1} = x_{i_0(i)+1:i+1}},
        \quad
    v_2^{(i-1)\Tcal_O + jk\tau-1} =
        \ind{z^{(j)}_{:i-i_0(i)} = x_{i_0(i)+1:i}}.
    \]
    The size $|O|=3k+8$ and its hidden node set size $|H_O| =2k+5$.
\end{restatable}

The proofs for Lemma~\ref{lemma:rnn_prod_q}, Lemma~\ref{lemma:rnn_exp_distinguisher} and Lemma~\ref{lemma:rnn_indicator_first_digit} are in Section~\ref{sec:proof_lemma789}.

We can now construct an RNN that implements the boosted next-token conditional distribution in Equation~\eqref{eq:ntp_q'_fg} by combining these RNN components, using the fact that product and reciprocal are transition functions. This construction, which we formalize in Lemma~\ref{lemma:rnn_boosting_construction}, yields an RNN of size $|Q|+|H_Q|+|D|+|H_D|+O\left(k\right)$ and hidden node set size $|H_Q|+|H_D|+O\left(k\right)$. The proof is in Section~\ref{sec:proof_lemma3}.

\subsection{Synchronized enumeration}
\label{sec:syn_enu}

The main task of the boosted LM is to compute the four functions $f_1,f_2,g_1,g_2$ in Equation~\eqref{eq:ntp_q'_fg} for all possible extensions of an input prefix, keeping the computations synchronized.  These computations require both enumerating the extensions and synchronizing the outputs of all functions so that they can be combined. Lemma~\ref{lemma:iterate_compose} shows that for any given RNN, we can construct a new RNN that iterates over all possible extensions of a given input prefix with length-$k$ strings. The latter RNN enumerates strings and tokens within them, producing the corresponding outputs at predetermined time steps. This lemma serves as the foundation for Lemma~\ref{lemma:rnn_prod_q} and Lemma~\ref{lemma:rnn_exp_distinguisher}.

The construction in Lemma~\ref{lemma:iterate_compose} requires a mechanism for precise, step-by-step control over RNNs. We first establish this tool in Lemma~\ref{lemma:rnn_load_run_hold}. This lemma provides a general method to augment any RNN, enabling it to dynamically, at any time step, \textbf{load} a new state from an external source, \textbf{run} its original transition function, or \textbf{hold} its state constant. This gated control is the essential building block for constructing the enumerating RNN in Lemma~\ref{lemma:iterate_compose}.

\begin{lemma}[RNN Augmentation for Gated State Updates]
\label{lemma:rnn_load_run_hold}
     Let $Q$ be an RNN with a node set $S\subseteq V_Q$. Suppose we are given:
     \begin{enumerate}
         \item[(i)]
         An RNN $C$ whose output $v_{C,\txtout}^t\in \{0,1,2\}$, representing $\mathrm{LOAD}$, $\mathrm{RUN}$, and $\mathrm{HOLD}$ respectively;
         \item [(ii)] An external vector $\mathcal{S}^t\in \R^{|S|}$, provided whenever $v_{C,\txtout}^t=0$ $\mathrm{(LOAD)}$.
     \end{enumerate}
    Then there exists an augmented RNN $\tilde{Q}$ with size $|\tilde{Q}|=|Q|+|C|$ that maintains a node set $\tilde{S}=\{\tilde{u}\mid u\in S\}$ corresponding to $S$,
    whose values update according to $v_{C,\txtout}$. Specifically, for each node $u\in S$, with update rule
    $u^{t}=f_u(V_Q^{t-1})$, its corresponding node $\tilde{u}\in\tilde{S}$ updates as
    \[
        \tilde{u}^{t} =
        \begin{cases}
        \mathcal{S}_u^{t-1} & \text{if } v_{C,\txtout}^{t-1} = 0 ; \quad \text{(LOAD from $\mathcal{S}$}\text{)}  \\
        f_u(V_{\tilde{Q}}^{t-1}) & \text{if } v_{C,\txtout}^{t-1} = 1; \quad \text{(RUN original logic on } S\text{)}\\
        \tilde{u}^{t-1} & \text{if } v_{C,\txtout}^{t-1}= 2.  \quad \text{(HOLD state)}
        \end{cases}
    \]
\end{lemma}
\begin{proof}
    Let $\tilde{Q}$ include all nodes in $Q$ and $C$. For each node $u\in S$ which transition function $f_u$, we specify the update rule of its corresopnding node $\tilde{u}\in\tilde{S}$ as
    \[
    \tilde{u}^t = \mathcal{S}_u^{t-1}\cdot \ind{v_{C,\txtout}^{t-1} = 0}
    + f_u(u_\txtin^{t-1},V_{\tilde{Q}}^{t-1})\cdot
    \ind{v_{C,\txtout}^{t-1} = 1}
    +\tilde{u}^{t-1}\cdot
    \ind{v_{C,\txtout}^{t-1} = 2}.
    \]
    This is a valid transition function since indicator functions are transition functions by Lemma~\ref{lemma:transition_function}.
\end{proof}

We now introduce the key lemma of this section. This lemma constructs an RNN that systematically iterates over and evaluates all possible continuations in a time-synchronized manner, while maintaining a compact hidden node set and  bounded model size. 

To better fit the next-token probability $q'$ from Equation~\eqref{eq:ntp_q'_fg}, we fix an offset $i_0^*\in[0,k-1]$. 
For input indices $i_1$ up to this offset $i_0^*$, the constructed RNN simply mirrors the original model. Starting from the index $i_0^*+1$, we partition all subsequent indices into contiguous non-overlapping blocks of length $k$, as illustrated in Figure~\ref{fig:lemma1_selfboost}. Let $i_0(i_1)$ be the starting index of the block that contains $i_1$. Then, for any index $i_1$ in this length-$k$ block, the constructed RNN's output is a function of a systematic evaluation that simulates all possible length-$k$ continuations from the index $i_0(i_1)$.
Specifically, the model iterates over all possible length-$k$ strings $z^{(j_1)}\in\Sigma^k$ for $1\leq j_1\leq |\Sigma|^k$. For each string $z^{(j_1)}$, it iteratively outputs the output of the original RNN on the input formed by concatenating the fixed prefix $x_{:i_0(i_1)+1}$ with the $r_1$-length prefix of that string $z^{(j_1)}_{:r_1+1}$. Lemma~\ref{lemma:rnn_load_run_hold} guarantees that the computation of $g_Q(x_{:i_0(i_1)+1}\cdot z^{(j_1)}_{:r_1+1})$ completes within a pre-selected time $\tau \geq \Tcal_Q+2$.

\begin{restatable}[Synchronized Enumeration]{lemma}{lemmaMain}
    \label{lemma:iterate_compose}
    Suppose we are given an RNN $Q$, which receives a new input token from the input stream every $\Tcal_Q$ steps, and $v_{Q,\txtout}^{i \cdot \Tcal_Q}=g_Q(x_1,\cdots,x_i)$ for any $1\leq i\leq n$.
    Let $i_0^*, k,n,\tau \in \N$, such that $\tau \geq \Tcal_Q+2$ and $0\leq i_0^* \leq k-1$. Let $I:=\{i\in \{0,1,2,\cdots,n\} \mid i\equiv i_0^* \Mod{k}\}$, and for any $i\in[1,n]$, let $i_0(i):=\max\{i'\in I \mid i'<i\}$ be the largest index in $I$ smaller than $i$.
    Then there exists an RNN $U$ such that for any input stream $x_1,x_2,\cdots,x_n$, it has RNN-time $\Tcal_U=(2^k+1)k\tau$ and for any $i_1\in N$,
        \begin{enumerate}
            \item ($Q$ on input till $i_1$) if $i_1\leq i_0^*$, for any $t\in [(i_1-1)\Tcal_U + \Tcal_Q,  i_1\Tcal_U]$, the output $v_{U,\txtout}$ equals the output of the RNN $Q$ on the input $x_{:i_1+1}$, i.e.,
            \[
            v_{U,\txtout}^t = g_Q\left(x_{:i_1+1}\right);
            \]
            \item ($Q$ on input till $i_0(i_1)$ with suffix $z^{(j_1)}$)
            if $i_1\geq i_0^*+1$, for any $1\leq j_1\leq 2^k, 1\leq r_1\leq k$, 
            and $t \in [(i_1-1)\Tcal_U + (j_1-1)k\tau + (r_1-1)\tau + \Tcal_Q, (i_1-1)\Tcal_U + (j_1-1)k\tau + r_1 \tau]$,
            the output $v_{U,\txtout}$ equals the output of RNN $Q$ on the input $x_{:i_0(i_1)+1}\cdot z^{(j_1)}_{:r_1+1}$, i.e.,
            \[
            v_{U,\txtout}^t = g_Q\left(x_{:i_0(i_1)+1}\cdot z^{(j_1)}_{:r_1+1}\right).
            \]
        \end{enumerate}        
 The size of $U$ is bounded by $|Q|+|H_Q|+2k+6$, and its hidden node set size is bounded by $|H_Q|+2k+6$.
\end{restatable}

\paragraph{Proof idea.}The computation of our constructed RNN $U$ involves \textbf{three nested layers of iteration}, ordered from outermost to innermost:
    \begin{itemize}
        \item \textbf{Input loop} iterates over the input entries for index $i_1\in [1,n]$. Each input token $x_{i_1}$ is processed over $\Tcal_U=(2^k+1)k\tau$ steps before advancing to the next entry. 
        \item \textbf{String loop} begins iterating over all $2^k$ binary strings $z^{(j_1)}$ of length $k$ to compute $g_Q(x_{:i_0+1}\cdot z^{(j_1)})$, where $1\leq j\leq 2^k$. Each string takes $k\tau$ steps to process. After this, it will compute $g_Q(x_{:i_0(i_1+1)+1})$ in $k\tau$ steps.
        \item \textbf{Digit loop} processes each of the $k$ digits of the current string $z^{(j_1)}$, with each digit requiring $\tau$ steps to compute. Specifically for the $r_1$-th digit, the RNN computes $g_Q(x_{:i_0+1}\cdot z^{(j_1)}_{:r_1+1})$.
    \end{itemize}

According to these, each time step $t\in\N$ corresponds to a specific position within these loops. We use the following formulation to precisely identify the current state of computation in all three nested loops at any given time step $t\in \N$.
    \begin{itemize}
        \item Each input loop iteration (i.e., processing one input $x_{i_1})$ takes $\Tcal_U=(2^k+1)k\tau$ steps. Thus, the current input index is
        \[
        i_1(t) = \lceil \frac{t}{\Tcal_U}\rceil.
        \]
        Moreover, for $i_1\leq i_0^*+1$,
        each input index $i_1$ maps to the most recent anchor index $i_0\in I$, thus, the current anchor index is
        \[
        i_0(t) =  \left(\lceil \frac{t - i_0^*\cdot \Tcal_U}{k\Tcal_U}\rceil-1\right)\cdot k+i_0^*.
        \]
        \item The index among the $2^kk+k$ total digit positions in the string loop is
        \[
        \mu(t) = \left( \lceil \frac{t}{\tau}\rceil - 1\right) \bmod (2^kk+k)+1.
        \]
        \item The index of the current  binary string $z^{(j_1)}$ being processed in each input string is
        \[
        j_1(t) = \lceil \frac{\mu(t)}{k} \rceil,
        \]
        where $j_1(t)\in[1,2^k+1]$ and $j_1(t)=2^k+1$ corresponds to the last null step (i.e., computing only $x_{:i_0(t+1)}$).
        \item Finally, the index of the digit within $z^{(j_1)}$ that is currently being computed in the inner loop is
        \[
        r_1(t) =
        (\mu(t) - 1)\bmod k  + 1,
        \]
        where $r_1(t)$ ranges in $[1,k]$.
    \end{itemize}

    In turn, we can represent $t$ using $i_1(t),j_1(t),r_1(t)$ and $s_1(t)\in [1,\tau]$, by
    \[
    t = (i_1(t) - 1)(2^kk+1)\tau + (j_1(t)-1)k\tau + (r_1(t)-1)\tau + s_1(t).
    \]

\begin{figure}[ht]
    \centering
    \includegraphics[width=0.8\linewidth]{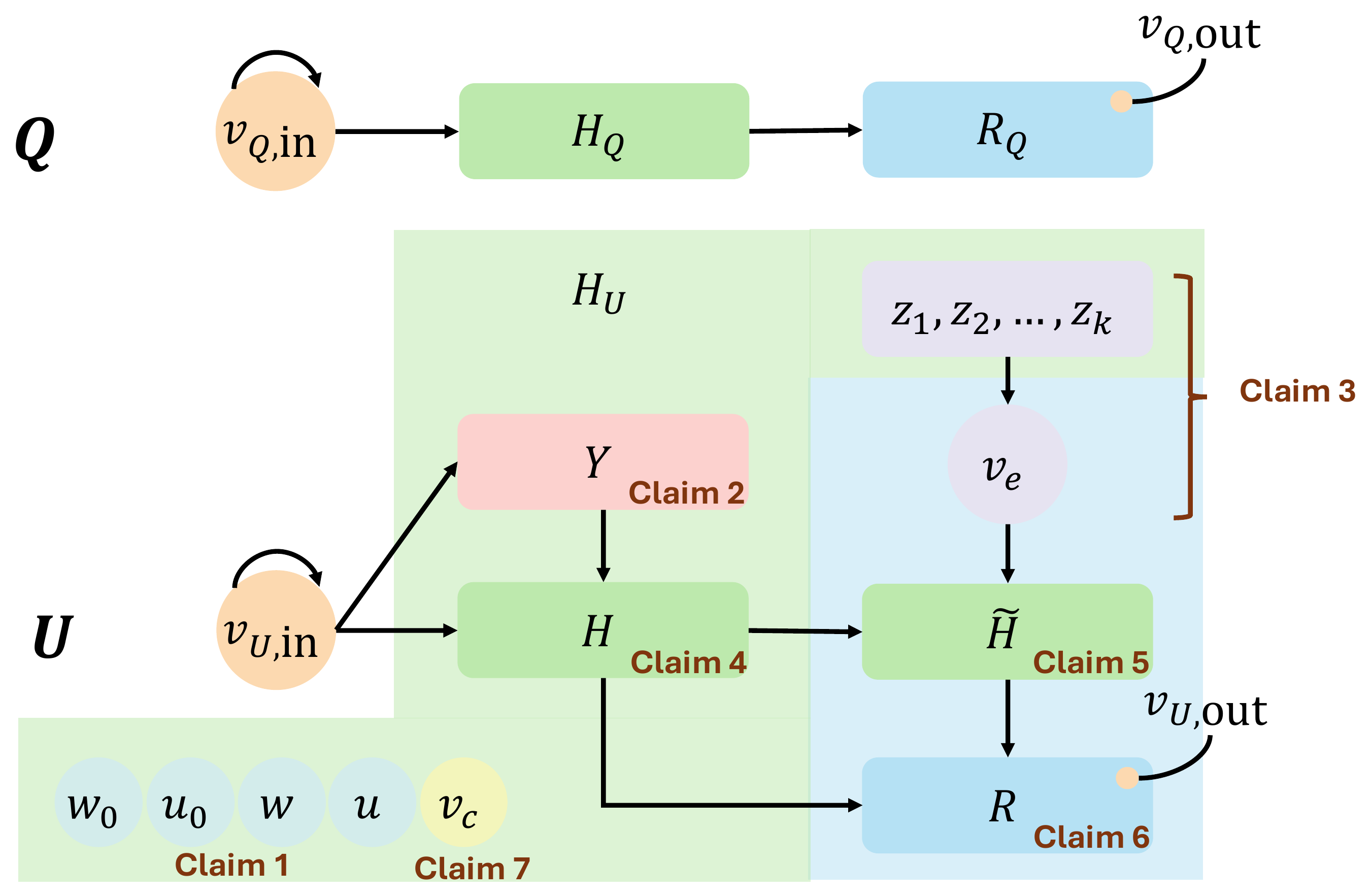}
    \caption{A sketch of the original RNN $Q$, and the constructed RNN $U$. The RNN $Q$ has a hidden node set $H_Q$, and the remaining nodes $R_Q$. The RNN $U$ maintains some counter nodes in Claim~\ref{claim:lem_proof_counters}, a node set $Y$ that stores the input subsequence $x_{i_0+1:i_1+1}$ in Claim~\ref{claim:lem_proof_input_set}, a node set $E$ to enumerates the digits of all length-$k$ strings in Claim~\ref{claim:lem_proof_enumerater}, a node set $H$ that stores the hidden node set corresponding to the prefix $x_{:i_0+1}$ in Claim~\ref{claim:lem_proof_H}, a node set $\tilde{H}$ that tracks the hidden node set computing from extending the fixed prefix $x_{:i_0+1}$ with the $r_1$-length prefix of the length-$k$ string $z^{(j_1)}_{:r_1+1}$ in Claim~\ref{claim:lem_proof_Htilde}, and a node set $R$ that produces the final output in Claim~\ref{claim:lem_proof_R}. Claim~\ref{claim:general_i0*} introduces another counter node, and studies the initial case when $i_1\leq i_0^*$. Note that the counter nodes and the node sets $Y,H,E$ serve as the hidden node set of the constructed RNN $U$.}
    \label{fig:lem_iter_construct}
\end{figure}

Next, we will provide the intuition behind the construction.
For each index of the input stream $i_1\in N$, we compute the function $g_Q$ on the extension of the prefix $x_{:i_0(i_1)+1}$ by a length-$k$ binary string $z^{(j)}$. To do so, we will retain the hidden node set of $x_{:i_0+1}$ for each $i_1\in[i_0,i_0+k-1]$. 
When processing the extension $z^{(j)}$, we copy the stored hidden node set of $x_{:i_0+1}$, and then update it sequentially using the digits of the string $z^{(j)}$. The hidden node set implies the output value through computation. 

We construct the RNN $U$ using the nodes $u_0,w_0,u,w,v_c$ and node sets $E,Y, H, \tilde{H}, R$. Specifically, we use the following claims to show each part of the constructed RNN.
\begin{itemize}
    \item Claim~\ref{claim:lem_proof_counters}: counter nodes.
    \begin{itemize}
        \item The node $u_0 \in [1,k]$ has value equal to $i_1-i_0$, where $i_1$ is the current input index, and $i_0=\max\{i'\in I \mid i' < i_1\}$ is the most recent anchor index in the set $I$. It tracks the offset of the current index relative to the latest anchor index.
        \item The node $w_{0} \in [1,\Tcal_U]$ serves as a time-step counter for the input loop, indicating the progress within the current processing cycle of input token $x_{i_1}$.
        \item The node $u\in[1,k]$ indicates the current digit index in the binary string $z^{(j)}$, and thus represents the iteration progress within the digit loop. 
        \item The node $w\in[1,\tau]$ functions as the time-step counter for the digit loop, tracking the computation steps for each digit $z^{(j)}_r$.
    \end{itemize}
    \item Claim \ref{claim:lem_proof_input_set}: The node set $Y$ stores the input subsequence from the most recent anchor to the current input position $x_{i_0+1},x_{i_0+2},\cdots, x_{i_1}$.
    \item Claim \ref{claim:lem_proof_enumerater}: The node set $E$ is responsible for enumerating the digits of all length-$k$ strings.
    \item Claim \ref{claim:lem_proof_H}: The node set $H$ stores the hidden state corresponding to the prefix $x_{:i_0+1}$. It replicates the structure of the hidden node set $H_Q$. During each iteration of the input loop (identified by $w$ and $u$), if the anchor index $i_0$ remains unchanged (indicated by $u_0$), the values in $H$ stay fixed. Otherwise, when the anchor index transitions from $i_0$ to $i_0+k$, the values in $H$ will update over $k\Tcal$ steps to compute the new prefix $x_{:i_0+k+1}$, and then remain fixed until the next change.
    \item Claim~\ref{claim:lem_proof_Htilde}: The node set $\tilde{H}$ tracks the hidden state computing from extending the prefix $x_{:i_0+1}$ with a partial string $z^{(j)}_{:r+1}$. It also replicates the structure of the hidden node set $H_Q$. In each iteration of the string loop (identified by $u$ and $w$), $\tilde{H}$ is copied from $H$, and then sequentially updated using the digits of the string $z^{(j)}$, as provided by the node set $E$. 
    \item Claim~\ref{claim:lem_proof_R}: The node set $R$ replicates the structure of the node set $G_Q \setminus H_Q$, and includes the output node. It connects to the node set $\tilde{H}$, and produces the final output.
    \item Claim~\ref{claim:general_i0*} studies the initial phase when $i_1\leq i_0^*$. We include another counter node $v_c\in [1,i_0^*+1]$ who equals $i_1$ if $i_1\leq i_0^*+1$, and $i_0^*+1$ otherwise.
\end{itemize}

Figure~\ref{fig:lem_iter_construct} gives a sketch for the construction. 
The proof of this lemma is in Section~\ref{sec:proof_lemma11}.

\subsection{Proofs}
\label{subsec:proof_lemmas78910}
In this section, we prove the main and auxiliary lemmas building up to the proof of the main theorem. 

\subsubsection{Proof of Lemma~\ref{lemma:lm_update_decrease_kl}: Boosted Text Distribution}
\label{sec:proof_lemma1}
We start with the proof of Lemma~\ref{lemma:lm_update_decrease_kl}, which shows a language model boosted by a distinguisher that decreases the KL divergence to the true language model. We restate the lemma here for the reader's convenience.
\lemmalmupdate*
\begin{proof}[Proof of Lemma~\ref{lemma:lm_update_decrease_kl}]
    We first define the following sets for $j\in\{0,1,\cdots,k-1\}$,
    \[
    R(j,n,k):= \{i\in[n] \mid i\equiv j \Mod{k}\} 
    =\left\{
    j,j+k,j+2k,\cdots, j + \left\lfloor \frac{n-1-j}{k}\right\rfloor \cdot k
    \right\}
    \subseteq \{0,1,\cdots,n-1\}.
    \]
    Denote the sizes of these sets as
    \[
    w_{j} := |R(j,n,k)| = 1 + \left\lfloor \frac{n-1-j}{k}\right\rfloor.
    \]
    Then, we can decompose the advantage $a(d,\bp,\bq)$ into a weighted average of $k$ terms $a_1,\cdots,a_k$ using modulus $k$.
    \[
    a(d,\bp,\bq) = \sum_{j=0}^{k-1}\frac{w_j}{n}a_j \text{ where }
    a_j := \Eop\limits_{y\sim p}\left[
    \frac{1}{w_j}\sum_{i\in R(j,n,k)} 
    \left(
    \Eop\limits_{x\sim \bq} 
    \left[d_{i+1}(x) \mid x_{:i+1}=y_{:i+1}\right]-d_{i+1}(y)
    \right)
    \right].
    \]
    We claim that for some $j$, $a_j\geq \alpha$. Let $j$ be the term with the largest value of $a_j$. Then if $a_j < \alpha$, we have 
    \[
    \alpha = \sum_{j=0}^{k-1}\frac{w_j}{n} a_j  < \sum_{j=0}^{k-1}\frac{w_j}{n}  \alpha= \alpha,
    \]
    which leads to a contradiction. So we know $a_j \geq \alpha$. 
    For the remainder of the proof, we fix the $i_0^*\in\{0,1,2,\cdots,k-1\}$ such that
    \[
    a_{i_0^*} \geq \alpha.
    \]

    Consider breaking $s$ into blocks of $k$ tokens, starting at the $i_0^*$-th token. Just as one can view $q$ as a token-by-token distribution, one can also view it as a block-by-block distribution over $b$ blocks $[i_0^*+1+tk:i_0^*+1+(t+1)k)$ as follows.
    \[
    \bq(x) = q(x_{:i_0^*+1}) \prod_{i \in R(i_0^*,n,k)} q(x_{i+1:i+k+1}\mid x_{:i+1})
    \]
    According to Equation~\eqref{eq:q_update_k_token}, the distribution $q'$ is the distribution $q$, but where the block starting at $i$ has a conditional distribution that is reweighted by $e^{-\alpha d_{i+1}(x)}$. The normalization term $Z(x_{:i+1})$ is defined so that the conditional block distributions sum to $1$. Then, we can quantify the KL divergence as follows.
    \begin{align*}
        &\KL(\bp\|\bq) - \KL(\bp\|\bq')\\
        =&\Eop\limits_{y\sim \bp} \left[\log \frac{\bp(y)}{\bq(y)}-\log \frac{\bp(y)}{\bq'(y)}\right]\\
        =&\Eop\limits_{y\sim \bp} \left[
        \log \frac{\bq'(y)}{\bq(y)}
        \right]\\
        =& \Eop\limits_{y\sim \bp} \left[
        \log \frac{
        q(y_{:i_0^*+1})\prod_{i\in R(i_0^*,n,k)} q(y_{i+1:i+k+1} \mid y_{:i+1})e^{-\alpha d_{i+1}(y)}/Z(y_{:i+1})
        }{
        q(y_{:i_0^*+1})\prod_{i\in R(i_0^*,n,k) }q(y_{i+1:i+k+1} \mid y_{:i+1})
        }
        \right]\\
        =&\Eop\limits_{y\sim \bp}\left[
        \sum_{i\in R(i_0^*,n,k)} \left( -\alpha d_{i+1}(y) - \log Z(y_{:i+1})
        \right)
        \right]\\
        =& \Eop\limits_{y\sim \bp}\left[
        \sum_{i\in R(i_0^*,n,k)}\left(
         \alpha \Eop\limits_{x\sim \bq}\left[
        d_{i+1}(x) \mid x_{:i+1}=y_{:i+1}
        \right] - \alpha d_{i+1}(y)
        - \alpha  \Eop\limits_{x\sim \bq}\left[
        d_{i+1}(x) \mid x_{:i+1}=y_{:i+1}
        \right]
        -\log Z(y_{:i+1})
        \right)
        \right]\\
        =&
        \alpha w_{i_0^*}a_{i_0^*} - 
        \Eop\limits_{y\sim \bp}\left[\sum_{i\in R(i_0^*,n,k)}
        \left(
        \alpha  \Eop\limits_{x\sim \bq}\left[
        d_{i+1}(x) \mid x_{:i+1}=y_{:i+1}
        \right]
        +\log Z(y_{:i+1})
        \right)
        \right]\\
        \geq &
        \alpha^2 w_{i_0^*} - 
        \Eop\limits_{y\sim \bp}\left[\sum_{i\in R(i_0^*,n,k)}
        \left(
        \alpha  \Eop\limits_{x\sim \bq}\left[
        d_{i+1}(x) \mid x_{:i+1}=y_{:i+1}
        \right]
        + Z(y_{:i+1}) - 1
        \right)
        \right]
    \end{align*}
    In the last step, we use the fact that $a_{i_0^*} \geq \alpha$ and
    the inequality that $\log x \leq x-1$ for $x>0$. Substituting the definition of $Z$ in the above gives,
    \begin{align*}
        \KL(\bp\|\bq) - \KL(\bp\|\bq')
        \geq& \alpha^2 w_{i_0^*}-
        \Eop\limits_{y\sim \bp}\left[\sum_{i\in R(i_0^*,n,k)}
        \Eop\limits_{x\sim \bq}\left[
        \alpha d_{i+1}(x) 
        +e^{-\alpha d_{i+1}(x)} - 1
        \mid x_{:i+1}=y_{:i+1}
        \right]
        \right]
    \end{align*} 
    By Taylor expansion, $e^{-\beta}+\beta - 1\leq \beta^2/2$ for $\beta>0$. Plugging this into the above gives,
    \begin{align*}
        \KL(\bp\|\bq) - \KL(\bp\|\bq')
        \geq&
        \alpha^2w_{i_0^*} - \Eop\limits_{y\sim \bp}\left[\sum_{i\in R(i_0^*,n,k)}
        \Eop\limits_{x\sim \bq}\left[
        \alpha^2 d_{i+1}^2(x)/2
        \mid x_{:i+1}=y_{:i+1}
        \right]
        \right]\\
        \geq & 
        \alpha^2 w_{i_0^*}  - w_{i_0^*} \alpha^2/2\\
        =& \alpha^2w_{i_0^*}/2
    \end{align*}
    In the second inequality, we use the definition of $d_i(x) \in \{0,1\}$. Recall that $w_{i_0^*} := 1+\lfloor (n-1-(i_0^*))/k\rfloor$. To finish the proof, it suffices to show that $w_{i_0^*} >n/(2k)$. To see this, note that if $n/(2k)<1$, then $w_{i_0^*} \geq 1 > n/(2k)$. For $n/(2k) \geq 1$, we have
    \[
    w_{i_0^*} \geq 1 + \lfloor \frac{n-1-(i_0^*)}{k}\rfloor
    \geq 1 +  \lfloor \frac{n-1-(k-1)}{k}\rfloor
    \geq 1 +  \frac{n-k}{k} - 1
    =\frac{n}{k}-1 
    \geq \frac{n}{k} - \frac{n}{2k}
    =\frac{n}{2k}
    \]
\end{proof}

\subsubsection{Proof of Lemma~\ref{lemma:self_boost_next_token_prob}: Boosted Next-token Probability}
\label{sec:proof_lemma2}
We next prove Lemma~\ref{lemma:self_boost_next_token_prob}, which characterizes the next-token probability of the boosted model given a distinguisher. 
\lemmaselfboostnexttoken*
\begin{proof}[Proof of Lemma~\ref{lemma:self_boost_next_token_prob}]

    By Lemma~\ref{lemma:lm_update_decrease_kl}, there exists $i_0^*\in [0,k-1]$ such that the model $q'$ defined as follows: for $\forall x \in\Sigma^n$,
    \begin{align*}
        q'(x_{:i_0^*+1}):=q(x_{:i_0^*+1}); \quad
    \forall i_0\in I:\quad 
    q'(x_{i_0+1:i_0+k+1}\mid x_{:i_0+1})\propto q(x_{i_0+1:i_0+k+1}\mid x_{:i_0+1})e^{-\alpha d_{i_0+1}(x)},
    \end{align*}
    where $I:=\{ i_0\in [n] \mid i_0 \equiv i_0^* \Mod{k} \}$. Then the next-token conditional probability $q'$ satisfies
    \[
    \KL(\bp\| \bq')\leq \KL(\bp\|\bq)-\frac{\alpha^2 n}{4k}.
    \]
    Next, we will compute the next-token conditional probability of $q'$.

    Firstly for $i\leq i_0^*$, $q'$ is the same as $q$. For each $i \geq i_0^*+1, i=i_0+r_0$, where $i_0\in I,1\leq r_0\leq k$. The next-token conditional probability depends on $q'(x_{i_0+1:i_0+k+1}\mid x_{:i_0+1})$. Thus, we only need to prove for a fixed $i_0$, and all $i_0\in I$ follow samely. We first compute $q'(x_{i_0+1:i+1}\mid x_{:i_0+1})$. Hence, 
    \begin{align*}
        q'(x_{i_0+1:i+1} \mid x_{:i_0+1})
        =& \sum_{s\in\Sigma^{k-(i-i_0)}} q'(x_{i_0+1:i+1}\cdot s \mid x_{:i_0+1})\\
        =& \frac{
            \sum\limits_{s\in\Sigma^{k-(i-i_0)}}q(x_{i_0+1:i+1}\cdot s\mid x_{:i_0+1})\exp\left(-\alpha d_{i_0+1}(x_{:i+1}\cdot s)\right)
            }{
            \sum\limits_{s\in\Sigma^{k}}q(s\mid x_{:i_0+1})\exp\left(-\alpha d_{i_0+1}(x_{:i_0+1}\cdot s)\right)
            }\\
            =&\frac{
            \sum\limits_{s\in \Sigma^k }q(s \mid x_{:i_0+1})\exp\left(-\alpha d_{i_0
            +1}(x_{:i_0+1}\cdot s)\right)\cdot\ind{s_{:i-i_0+1} = x_{i_0+1:i+1}}
            }{
            \sum\limits_{s\in\Sigma^{k}}q(s\mid x_{:i_0+1})
            \exp\left(-\alpha d_{i_0+1}(x_{:i_0+1}\cdot s)\right)
            }
    \end{align*}
    Thus, the next-token conditional probability can be computed as
    \begin{align*}
        q'(x_i \mid x_{:i})
        = & \frac{
        q'(x_{i_0+1:i+1} \mid x_{:i_0+1})
        }{
        q'(x_{i_0+1:i} \mid x_{:i_0+1})
        }\\
        =& \frac{
        \sum\limits_{s\in \Sigma^k }q(s \mid x_{:i_0+1})\exp\left(-\alpha d_{i_0+1}(x_{:i_0+1}\cdot s)\right)\cdot \ind{s_{:i-i_0+1} = x_{i_0+1:i+1}}
        }{
        \sum\limits_{s\in \Sigma^k }q(s \mid x_{:i_0+1})\exp\left(-\alpha d_{i_0+1}(x_{:i_0+1}\cdot s)\right)\cdot\ind{s_{:i-i_0} = x_{i_0+1:i}}
        }
    \end{align*}
\end{proof}

\subsubsection{Proof of Lemma~\ref{lemma:iterate_compose}: Synchronized Enumeration}
\label{sec:proof_lemma11}
We now present the proof, which follows the proof idea described in Section~\ref{sec:syn_enu}.
\begin{proof}[Proof of Lemma~\ref{lemma:iterate_compose}]
We first prove that the lemma holds when $i_0^*=0$, and then we will generalize it to any $i_0^*\in[0,k-1]$ in Claim~\ref{claim:general_i0*}. Finally, we analyze the size and hidden set size of the RNN in Claim~\ref{claim:complexity}.
    \begin{claim}[Counter Nodes $w_0,u_0,w,u$]
    \label{claim:lem_proof_counters}
        There exists an RNN with four counter nodes $w_0,u_0,w,u$ with 
        \begin{align}
        \label{eq:lem_iter_vpt_value}
            w_0^t = (t-1) \bmod (\Tcal_U)+1
        \end{align}
        serves as a step counter inside the index loop with range $w_0 \in [1,\Tcal_U]$.
        \begin{align}
            \label{eq:lem_iter_u0t_value}
            u_0^t = i_1(t)-i_0(t)
        \end{align}
        computes the difference between the current input index and the latest anchor index with range $u_0\in[1,k]$.
        \begin{align}
            \label{eq:lem_iter_w_value}
            w^t = s_1(t)
        \end{align}
        serves as a step counter inside the digit loop with range $w\in[1,\tau]$.
        \begin{align}
            u^t = r_1(t+1),
        \end{align}
        which is the index of the processed digit for time $t+1$ with range $u\in [1,k]$.
    \end{claim}
    \begin{proof}[Proof of Claim~\ref{claim:lem_proof_counters}]
            We construct the RNN with four nodes $w_0,u_0,w,u$.
            Let $w_{0}$ be a time-step counter node for the input loop, indicating the progress within the current processing cycle of the input token. The RNN receives a new input token from the input stream whenever $w_0^t = \Tcal_U$. Being initialized as $1$, $w_{0}$ update as
        \begin{align}
        \label{eq:lem_iter_vpt_update}
            w_{0}^t = \begin{cases}
            w_{0}^{t-1} + 1 & \text{ if }w^{t-1}_0 \leq \Tcal_U - 1;\\
            1 & \text{ if }w^{t-1}_{0} = \Tcal_U .
        \end{cases} 
        \end{align}
        By computation, its value satisfies Equation~\eqref{eq:lem_iter_vpt_value}.

        Next, we construct the node $u_0$ that tracks the difference between the current input index $i_1(t)$ and the most recent anchor $i_0(t)$, cycling over the range $[1,k]$. It is initialized to $1$, and increments by $1$ each time the RNN receives a new input token, i.e., when $i_1(t)=i_1(t-1)+1$. Once it reaches its maximum value $k$, it wraps around and resets to $1$.
        \begin{align}
            \label{eq:lem_iter_u0_update}
            u_0^t = \begin{cases}
                u_0^{t-1}+1 & \text{ if }w_0^{t-1} = \Tcal_U \text{ and }u_0^{t-1}\leq k-1;\\
                1 & \text{ if }w_0^{t-1} = \Tcal_U \text{ and }u_0^{t-1}= k;\\
                u_0^{t-1} & \text{ otherwise.}
            \end{cases}
        \end{align}
        Note that whenever $i_1(t)=i_1(t-1)+1$, we have $w_0^{t-1}=\Tcal_U$. Thus, the node $u_0$ also satisfies Equation~\eqref{eq:lem_iter_u0t_value}.
    
        For the node $w$, we let
        $w$ to be initialized as $1$, and increments by one at each step until it reaches its maximum $\tau$, at which point it wraps around and resets to $1$. Specifically,
        \begin{align}
        \label{eq:lem_iter_wt_update}
            w^t = \begin{cases}
                1 &\text{ if }w^{t-1}=\tau;\\
                w^{t-1}+1 & \text{ otherwise.}
            \end{cases}
        \end{align}
        This update ensures that $w^t$ counts iteratively in the range $[1,\tau]$. Since each digit loop takes $\tau$ steps, we have $w^t=s_1(t)$.
        
        The node $u$ has the predecessors $w$ and $u$. It is initialized as $1$. It retains its value until $w^{t-1}=\tau-1$, at which point it increments by $1$.  If it reaches $k$ and needs to increment, it wraps around and resets to $1$. Its update rule is defined as follows.

        \begin{align}
        \label{eq:lem_iter_ut_update}
            u^t =  \begin{cases}
                1  &\text{ if }w^{t-1}=\tau-1 \text{ and }u^{t-1}=k;\\
                u^{t-1} + 1 & \text{ if }w^{t-1}=\tau-1 \text{ and }u^{t-1}<k;\\
                 u^{t-1} & \text{ if }w^{t-1}\neq \tau-1.
            \end{cases}
        \end{align}

        For the initial time step, $u^1 = 1 = r_1(2)$.
        By the update rule, $u^t$ will remain its value $1$ until $w^{t-1}=\tau-1$. That is, $u^t = 1$ for $t\in [1,\tau-1]$. Since $r_1(t)=1$ for $t\in [1,\tau]$, we have $u^t = r_1(t+1)$ for $t\in [1,\tau-1]$. Note that $u$ increments by $1$ whenever $w^{t-1}=\tau-1$ and $u^{t-1}<k$. That is, $u$ will increments by $1$ when $t=\gamma \cdot \tau$ for $1\leq \gamma \leq k$. So we know $u^{t} = \gamma$ for $t\in[(\gamma-1) \tau,\gamma \tau-1]$. This also equals $r_1(t+1)$ since $r_1(t) = \gamma$ for $t\in [(\gamma-1)\tau+1,\gamma\tau]$. Then for $t = k\tau$, $u^t=1=r_1(t+1)$. Here we have shown $u^t=r_1(t+1)$ within a digit loop. The remaining digit loops can be shown in the same way, as they have the same initialization and update rules.

    \end{proof}

    \begin{claim}[Input Storage Node set $Y$]
        \label{claim:lem_proof_input_set}
        There exists an RNN $U$, which receives a new input token from the input stream every $\tau$ steps, and $U$ includes
        all nodes in Claim~\ref{claim:lem_proof_counters}, input node $v_{U,\txtin}$, 
        and a node set $Y=\{y_1,y_2,\cdots,y_k\}$, 
        such that each node $y_j$ satisfies the follows for $1\leq j\leq k$.
        \begin{align}
        \label{eq:lem_proof_y_set_value}
            y_j^t = \begin{cases}
                x_{i_0(t)+j} & \text{ if }i_1(t)-i_0(t)> j \text{ or }i_1(t)-i_0(t)=j,w_0^t\geq 2;\\
                0& \text{ otherwise.}
            \end{cases}
        \end{align}
    \end{claim}
    \begin{proof}[Proof of Claim~\ref{claim:lem_proof_input_set}]
        We construct the RNN by including the nodes from Claim~\ref{claim:lem_proof_counters}, input node $v_{U,\txtin}$, 
        and the node set $Y=\{y_1,y_2,\cdots,y_k\}$. Each node in $Y$ is initialized as $0$.
        
        For each input loop, suppose $i_1(t)=i_0(t)+j$ for $1\leq j \leq k$. Then the value of the node $ y_j$ is copied from the input node $v_{U,\txtin}$, and remains fixed until the anchor index $i_0(t)$ changes. When the anchor index increments, i.e., $i_0(t)=i_0(t-1)+k$, we have $i_1(t)=i_0(t)+1$. We reset all $y_j$ to be $0$, and copy the input node to the node $y_1$. Formally, we update the node set $Y$ as follows.
        \begin{align}
        \label{eq:lem_proof_y_set_update}
            y_j^t = \begin{cases}
                0 &\text{ if }w_0^{t-1}=\Tcal_U \text{ and } u_0^{t-1}=k;\\
                v_{U,\txtin}^{t-1}&\text{ if }w_0^{t-1}=1 \text{ and }u_0^{t-1}=j;\\
                y_{j-1}^{t-1} & \text{ otherwise.}
            \end{cases}
        \end{align}
        Since $u_0^t = i_1(t)-i_0(t)$ , this update gives the value in Equation~\eqref{eq:lem_proof_y_set_value}.

    \end{proof}

    Note that within each input loop, the RNN first iterates over all binary strings during the first $2^kk\tau$ steps, and then updates the hidden states without string extension in the final $k\tau$ steps. In the following claim, we construct a node set $E$, where the output node $v_e$ serves as the input to update the hidden nodes. The node $v_e$ first enumerates over all digits of all binary strings in $ 2^k k\tau$ steps, and then becomes $0$ for the final $k\tau$ steps.

    \begin{claim}[Enumerator Node Set $E$]
    \label{claim:lem_proof_enumerater}
        There exists an RNN with all nodes in Claim~\ref{claim:lem_proof_input_set} and a node set $E=\{z_{e1},z_{e2},\cdots,z_{ek},v_e\}$, such that its output node $v_e\in E$ iterates as follows: it remains at $0$ for the initial 
        $k\tau$ time steps, and then sequentially enumerates the digits of all binary strings of length $k$, with each digit held for $\tau$ time steps. Formally, the output node
        \begin{align}
            \label{eq:lem_iter_ve_value}
            v_e^t = \begin{cases}
                0 & \text{ if }j_1(t)=2^k+1;\\
                z^{\left(j_1(t)\right)}_{r_1(t)} & \text{ otherwise.}
            \end{cases}
        \end{align}
    \end{claim}
    \begin{proof}[Proof of Claim~\ref{claim:lem_proof_enumerater}]
        Denote $z_{e1}, z_{e2},\cdots, z_{ek}$ as $k$ nodes representing the $k$ digits of the string $z^{(j)}$, with each $z_{er}\in\{0,1\}$ for $1\leq r\leq k$. Those nodes are reset to the value $0$ whenever the input moves forward to the next digit. 
        For all string loops except the last one, the nodes hold their values unchanged until the last time step of the loop. At this point, they are updated via transition functions that perform incrementation in base-$2$ over $k$ digits. The transition functions, denoted by $f_{\text{add1},r}:\{0,1\}^r \rightarrow \{0,1\}$, update the digits according to the following rule. For $1\leq r\leq k$,
        \begin{align*}
                f_{\text{add1}, r}(z_{e1},\cdots,z_{er}) = 
            \begin{cases}
                z_{er} &  \text{if } \exists j\leq r-1 \text{ s.t. }  z_{ej} =0;\\
                z_{er} + 1 & \text{if } z_{ej}=1,\forall j\leq r-1,\text{ and } z_{er} =0;\\
                0& \text{if } z_{ej}=1,\forall j\leq r.
            \end{cases}
        \end{align*}
        This is a transition function by Lemma~\ref{lemma:transition_function}. Formally, those $k$ nodes are all initialized as $0$, and updated as follows. For each $1\leq r\leq k$,
        \begin{align}
        \label{eq:lem_iter_xer_update}
            z^t_{er} = \begin{cases}
                0 &\text{ if }w_0^{t-1}=\Tcal_U;\\
                f_{\text{add1}, r}(z^{t-1}_{e1},\cdots,z^{t-1}_{ek}) &\text{ if }w^{t-1}=\tau-1 \text{ and }u^{t-1}=k;\\
                z_{er}^{t-1} &\text{ otherwise.}
            \end{cases}
        \end{align}
        For each input loop, $z_{er}^t$ is set to be $0$ when $w_0^{t}=1$. That is, $(z_{e1},\cdots,z_{ek})$ is initialized as $z^{(1)}$.
        Then, inside each string loop, it remains fixed until       
        $w^{t-1}=\tau-1$ and $u^{t-1}=k$. Equivalently, this corresponds to the condition $s_1(t)=\tau, r_1(t)=k$, meaning $z_{er}$ is updated to the next binary string only at the last time step of each string loop. Since $j_1(t)\in[1,2^k+1]$, and there are $2^k$ binary strings of length $k$. Thus when $j_1(t+1)=2^k+1$, $(z_{e1},\cdots,z_{ek})^t=z^{(1)}$. Finally, in the last step of the input loop, it becomes $z^{(2)}$. Formally,
            \begin{align}
            \label{eq:lem_iter_xer_value}
                (z_{e1},\cdots,z_{ek})^t = 
                \begin{cases}
                    z^{(j_1(t+1))} & \text{ if }j_1(t+1)\in[1,2^k] \text{ and }w_0^{t}\leq \Tcal_U-1;\\
                    z^{(1)} & \text{ if }j_1(t+1)=2^k+1;\\
                    z^{(2)} &  \text{ if }w_0^{t}= \Tcal_U.
                \end{cases}
            \end{align}
        The node $v_e$ is initialized as $0$. Inside each input loop, it has a value of $0$ at the first step or the last $\tau$ steps. Otherwise, the node $v_e$ copies the value from $z_{er}$ if $u^{t-1}=k+1-r$ for $1\leq r\leq k$. Otherwise, it has the value $0$.
        \begin{align}
            \label{eq:lem_iter_ve_update}
            v^t_{e} = \ind{w_0^{t-1}\leq 2^kk\tau-1} \cdot \sum_{r=1}^k \ind{u^{t-1}=k+1-r}\cdot z_{er}^{t-1}.
        \end{align}
        Combining Equation~\eqref{eq:lem_iter_xer_value} with Equation~\eqref{eq:lem_iter_ve_update}, we obtain the value of node $v_e$ as given in Equation~\eqref{eq:lem_iter_ve_value}.
        
    \end{proof}
    Note that the node $v_e$ in Claim~\ref{claim:lem_proof_enumerater} precisely matches the digit that is processed at time $t$.

    Next, we will construct node sets that duplicate some node set in the RNN $Q$. 
    Since the input of RNN $Q$ moves to the next entry after $\Tcal_Q$ steps, each node in the RNN will also update for $\Tcal_Q$ steps before the new input is read. For any node $v_Q \in V_Q$, we denote its value after $\kappa$ steps of update after reading the string $\omega$ as $v_Q(\omega)^{\langle\kappa\rangle}$, where $1\leq \kappa \leq \Tcal_Q$. 
    For any node $v\in G_Q$ that corresponds to $v_Q$ by the construction, we also write $v(\omega)^{\langle \kappa\rangle}$ as its value after $\kappa$ updates after reading the string $\omega$.

    In the following claim, we construct the node set $H$ such that during each input loop with index $i_1$, it stores the values of the hidden nodes $H_Q$ with prefix $x_{:i_0(i_1)+1}$ in the first $2^kk\tau$ steps, and updates to  $x_{:i_0(i_1+1)+1}$ in the last $k\tau$ steps.

 \begin{claim}[Node Set $H$]
        \label{claim:lem_proof_H}
        There exists an RNN with all nodes in Claim~\ref{claim:lem_proof_enumerater}, and a node set $H$ where each node $h\in H$ corresponds to a node $h_Q\in H_Q$ in the hidden node set of the RNN $Q$, and has value
        \begin{align}
            \label{eq:lem_iter_ht_value}
            h^t = \begin{cases}
                h_q\left(x_{:i_0(t)+1}\right)^{\langle\Tcal_Q\rangle } & \text{ if }i_1(t)\neq i_0(t)+k \text{ or }i_1(t)=i_0(t)+k  \text{ and }
                w_0^{t}\in [1,2^kk\tau];\\
                h_q\left(x_{:i_0(t)+r_1(t)}\right)^{\langle \Tcal_Q \rangle } & \text{ if }i_1(t)=i_0(t) +k \text{ and }
                s_1(t) =1 \text{ and }
                w_0^t \geq 2^kk\tau+1;\\
                h_q\left(x_{:i_0(t)+r_1(t)+1}\right)^{\langle s_1(t)-1 \rangle } & \text{ if }i_1(t)=i_0(t) +k \text{ and }
                s_1(t) \in [2,\Tcal_Q] \text{ and }
                w_0^t \geq 2^kk\tau+1;\\
                h_q\left(x_{:i_0(t)+r_1(t)+1}\right)^{\langle \Tcal_Q \rangle } & \text{ if }i_1(t)=i_0(t) +k  \text{ and }
                s_1(t) \in [\Tcal_Q+1,\tau] \text{ and }
                 w_0^t \geq 2^kk\tau+1.
        \end{cases}
        \end{align}
    \end{claim}
    \begin{proof}[Proof of Claim~\ref{claim:lem_proof_H}]
        Besides the nodes in Claim~\ref{claim:lem_proof_enumerater}, we let the RNN include a node set $H$ by duplicating all nodes from the hidden node set $H_Q$. Consequently, each node $h_Q\in H_Q$ corresponds to a node $h\in H$. For each node $h_Q\in H_Q$ that is connected to the input node $v_{Q,\txtin}$, we add $k$ edges $(y_1, h), (y_2,h), \cdots, (y_3,h)$, where we consider $y_i$ as the `input node' that connect to $H$. Let the nodes in $H$ have the same initialization as the corresponding nodes in $H_Q$. For each node $h_Q\in H_Q$ that is updated as 
        \[
        h_Q^t =f_{h_Q}\left(v_{Q,\txtin}^{t-1}, H_Q^{t-1}\right)\text{, where }f_{h_Q}\text{ is a transition function},
        \]
        its corresponding node $h\in H$ updates as follows. During each input loop, we update the node $h$ based on whether the anchor index $i_0$ will change for the next input. If $u_0^t=i_1(t)-i_0(t)\neq k$, the anchor index remains the same, and we keep $h$ HOLD throughout the loop. If $u_0^t =k$, the anchor index will change from $i_0(t)$ to $i_0(t)+k$ in the next input loop, so we will RUN $h$ during the last $k\tau$ steps during the current input loop. We proceed in two phases:
        \begin{itemize}
            \item During the first $2^kk\tau$ steps, the node $h$ is set to HOLD.
            \item For the last $k\tau$ steps, for $1\leq j\leq k$, we sequentially RUN $h$ over $\Tcal_Q$ steps using $f_{h_Q}$ with the input node $y_j$, and then HOLD for the remainder of each corresponding $\tau$-step segment.
        \end{itemize}
        Figure~\ref{fig:rnn_construct_setH} shows the LOAD-RUN-HOLD schedule of the set $H$ within each input loop.
        Formally, $h^t$ updates as follows.
        \begin{align}
            \label{eq:lem_iter_h_update}
             h^t =
            \begin{cases}
                f_{h_Q}\left(y_j^{t-1},N^{t-1}(h)\right) & \text{ if }w_0^{t-1}\geq 2^kk\tau \text{ and }
                u_0^{t-1}=k \text{ and } 
                w^{t-1}\in[1,\Tcal_Q] \text{ and }
                u^{t-1}=j;\\
                h^{t-1} &\text{ otherwise.}
            \end{cases}
        \end{align}
        That gives the value in Equation~\eqref{eq:lem_iter_ht_value}.

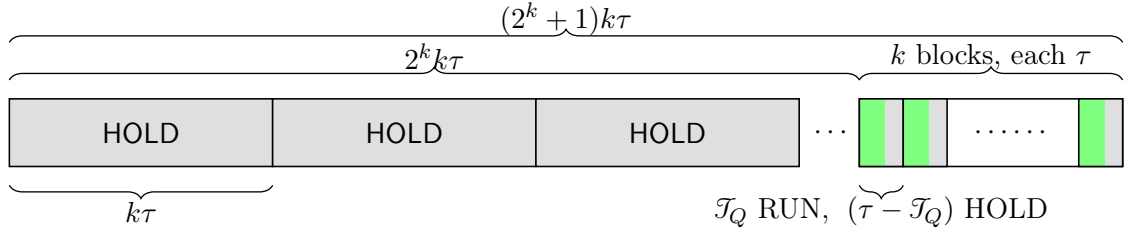
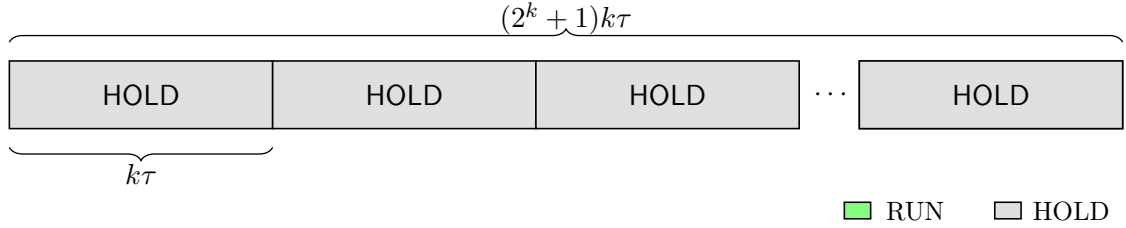
\begin{figure}[t]
\centering

\begin{subfigure}[t]{\textwidth}
\centering
\begin{tikzpicture}[x=1cm,y=1cm]

\def\BlockW{3.5}
\def\BarH{0.9}
\def\Shown{3}
\def\Gap{0.8}
\def\k{6}
\def\TQfrac{0.6}
\def\StepFrac{0.10}

\tikzset{
  runseg/.style ={fill=green!50},
  holdseg/.style={fill=gray!25},
  outline/.style={line width=0.6pt},
  bracelab/.style={decorate, decoration={brace, amplitude=5pt}, line width=0.5pt},
}

\foreach \i in {0,...,\numexpr\Shown-1\relax} {
  \draw[holdseg, outline] (\i*\BlockW,0) rectangle ++(\BlockW,\BarH);
  \node at (\i*\BlockW+0.5*\BlockW,0.5*\BarH) {\sffamily HOLD};
}

\node at (\Shown*\BlockW+0.45,0.5*\BarH) {$\cdots$};

\pgfmathsetmacro{\LastX}{\Shown*\BlockW + \Gap}
\pgfmathsetmacro{\tauW}{\BlockW/\k}

\draw[outline] (\LastX,0) rectangle ++(\BlockW,\BarH);

\pgfmathsetmacro{\xa}{\LastX}
\pgfmathsetmacro{\xb}{\LastX + \tauW}
\pgfmathsetmacro{\xr}{\xa + \TQfrac*\tauW}
\fill[runseg]  (\xa,0) rectangle (\xr,\BarH);
\fill[holdseg] (\xr,0) rectangle (\xb,\BarH);
\draw[outline] (\xa,0) rectangle (\xb,\BarH);

\pgfmathsetmacro{\xa}{\LastX + \tauW}
\pgfmathsetmacro{\xb}{\LastX + 2*\tauW}
\pgfmathsetmacro{\xr}{\xa + \TQfrac*\tauW}
\fill[runseg]  (\xa,0) rectangle (\xr,\BarH);
\fill[holdseg] (\xr,0) rectangle (\xb,\BarH);
\draw[outline] (\xa,0) rectangle (\xb,\BarH);

\node at (\LastX + 3.5*\tauW, 0.5*\BarH) {$\cdots\cdots$};

\pgfmathsetmacro{\xa}{\LastX + (\k-1)*\tauW}
\pgfmathsetmacro{\xb}{\LastX + \k*\tauW}
\pgfmathsetmacro{\xr}{\xa + \TQfrac*\tauW}
\fill[runseg]  (\xa,0) rectangle (\xr,\BarH);
\fill[holdseg] (\xr,0) rectangle (\xb,\BarH);
\draw[outline] (\xa,0) rectangle (\xb,\BarH);

\draw[bracelab] (0,\BarH+0.75) -- (\LastX+\BlockW,\BarH+0.75)
  node[midway,yshift=9pt] {$(2^k+1)k\tau$};

\draw[bracelab] (0,\BarH+0.25) -- (\LastX,\BarH+0.25)
node[midway,yshift=9pt] {$2^kk\tau$};

\draw[bracelab, yshift=-0.25cm, decoration={brace,mirror,amplitude=5pt}]
  (0,0) -- (\BlockW,0)
  node[midway,yshift=-10pt] {$k\tau$};

\draw[bracelab] (\LastX,\BarH+0.25) -- (\LastX+\BlockW,\BarH+0.25)
  node[midway,yshift=9pt] {$k$ blocks, each $\tau$};

\pgfmathsetmacro{\xaDemo}{\LastX}
\pgfmathsetmacro{\xhDemo}{\LastX + \StepFrac*\tauW}
\pgfmathsetmacro{\xrDemo}{\xhDemo + \TQfrac*\tauW}
\pgfmathsetmacro{\xbDemo}{\LastX + \tauW}

\draw[bracelab,decoration={brace,mirror,amplitude=5pt},yshift=-0.25cm]
  (\LastX,0) -- (\LastX+\BlockW/\k,0)
  node[midway,yshift=-10pt] {$\mathcal{T}_Q$ RUN,\; $(\tau-\mathcal{T}_Q)$ HOLD};

\end{tikzpicture}
\caption{When $i_1(t)=i_0(t)+k$, we update the node set $H$ using prefix stored in $Y$ in the last $k\tau$ steps.}
\end{subfigure}

\begin{subfigure}[t]{\textwidth}
\centering
\begin{tikzpicture}[x=1cm,y=1cm]

\def\BlockW{3.5}
\def\BarH{0.9}
\def\Shown{3}
\def\Gap{0.8}
\def\k{6}
\def\TQfrac{0.6}
\def\StepFrac{0.10}

\tikzset{
  runseg/.style ={fill=green!50},
  holdseg/.style={fill=gray!25},
  outline/.style={line width=0.6pt},
  bracelab/.style={decorate, decoration={brace, amplitude=5pt}, line width=0.5pt},
}

\foreach \i in {0,...,\numexpr\Shown-1\relax} {
  \draw[holdseg, outline] (\i*\BlockW,0) rectangle ++(\BlockW,\BarH);
  \node at (\i*\BlockW+0.5*\BlockW,0.5*\BarH) {\sffamily HOLD};
}

\node at (\Shown*\BlockW+0.45,0.5*\BarH) {$\cdots$};

\pgfmathsetmacro{\LastX}{\Shown*\BlockW + \Gap}
\pgfmathsetmacro{\tauW}{\BlockW/\k}

\draw[outline] (\LastX,0) rectangle ++(\BlockW,\BarH);

\fill[holdseg] (\LastX,0) rectangle ++(\BlockW,\BarH);
\draw[outline] (\LastX,0) rectangle ++(\BlockW,\BarH);
\node at (\LastX+0.5*\BlockW, 0.5*\BarH) {\sffamily HOLD};

\draw[bracelab] (0,\BarH+0.25) -- (\LastX+\BlockW,\BarH+0.25)
  node[midway,yshift=9pt] {$(2^k+1)k\tau$};

\draw[bracelab, yshift=-0.25cm, decoration={brace,mirror,amplitude=5pt}]
  (0,0) -- (\BlockW,0)
  node[midway,yshift=-10pt] {$k\tau$};

\pgfmathsetmacro{\xaDemo}{\LastX}
\pgfmathsetmacro{\xhDemo}{\LastX + \StepFrac*\tauW}
\pgfmathsetmacro{\xrDemo}{\xhDemo + \TQfrac*\tauW}
\pgfmathsetmacro{\xbDemo}{\LastX + \tauW}

\begin{scope}[shift={(\LastX-0.2,-1.2)}]
  \draw[runseg,outline] (0,0) rectangle +(0.35,0.22);
  \node[anchor=west] at (0.42,0.11) {\small RUN};
  \draw[holdseg,outline] (2.0,0) rectangle +(0.35,0.22);
  \node[anchor=west] at (2.37,0.11) {\small HOLD};
\end{scope}

\end{tikzpicture}
\caption{When $i_1(t)\neq i_0(t)+k$, we hold the node set $H$.}
\end{subfigure}
\caption{The Load-Run-Hold schedule for the node set $H$ within each input loop.}
\label{fig:rnn_construct_setH}
\end{figure}

    \end{proof}

    By the construction of $H$, we note that inside each input loop,
    \begin{enumerate}
        \item During the first $2^kk\tau$ steps, the node set $H$ stores the hidden nodes of anchor prefix $x_{:i_0(i_1)+1}$.
        \item At the last step, each node $h\in H$ satisfies
        \[
        h^t = \begin{cases}
            h_q\left(x_{:i_0(t)+1}\right)^{\langle \Tcal_Q\rangle} & \text{ if }i_1(t)\neq i_0(t)+k;\\[0.05in] 
            h_q\left(x_{:i_0(t)+k+1}\right)^{\langle \Tcal_Q\rangle}& \text{ if }i_1(t)= i_0(t)+k.
        \end{cases}
        \]
        Note that $i_0(i_1+1)\neq i_0(i_1)$ only when $i_1-i_0(i_1) = k$. Consequently, in the last step, the node set $H$ stores the hidden nodes of the anchor prefix of the next input. That is, $H^t = H_q\left(x_{:i_0(i_1(t)+1)+1}\right)^{\langle \Tcal_Q\rangle }$.
    \end{enumerate}
    Followed by this, we construct another node set $\tilde{H}$ which serves as the hidden nodes $H_Q$ with input $\{x_{:i_0+1}\cdot z^{(j_1)}_{:r_1+1}\}$.

    \begin{claim}[Node Set $\tilde{H}$]
        \label{claim:lem_proof_Htilde}
        There exists an RNN with all nodes in Claim~\ref{claim:lem_proof_H}, and a node set $\tilde{H}$ where each node $\tilde{h}\in\tilde{H}$ corresponds to a node $h_Q\in H_Q$ in the hidden node set of the RNN $Q$, and has value
        \begin{align}
            \label{eq:lem_iter_tildeht_value}
            \tilde{h}^t = \begin{cases}
            h_q\left(x_{:i_0(t)+1}\cdot z^{(j_1(t))}_{:r_1(t)}\right)^{\langle \Tcal_Q\rangle } & \text{ if }s_1(t)=1  \text{ and }w_0^t \leq 2^kk\tau;\\[0.1in] 
            h_q\left(x_{:i_0(t)+1}\cdot z^{(j_1(t))}_{:r_1(t)+1}\right)^{\langle s_1(t)-1 \rangle } & \text{ if }s_1(t)\in[2,\Tcal_Q]\text{ and }w_0^t \leq 2^kk\tau;\\[0.1in] 
            h_q\left(x_{:i_0(t)+1}\cdot z^{(j_1(t))}_{:r_1(t)+1}\right)^{\langle \Tcal_Q-1\rangle } & \text{ if }s_1(t)\in[\Tcal_Q+1,\tau]\text{ and }w_0^t \leq 2^kk\tau;\\[0.1in] 
            h_q\left(x_{:i_0(t)+1}\right)^{\langle \Tcal_Q\rangle } & \text{ if }w_0^t >2^kk\tau.
        \end{cases}
        \end{align}    
    \end{claim}

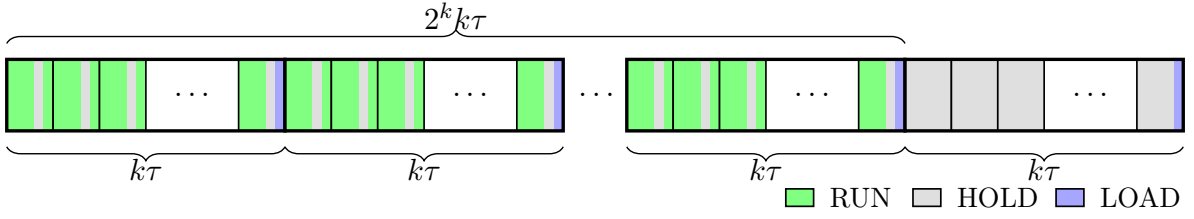
\begin{figure}[t]
\centering
\resizebox{0.95\textwidth}{!}{%
\begin{tikzpicture}[x=1cm,y=1cm]

\def\BlockW{3.5} 
\def\BarH{0.9}
\def\Shown{2} 
\def\Gap{0.8}
\def\k{6}
\def\TQfrac{0.6}
\def\LoadFrac{0.08} 
\def\WiderLoadFrac{0.2} 

\tikzset{
  runseg/.style ={fill=green!50},
  holdseg/.style={fill=gray!25},
  loadseg/.style={fill=blue!35},
  outline/.style={line width=0.6pt}, 
  thickoutline/.style={line width=1.2pt}, 
  bracelab/.style={decorate, decoration={brace, amplitude=5pt}, line width=0.5pt},
}

\pgfmathsetmacro{\tauW}{\BlockW/\k}
\pgfmathsetmacro{\TQfracAdj}{max(\TQfrac,0.5)}

\foreach \i in {0,...,\numexpr\Shown-1\relax}{
  \pgfmathsetmacro{\xA}{\i*\BlockW} 

  \pgfmathsetmacro{\j}{0}
  \pgfmathsetmacro{\xa}{\xA + \j*\tauW}
  \pgfmathsetmacro{\xb}{\xA + (\j+1)*\tauW}
  \pgfmathsetmacro{\xr}{\xa + \TQfracAdj*\tauW}
  \pgfmathsetmacro{\xl}{\xb - \WiderLoadFrac*\tauW} 
  \fill[runseg]  (\xa,0) rectangle (\xr,\BarH);
  \fill[holdseg] (\xr,0) rectangle (\xl,\BarH);
  \fill[runseg]  (\xl,0) rectangle (\xb,\BarH);
  \draw[outline] (\xa,0) rectangle (\xb,\BarH);

  \pgfmathsetmacro{\j}{1}
  \pgfmathsetmacro{\xa}{\xA + \j*\tauW}
  \pgfmathsetmacro{\xb}{\xA + (\j+1)*\tauW}
  \pgfmathsetmacro{\xr}{\xa + \TQfracAdj*\tauW}
  \pgfmathsetmacro{\xl}{\xb - \WiderLoadFrac*\tauW} 
  \fill[runseg]  (\xa,0) rectangle (\xr,\BarH);
  \fill[holdseg] (\xr,0) rectangle (\xl,\BarH);
  \fill[runseg]  (\xl,0) rectangle (\xb,\BarH);
  \draw[outline] (\xa,0) rectangle (\xb,\BarH);
  
  \pgfmathsetmacro{\j}{2}
  \pgfmathsetmacro{\xa}{\xA + \j*\tauW}
  \pgfmathsetmacro{\xb}{\xA + (\j+1)*\tauW}
  \pgfmathsetmacro{\xr}{\xa + \TQfracAdj*\tauW}
  \pgfmathsetmacro{\xl}{\xb - \WiderLoadFrac*\tauW}
  \fill[runseg]  (\xa,0) rectangle (\xr,\BarH);
  \fill[holdseg] (\xr,0) rectangle (\xl,\BarH);
  \fill[runseg]  (\xl,0) rectangle (\xb,\BarH);
  \draw[outline] (\xa,0) rectangle (\xb,\BarH);

\node at (\xA + 4.05*\tauW, 0.5*\BarH) {$\cdots$};

  \pgfmathsetmacro{\j}{\k-1}
  \pgfmathsetmacro{\xa}{\xA + \j*\tauW}
  \pgfmathsetmacro{\xb}{\xA + (\j+1)*\tauW}
  \pgfmathsetmacro{\xr}{\xa + \TQfracAdj*\tauW}
  \pgfmathsetmacro{\xl}{\xb - \WiderLoadFrac*\tauW} 
  \fill[runseg]  (\xa,0) rectangle (\xr,\BarH);
  \fill[holdseg] (\xr,0) rectangle (\xl,\BarH);
  \fill[loadseg] (\xl,0) rectangle (\xb,\BarH);
  \draw[outline] (\xa,0) rectangle (\xb,\BarH);
  
  \draw[thickoutline] (\xA,0) rectangle ++(\BlockW,\BarH); 
}

\node at (\Shown*\BlockW+0.45,0.5*\BarH) {$\cdots$};

\pgfmathsetmacro{\BlockThreeStartX}{\Shown*\BlockW + \Gap} 
\pgfmathsetmacro{\BlockThreeEndX}{\BlockThreeStartX + \BlockW} 
\pgfmathsetmacro{\BlockFourStartX}{\BlockThreeEndX} 
\pgfmathsetmacro{\BlockFourEndX}{\BlockFourStartX + \BlockW}     

\pgfmathsetmacro{\xA}{\BlockThreeStartX} 

\pgfmathsetmacro{\j}{0}
\pgfmathsetmacro{\xa}{\xA + \j*\tauW}
\pgfmathsetmacro{\xb}{\xA + (\j+1)*\tauW}
\pgfmathsetmacro{\xr}{\xa + \TQfracAdj*\tauW}
\pgfmathsetmacro{\xl}{\xb - \WiderLoadFrac*\tauW}
\fill[runseg]  (\xa,0) rectangle (\xr,\BarH);
\fill[holdseg] (\xr,0) rectangle (\xl,\BarH);
\fill[runseg]  (\xl,0) rectangle (\xb,\BarH);
\draw[outline] (\xa,0) rectangle (\xb,\BarH);

\pgfmathsetmacro{\j}{1}
\pgfmathsetmacro{\xa}{\xA + \j*\tauW}
\pgfmathsetmacro{\xb}{\xA + (\j+1)*\tauW}
\pgfmathsetmacro{\xr}{\xa + \TQfracAdj*\tauW}
\pgfmathsetmacro{\xl}{\xb - \WiderLoadFrac*\tauW}
\fill[runseg]  (\xa,0) rectangle (\xr,\BarH);
\fill[holdseg] (\xr,0) rectangle (\xl,\BarH);
\fill[runseg]  (\xl,0) rectangle (\xb,\BarH);
\draw[outline] (\xa,0) rectangle (\xb,\BarH);

\pgfmathsetmacro{\j}{2}
\pgfmathsetmacro{\xa}{\xA + \j*\tauW}
\pgfmathsetmacro{\xb}{\xA + (\j+1)*\tauW}
\pgfmathsetmacro{\xr}{\xa + \TQfracAdj*\tauW}
\pgfmathsetmacro{\xl}{\xb - \WiderLoadFrac*\tauW}
\fill[runseg]  (\xa,0) rectangle (\xr,\BarH);
\fill[holdseg] (\xr,0) rectangle (\xl,\BarH);
\fill[runseg]  (\xl,0) rectangle (\xb,\BarH);
\draw[outline] (\xa,0) rectangle (\xb,\BarH);

\node at (\xA + 4.05*\tauW, 0.5*\BarH) {$\cdots$};

\pgfmathsetmacro{\j}{\k-1}
\pgfmathsetmacro{\xa}{\xA + \j*\tauW}
\pgfmathsetmacro{\xb}{\xA + (\j+1)*\tauW}
\pgfmathsetmacro{\xr}{\xa + \TQfracAdj*\tauW}
\pgfmathsetmacro{\xl}{\xb - \WiderLoadFrac*\tauW} 
\fill[runseg]  (\xa,0) rectangle (\xr,\BarH);
\fill[holdseg] (\xr,0) rectangle (\xl,\BarH);
\fill[loadseg] (\xl,0) rectangle (\xb,\BarH);
\draw[outline] (\xa,0) rectangle (\xb,\BarH); 

\draw[thickoutline] (\xA,0) rectangle ++(\BlockW,\BarH);

\pgfmathsetmacro{\xA}{\BlockFourStartX} 

\pgfmathsetmacro{\j}{0}
\pgfmathsetmacro{\xa}{\xA + \j*\tauW}
\pgfmathsetmacro{\xb}{\xA + (\j+1)*\tauW}
\fill[holdseg] (\xa,0) rectangle (\xb,\BarH);
\draw[outline] (\xa,0) rectangle (\xb,\BarH);

\pgfmathsetmacro{\j}{1}
\pgfmathsetmacro{\xa}{\xA + \j*\tauW}
\pgfmathsetmacro{\xb}{\xA + (\j+1)*\tauW}
\fill[holdseg] (\xa,0) rectangle (\xb,\BarH);
\draw[outline] (\xa,0) rectangle (\xb,\BarH);

\pgfmathsetmacro{\j}{2}
\pgfmathsetmacro{\xa}{\xA + \j*\tauW}
\pgfmathsetmacro{\xb}{\xA + (\j+1)*\tauW}
\fill[holdseg] (\xa,0) rectangle (\xb,\BarH);
\draw[outline] (\xa,0) rectangle (\xb,\BarH);

\node at (\xA + 4.05*\tauW, 0.5*\BarH) {$\cdots$};

\pgfmathsetmacro{\j}{\k-1}
\pgfmathsetmacro{\xa}{\xA + \j*\tauW}
\pgfmathsetmacro{\xb}{\xA + (\j+1)*\tauW}
\pgfmathsetmacro{\xl}{\xb - \WiderLoadFrac*\tauW} 
\fill[holdseg] (\xa,0) rectangle (\xl,\BarH); 
\fill[loadseg] (\xl,0) rectangle (\xb,\BarH); 
\draw[outline] (\xa,0) rectangle (\xb,\BarH); 

\draw[thickoutline] (\xA,0) rectangle ++(\BlockW,\BarH);

\path[use as bounding box] (0,-1.0) rectangle ({\BlockFourEndX},{\BarH+0.9}); 

\draw[bracelab] (0,\BarH+0.2) -- ({\BlockFourStartX},\BarH+0.2) 
  node[midway,yshift=10pt] {$2^k k\tau$};
  
\draw[bracelab,decoration={brace,mirror,amplitude=5pt}, yshift=-0.2cm]
  (0,0) -- (\BlockW,0)
  node[midway,yshift=-8pt] {$k\tau$};
  
\draw[bracelab,decoration={brace,mirror,amplitude=5pt}, yshift=-0.2cm]
  (\BlockW,0) -- (2*\BlockW,0)
  node[midway,yshift=-8pt] {$k\tau$};

\draw[bracelab,decoration={brace,mirror,amplitude=5pt}, yshift=-0.2cm]
  ({\BlockThreeStartX},0) -- ({\BlockThreeEndX},0)
  node[midway,yshift=-8pt] {$k\tau$};

\draw[bracelab,decoration={brace,mirror,amplitude=5pt}, yshift=-0.2cm]
  ({\BlockFourStartX},0) -- ({\BlockFourEndX},0)
  node[midway,yshift=-8pt] {$k\tau$};
  
\begin{scope}[shift={({(\BlockFourEndX - 5)},-0.95)}] 
  \draw[runseg,outline] (0,0) rectangle +(0.35,0.22);
  \node[anchor=west] at (0.42,0.11) {\small RUN};
  \draw[holdseg,outline] (1.6,0) rectangle +(0.35,0.22);
  \node[anchor=west] at (2.02,0.11) {\small HOLD};
  \draw[loadseg,outline] (3.4,0) rectangle +(0.35,0.22);
  \node[anchor=west] at (3.82,0.11) {\small LOAD};
\end{scope}

\end{tikzpicture}%
}
\caption{The Load-Run-Hold schedule for the node set $\tilde{H}$ within each input loop of $(2^k+1)k\tau$ steps. The loop begins with $2^k$ ``string loops'' of length $k\tau$, each divided into $k$ ``digit loops'' of length $\tau$. For the first $k-1$ digit loops, the schedule is $\Tcal_Q-1$ steps of RUN, $\tau-\Tcal_Q$ steps of HOLD, and one final step of RUN. The final digit loop (the $k$-th one) within each of these string loops is $\Tcal_Q-1$ steps of RUN, $\tau-\Tcal_Q$ steps of HOLD, and one step of LOAD from the node set $H$. The entire sequence concludes with a final $k\tau$ ``string loop'' which consists of $k\tau-1$ steps of HOLD, followed by one last step of LOAD from the node set $H$.}
\label{fig:rnn_construct_setHtilde}
\end{figure}

    \begin{proof}[Proof of Claim~\ref{claim:lem_proof_Htilde}]
        Besides the nodes in Claim~\ref{claim:lem_proof_H}, we let the RNN include a node set $\tilde{H}$ by duplicating all nodes from the hidden node set $H_Q$. Consequently, each node in $\tilde{H}$ corresponds to a node $H_Q$. Meanwhile, we let the node $v_e$ connect to $\tilde{H}$, which corresponds to the input node $v_{Q,\txtin} \in G_Q$ that connects to $H_Q$. Let the nodes in $\tilde{H}$ have the same initialization as the corresponding nodes in $H_Q$. 
        Each input loop of $(2^k+1)k\tau$ steps in divided into two distinct phases:
        \begin{itemize}
            \item Enumeration Phase (first $2^kk\tau$ steps):
            This phase sequentially enumerates all $2^k$ binary strings, corresponding to $2^k$ string loops, each lasting $k\tau$ steps. Each string loop is further divided into $k$ digit loops, each lasting $\tau$ steps.
            \begin{itemize}
                \item The first $k-1$ digit loops operate as: $\Tcal_Q-1$ steps of RUN , $\tau - \Tcal_Q$ steps of HOLD, and finally $1$ step of RUN.
                \item The final digit loop operate as: $\Tcal_Q-1$ steps of RUN, $\tau - \Tcal_Q$ steps of HOLD, and finally $1$ step of LOAD from the node set $H$.
            \end{itemize}
            \item Final Phase (last $k\tau$ steps): $\tau-1$ steps of HOLD, and $1$ step of LOAD from the node set $H$.
        \end{itemize}
        Figure~\ref{fig:rnn_construct_setHtilde} shows the LOAD-RUN-HOLD schedule of the set $\tilde{H}$ within each input loop.
        Formally, for each node $h_Q\in H_Q$ that has the transition function $f_{hQ}$, we let its corresponding node $\tilde{h}_Q$ update as
        \begin{align}
            \label{eq:lem_iter_tildeht_update}
            \tilde{h}^t = \begin{cases}
                h^{t-1} & \text{ if }w^{t-1}=\tau \text{ and }u^{t-1}=1
                ;\\
                f_{hQ}\left(v_e^{t-1},\tilde{H}^{t-1}\right)  & 
                \begin{aligned}
                    &\text{ if }
                w^{t-1}\leq \Tcal_Q-1\text{ and }w_0^{t-1}\leq 2^kk\tau\text{ or }\\
                &\text{  }w^{t-1}=\tau\text{ and } u^{t-1}\neq 1 \text{ and }w_0^{t-1}\leq 2^kk\tau;
                \end{aligned}\\
                \tilde{h}^{t-1} & \text{ otherwise.}
            \end{cases}
        \end{align}
        We show next that Equation~\eqref{eq:lem_iter_tildeht_value} holds. We first analyze the enumeration phase (first $2^kk\tau$ steps).
        \begin{itemize}
            \item In the initial step of this input loop, where $r_1(t)=s_1(t)=1$. This implies that $w^{t-1}=\tau$ and $u^{t-1}=1$.
            Then the value of $\tilde{h}$ is copied from $h$ in the last time step by the update rule \eqref{eq:lem_iter_tildeht_update}. By Claim~\ref{claim:lem_proof_H},
            \[
            \tilde{h}^t = h^{t-1} 
            =h_q\left(x_{:i_0(i_1(t-1)+1)+1}\right)^{\langle \Tcal_Q\rangle }
            = h_q\left(x_{:i_0(i_1(t))+1}\right)^{\langle \Tcal_Q\rangle }
            = h_q\left(x_{:i_0(t)+1}\right)^{\langle \Tcal_Q\rangle }.
            \]
            \item Next, we consider the initial step of each string loop, where $r_1(t)=s_1(t)=1$. This implies that $w^{t-1}=\tau$ and $u^{t-1}=1$. By Claim~\ref{claim:lem_proof_H},
            \[
            \tilde{h}^t = h^{t-1} = h_q\left(x_{:i_0(t)+1}\right)^{\langle \Tcal_Q\rangle}.
            \]
 
        \item We then look into a single string loop, that is, when $t\in [(i_1-1)\Tcal_U+(j_1-1)k\tau +1, (i_1-1)\Tcal_U+j_1 k\tau]$ for some $i_1\in[1,n]$ and $j_1\in[1,2^k]$. Throughout this interval, the following quantities remain fixed: $i_1(t)=i_1, i_0(t)=i_0(i_1), j_1(t) = j_1$. The initial value has been computed as
        \[
        \tilde{h}^{(i_1-1)\Tcal_U+(j_1-1)k\tau +1} = h_q\left(x_{:i_0(t)+1}\right)^{\langle \Tcal_Q \rangle }.
        \]
        We now use induction to show that Equation~\eqref{eq:lem_iter_tildeht_value} holds for all $t$ in this interval. The base case (initial time) has already been established. Assume that the equation holds at time $t_0$; we will show it continues to hold at time $t_0+1$. We consider several cases based on the update progress $s_1$.
        \begin{enumerate}
            \item If $s_1(t_0)=1$, we have $r_1(t_0+1)=r_1(t_0)$. By the induction hypothesis,
            \[
            \tilde{h}^{t_0} = h_q\left(x_{:i_0(t)+1}\cdot z^{(j_1(t_0))}_{:r_1(t_0)}\right)^{\langle \Tcal_Q\rangle }.
            \]
            Then by the update rule \eqref{eq:lem_iter_tildeht_update} and the value $v_e$ from Claim~\ref{claim:lem_proof_enumerater},
            \[
            \tilde{h}^{t_0+1} 
            = f_{hQ}\left(z^{(j_1(t_0))}_{r_1(t_0)},\tilde{H}^{t_0}\right)
            =  h_q\left(x_{:i_0(t_0)+1}\cdot z^{(j_1(t_0))}_{:r_1(t_0)+1}\right)^{\langle 1\rangle }
            = h_q\left(x_{:i_0(t_0+1)+1}\cdot z^{(j_1(t_0+1))}_{:r_1(t_0+1)+1}\right)^{\langle 1\rangle }.
            \]
            
            \item If $s_1(t_0) \in [1,\Tcal_Q-1]$, we have $r_1(t_0+1)=r_1(t_0), s_1(t_0+1)=s_1(t_0)+1$. By the induction hypothesis,
            \[
            \tilde{h}^{t_0}=
            h_q\left(x_{:i_0(t_0)+1}\cdot z^{(j_1(t_0))}_{:r_1(t_0)+1}\right)^{\langle s_1(t_0)-1 \rangle }.
            \]
            Then by the update rule \eqref{eq:lem_iter_tildeht_update} and the value $v_e$ from Claim~\ref{claim:lem_proof_enumerater},
            \begin{align*}
                \tilde{h}^{t_0+1}
            =& f_{hQ}\left( z^{(j_1(t_0))}_{r_1(t_0)}, \tilde{H}^{t_0}\right)\\
            =& h_q\left(x_{:i_0(t_0)+1}\cdot z^{(j_1(t_0))}_{:r_1(t_0)+1}\right)^{\langle s_1(t_0)\rangle }\\
            =&  h_q\left(x_{:i_0(t_0+1)+1}\cdot z^{(j_1(t_0+1))}_{:r_1(t_0)+1}\right)^{\langle s_1(t_0+1)-1\rangle }
            \end{align*}
            
            \item  If $s_1(t_0) \in [\Tcal_Q,\tau-1]$, we have $r_1(t_0+1)=r_1(t_0)$. By the induction hypothesis,
            \[
            \tilde{h}^{t_0}
            = h_q\left(x_{:i_0(t_0)+1}\cdot z^{(j_1(t_0))}_{:r_1(t_0)+1}\right)^{\langle \Tcal_Q-1\rangle }.
            \]
            Then by the update rule \eqref{eq:lem_iter_tildeht_update},
            \[
            \tilde{h}^{t_0+1}
            = \tilde{h}^{t_0}
            =h_q\left(x_{:i_0(t)+1}\cdot z^{(j_1(t_0))}_{:r_1(t_0)+1}\right)^{\langle \Tcal_Q-1\rangle }
            = h_q\left(x_{:i_0(t+1)+1}\cdot z^{(j_1(t_0+1))}_{:r_1(t_0)+1}\right)^{\langle \Tcal_Q-1\rangle }.
            \]
            \item If $s_1(t_0)=\tau$. This indicates that $t_0$ is the last step of processing the current digit. We have $w^{t_0} = s_1(t_0)=\tau$ and $s_1(t_0+1)=1$. By the induction hypothesis,
            \[
            \tilde{h}^{t_0}
            = h_q\left(x_{:i_0(t_0)+1}\cdot z^{(j_1(t_0))}_{:r_1(t_0)+1}\right)^{\langle \Tcal_Q-1\rangle }.
            \]
            There are two cases for updating, depending on whether $u^{t_0}=1$. If $u^{t_0}=1$, this implies that we are in the last time step of the current string loop. Thus $t_0+1$ is not in the interval we are considering.
            If $u^{t_0}\neq 1$, this implies that we are not in the last time step of the current string loop, and thus $j_1(t_0+1)=j_1(t_0)$ and $r_1(t_0+1)=r_1(t_0)+1$. By Equation~\eqref{eq:lem_iter_tildeht_update},
            $\tilde{h}$ is updated using $f_{hQ}$. Combined with Claim~\ref{claim:lem_proof_enumerater},
            \begin{align*}
                \tilde{h}^{t_0+1}
                = f_{hQ}\left(z^{(j_1(t_0))}_{r_1(t_0)},
                \tilde{H}^{t_0}\right)
                = h_q\left(x_{:i_0(t_0)+1}\cdot z^{(j_1(t_0))}_{:r_1(t_0)+1}\right)^{\langle \Tcal_Q\rangle }
                = h_q\left(x_{:i_0(t_0+1)+1}\cdot z^{(j_1(t_0+1))}_{:r_1(t_0+1)}\right)^{\langle \Tcal_Q\rangle }.
            \end{align*}
        \end{enumerate}
        Thus, by induction, we conclude that Equation~\eqref{eq:lem_iter_tildeht_value} holds for all $t$ in the interval $[(i_1-1)\Tcal_U+(j_1-1)k\tau +1, (i_1-1)\Tcal_U+j_1 k\tau]$. Since for each fixed pair $(i_1,j_1)$, the initialized value has been computed and the update rules are identical, the argument extends for all $i_1,j_1$, and hence holds for all corresponding time steps $t$ during the first $2^kk\tau$ steps.
        \end{itemize}
        Finally, we analyze the last $k\tau$ steps in each input loop, that is when $t\in [(i_1-1)\Tcal_U+2^kk\tau+1, (i_1-1)\Tcal_U+(2^k+1)k\tau]$ for some $i_1\in[1,n]$. When $t=(i_1-1)\Tcal_U+2^kk\tau+1$, by the update rule \eqref{eq:lem_iter_tildeht_update},
        $\tilde{h}$ is copied from $h$. By Claim~\ref{claim:lem_proof_H},
        \[
        \tilde{h}^t = h^{t-1} 
        = h_q\left(x_{i_0(t-1)+1}\right)^{\Tcal_Q}
        = h_q\left(x_{i_0(t)+1}\right)^{\Tcal_Q}
        \]
        The last equality holds because $i_0(t)$ does not change in the current $k\tau$ period. After that, by the update rule \eqref{eq:lem_iter_tildeht_update}, $\tilde{h}$ remains fixed. 
        
    \end{proof}
    Finally, we will construct the node set $R$ that replicates the structure of the node set $R_Q:=G_Q\setminus \{H_Q \bigcup \{v_{Q,\txtin}\}\}$. It includes the output node $v_{U,\txtout}$. 

    \begin{claim}[Node Set $R$]
        \label{claim:lem_proof_R}
        There exists an RNN with all nodes in Claim~\ref{claim:lem_proof_Htilde}, and a node set $R$ where each node in $R$ corresponds to a node in $R_Q$. Let $v_{U,\txtout}\in R$ be the node that corresponds to the output node $v_{Q,\txtout}\in R_Q$. Then its value
        \begin{align}
        \label{eq:lem_iter_vr_value}
            v_{U,\txtout}^t = \begin{cases}
            0 & \text{ if }s_1(t) = 1 ,r_1(t)=1,
             w_0^t\leq 2^kk\tau 
             \text{ or }
             w_0^t > 2^kk\tau
             ;\\
            v_{Q,\txtout}\left(x_{:i_0(t)+1}\cdot z^{(j_1(t))}_{:r_1(t)}\right)^{\langle \Tcal_Q\rangle}  & \text{ if }s_1(t) =1,
            r_1(t)\neq 1,
             w_0^t\leq 2^kk\tau ;\\
            v_{Q,\txtout}\left(x_{:i_0(t)+1}\cdot z^{(j_1(t))}_{:r_1(t)+1}\right)^{\langle s_1(t)-1\rangle}  & \text{ if }s_1(t)  \in [2,\Tcal_Q],
             w_0^t\leq 2^kk\tau;\\
             v_{Q,\txtout}\left(x_{:i_0(t)+1}\cdot z^{(j_1(t))}_{:r_1(t)+1}\right)^{\langle \Tcal_Q-1\rangle}  & \text{ if }s_1(t)  \in [\Tcal_Q+1,\tau],
             w_0^t\leq 2^kk\tau.
        \end{cases}
        \end{align}
    \end{claim}
    \begin{proof}[Proof of Claim~\ref{claim:lem_proof_R}]
        We construct a set of nodes $R$ by duplicating all nodes from the $R_Q$. Thus each node in $G_Q$ has a corresponding node in $\{v_e\} \bigcup \tilde{H}\bigcup R$. We add edges to the graph $G_U$ according to the connections in $G_Q$. For each node $r_Q\in R_Q$, who has update rule 
        $r_Q^t = f_{r}(N^{t-1}(r_Q))$.
        We set the corresponding node $v_R\in R$ to be initialized as $0$. The node $v_R$ follows the exact same RUN-HOLD-HOLD operation timing as the nodes $\tilde{H}$ defined in Claim~\ref{claim:lem_proof_Htilde}, as illustrated in Figure~\ref{fig:rnn_construct_setHtilde}. The RUN step uses the function $f_r$, and the LOAD step simply loads a constant value $0$.
        Formally,
        \begin{align}
            \label{eq:lem_iter_vr_update}
            v_R^t = 
            \begin{cases}
                0 & \text{ if }w^{t-1}=\tau \text{ and }u^{t-1}=1;\\
                f_r\left(N^{t-1}\left(v_R\right)\right) & 
                \begin{aligned}
                    &\text{ if }
                w^{t-1}\leq \Tcal_Q-1\text{ and }w_0^{t-1}\leq 2^kk\tau\text{ or }\\
                &\text{  }w^{t-1}=\tau\text{ and } u^{t-1}\neq 1 \text{ and }w_0^{t-1}\leq 2^kk\tau;
                \end{aligned}\\
                v_R^{t-1} & \text{ otherwise. }
            \end{cases}
        \end{align}

        Denote $f_o$ as the transition function of the output node $v_{Q,\txtout}$.
        We will analyze $v_{U,\txtout}$ using the update rule \eqref{eq:lem_iter_vr_update}.

        Note that, by the definition of hidden node set, $v_{Q,\txtout}^{i\Tcal_Q} = \psi_{Q,H}\left(v_{Q,\txtin}^{(i-1)\Tcal_Q+1}, H_Q^{(i-1)\Tcal_Q+1}\right)$. This implies that the output at time $i\Tcal_Q$ depends solely on the input and hidden nodes at time $(i-1)\Tcal_Q+1$ and is unaffected by the values of nodes in $R_Q$ at that time. In other words, regardless of how $R_Q$ is initialized at any time $t \equiv 1 \Mod{\Tcal_Q}$, running the RNN for $\Tcal_Q$ steps ensures that $v_{Q,\txtout}^{i\Tcal_Q}$ is computed correctly. Similarly, the initialization of $R$ at the beginning of each digit loop does not affect the output $v_{U,\txtout}$.

        For each input loop, we first analyze the first $2^kk\tau$ steps.
        \begin{itemize}
            \item In the initial step of the input loop, the value of $v_{U,\txtout}$ is set to be $0$ by the update rule \eqref{eq:lem_iter_vr_update}.
            \item Next, we consider the initial step of each string loop, where $r_1(t)=s_1(t)=1$. This implies that  $w^{t-1}=\tau$ and $u^{t-1}=1$. By the update rule \eqref{eq:lem_iter_vr_update}, the output node $v_{U,\txtout}$ is also set to be $0$.
            \item We then look into a single string loop, that is, when $t\in [(i_1-1)\Tcal_U+(j_1-1)k\tau+1,(i_1-1)\Tcal_U+j_1k\tau]$ for some $i_1\in[1,n]$ and $j_1\in[1,2^k]$. Throughout this interval, the following quantities remain fixed: $i_1(t)=i_1, i_0(t)=i_0,j_1(t)=j_1$.
            The initial value of $v_{U,\txtout}$ is $0$. We now use induction to show that Equation~\eqref{eq:lem_iter_vr_value} holds for all $t$ in this interval. The base case has already been established. Assume that the equation holds at time $t_0$; we will show it continues to hold at time $t_0+1$. We consider several cases based on the update progress $s_1$.
            \begin{enumerate}
                \item If $s_1(t_0)=1$, we have $r_1(t_0+1)=r_1(t_0)$.
                by induction hypothesis,
                \[
                v_{U,\txtout}^{t_0}=0.
                \]
                By the update rule \eqref{eq:lem_iter_vr_update}, it will update using the transition function $f_o$. By Claim~\ref{claim:lem_proof_enumerater} and Claim~\ref{claim:lem_proof_Htilde}, the node $v_e^{t_0}=z_{r_1(t_0)}^{\left(j_1(t_0)\right)}$ and each node $\tilde{h}$ in the node set $\tilde{H}$ has value
                \[
                \tilde{h}^{t_0} = h_q\left(x_{:i_0(t_0)+1}\cdot z^{(j_1(t_0))}_{:r_1(t_0)}\right)^{\langle \Tcal_Q\rangle } .
                \]
                Thus we have
                \begin{align*}
                    v_{U,\txtout}^{t_0+1}
                    =&f_o\left(N^{t_0}(v_{U,\txtout})\right)\\
                    =& v_{Q,\txtout}\left( x_{:i_0(t_0)+1}\cdot z^{(j_1(t_0))}_{:r_1(t_0)+1}\right)^{\langle 1\rangle}\\
                    = &v_{Q,\txtout}\left( x_{:i_0(t_0+1)+1}\cdot z^{(j_1(t_0+1))}_{:r_1(t_0+1)+1}\right)^{\langle 1\rangle}.
                \end{align*}

                \item If $s_1(t_0) \in [1,\Tcal_Q-1]$, we have $r_1(t_0+1)=r_1(t_0)$, $s_1(t_0+1)=s_1(t_0)+1$. By the induction hypothesis,
                \[
                v_{U,\txtout}^{t_0} = v_{Q,\txtout}\left(x_{:i_0(t_0)+1}\cdot z^{(j_1(t_0))}_{:r_1(t_0)+1}\right)^{\langle s_1(t_0)-1\rangle}  
                \]
                By Claim~\ref{claim:lem_proof_enumerater} and Claim~\ref{claim:lem_proof_Htilde}, the node $v_e^{t_0}=z_{r_1(t_0)}^{\left(j_1(t_0)\right)}$ and each node $\tilde{h}$ in the node set $\tilde{H}$ has value
                \[
                \tilde{h}^{t_0} = h_q\left(x_{:i_0(t)+1}\cdot z^{(j_1(t_0))}_{:r_1(t_0)+1}\right)^{\langle s_1(t_0)-1\rangle } .
                \]
                Thus, by the update rule \eqref{eq:lem_iter_vr_update}, $v_{U,\txtout}$ will update using the transition function $f_o$. 
                \begin{align*}
                    v_{U,\txtout}^{t_0+1}
                    =&f_o\left(N^{t_0}(v_{U,\txtout})\right)\\
                    =& v_{Q,\txtout}\left(x_{:i_0(t_0)+1}\cdot z^{(j_1(t_0))}_{:r_1(t_0)+1}\right)^{\langle s_1(t_0)\rangle} \\
                    = &v_{Q,\txtout}\left( x_{:i_0(t_0+1)+1}\cdot z^{(j_1(t_0+1))}_{:r_1(t_0+1)+1}\right)^{\langle s_1(t_0+1)-1\rangle}
                \end{align*}

                \item If $s_1(t_0)\in[\Tcal_Q,\tau-1]$, we have $r_1(t_0+1)=r_1(t_0)$. By the induction hypothesis,
                \[
                v_{U,\txtout}^{t_0} =
                 v_{Q,\txtout}\left(x_{:i_0(t_0)+1}\cdot z^{(j_1(t_0))}_{:r_1(t_0)+1}\right)^{\langle \Tcal_Q-1\rangle} 
                \]
                Thus, by the update rule \eqref{eq:lem_iter_vr_update}, $v_{U,\txtout}$ will retain the same.
                \begin{align*}
                    v_{U,\txtout}^{t_0+1}
                    =&v_{U,\txtout}^{t_0}\\
                    =&v_{Q,\txtout}\left(x_{:i_0(t_0)+1}\cdot z^{(j_1(t_0))}_{:r_1(t_0)+1}\right)^{\langle \Tcal_Q-1\rangle} \\
                    =& v_{Q,\txtout}\left(x_{:i_0(t_0+1)+1}\cdot z^{(j_1(t_0+1))}_{:r_1(t_0+1)+1}\right)^{\langle \Tcal_Q-1\rangle}
                \end{align*}

                \item If $s_1(t_0)=\tau$. This indicates that $t_0$ is the last step of processing the current digit. We have $w^{t_0} = s_1(t_0)=\tau$ and $s_1(t_0+1)=1$. By the induction hypothesis, 
                \[
                v_{U,\txtout}^{t_0}
                =v_{Q,\txtout}\left(x_{:i_0(t_0)+1}\cdot z^{(j_1(t_0))}_{:r_1(t_0)+1}\right)^{\langle \Tcal_Q-1\rangle}.
                \]
                There could be two cases of updating depending on whether $u^{t_0}=1$. If $u^{t_0}=1$, this implies that we are in the last time step of the current string loop. Thus, $t_0+1$ is not in the interval we are considering.
                If $u^{t_0}\neq 1$, this implies that we are not in the last time step of the current string loop, and thus $j_1(t_0+1)=j_1(t_0)$ and $r_1(t_0+1)=r_1(t_0)+1$.
                By Claim~\ref{claim:lem_proof_enumerater} and Claim~\ref{claim:lem_proof_Htilde}, the node $v_e^{t_0}=z_{r_1(t_0)}^{(j_1(t_0))}$ and each node $\tilde{h}$ in the node set $\tilde{H}$ has value
                \[
                \tilde{h}^{t_0} = h_q\left(x_{:i_0(t)+1}\cdot z^{(j_1(t_0))}_{:r_1(t_0)+1}\right)^{\langle \Tcal_Q-1\rangle } .
                \]
                Thus, by the update rule \eqref{eq:lem_iter_vr_update}, $v_{U,\txtout}$ will update using $f_o$.
                \begin{align*}
                    v_{U,\txtout}^{t_0+1}
                    =&f_o\left(N^{t_0}(v_{U,\txtout})\right)\\
                    =& v_{Q,\txtout}\left(x_{:i_0(t_0)+1}\cdot z^{(j_1(t_0))}_{:r_1(t_0)+1}\right)^{\langle \Tcal_Q\rangle}\\
                    =& v_{Q,\txtout}\left(x_{:i_0(t_0+1)+1}\cdot z^{(j_1(t_0+1))}_{:r_1(t_0+1)}\right)^{\langle \Tcal_Q\rangle}.
                \end{align*}
            \end{enumerate}
            Thus, by induction, we conclude that Equation~\eqref{eq:lem_iter_vr_value} holds for all $t$ in the interval $[(i_1-1)\Tcal_U+(j_1-1)k\tau +1, (i_1-1)\Tcal_U+j_1 k\tau]$. Since for each fixed pair $(i_1,j_1)$, the initialized value has been computed and the update rules are identical, the argument extends for all $i_1,j_1$, and hence holds for all corresponding time steps $t$ during the first $2^kk\tau$ steps.
        \end{itemize}
        Finally, we analyze the last $k\tau$ steps in each input loop, that is when $t\in [(i_1-1)\Tcal_U+2^kk\tau+1, (i_1-1)\Tcal_U+(2^k+1)k\tau ]$ for some $i_1\in [1,n]$. When $t=(i_1-1)\Tcal_U+2^kk\tau+1$, we know $w^{t-1}=\tau$ and $u^{t-1}=1$. Thus by the update rule~\eqref{eq:lem_iter_vr_update}, we know $v_{U,\txtout}^t=0$. Followed by this, either the node $v_{U,\txtout}$ is reset to $0$ (when $w^{t-1}=\tau$ and $u^{t-1}=1$), or it retains its value $0$. Thus, during the $k\tau$ steps, we have $v_{U,\txtout}^t=0$.

    \end{proof}

    By Claim~\ref{claim:lem_proof_R}, we know for for $i_1\in[1,n], j_1\in[1,2^k], r_1\in[1,k]$, for any time
    $t\in [(i_1-1)\Tcal_U+(j_1-1)k\tau+(r_1-1)\tau+\Tcal_Q,(i_1-1)\Tcal_U+(j_1-1)k\tau+r_1\tau]$, the output node value satisfies that
    \begin{align}
        \label{eq:proof_lemma_vuout_i0*1}
        v_{U,\txtout}^t =  v_{Q,\txtout}\left(x_{:i_0(i_1)+1}\cdot z^{(j_1)}_{:r_1+1}\right)^{\langle \Tcal_Q-1\rangle } 
    =g_Q\left(x_{:i_0(i_1)+1}\cdot z^{(j_1)}_{:r_1+1}\right).
    \end{align}

    Thus, we have proved our lemma for the case when $i_0^*=0$. In the following claim, we will generalize it to any $i_0^* \in[0,k-1]$ by a slight modification on the RNN.

    \begin{claim}[Initial Phase $i_1\leq i_0^*$]
        \label{claim:general_i0*}
        For any $i^* \in[0,k-1]$, there exists an RNN $U$ with output node satisfies that
        \begin{enumerate}
            \item For $i_1\leq i_0^*$ and $t\in [(i_1-1)\Tcal_U + \Tcal_Q,  i_1\Tcal_U]$,
            \[
            v_{U,\txtout}^t = g_Q\left(x_{:i_1+1}\right).
            \]
            \item For $i_1\geq i_0^*+1, 1\leq j_1\leq 2^k, 1\leq r_1\leq k$ and  $t \in [(i_1-1)\Tcal_U + (j_1-1)k\tau + (r_1-1)\tau + \Tcal_Q, (i_1-1)\Tcal_U + (j_1-1)k\tau + r_1\tau]$,
            \[
            v_{U,\txtout}^t = g_Q\left(x_{:i_0(i_1)+1}\cdot z^{(j_1)}_{:r_1+1}\right).
            \]
        \end{enumerate}
    \end{claim}
    \begin{proof}[Proof of Claim~\ref{claim:general_i0*}]
        We modify the RNN in Claim~\ref{claim:lem_proof_R} by adding a counter node, changing the initialization of node $u_0$,
        and changing the updating functions of the node sets $H$, $\tilde{H}$ and $R$. 

        We begin by adding a counter node $v_c \in [1, i_0^*+1]$ that increments by $1$ for each new input and remains at $i_0^*+1$ thereafter. It is initialized as $1$ and updated as
    \begin{align}
    \label{eq:lem_iter_vc_update}
        v_c^t = \begin{cases}
            v_c^{t-1}+1 & \text{ if }v_c^{t-1}\leq i_0^* \text{ and }w_0^{t-1} = \Tcal_U;\\
            i_0^*+1 & \text{ if }v_c^{t-1}> i_0^* \text{ and }w_0^{t-1} = \Tcal_U;\\
            v_c^{t-1} & \text{ otherwise.}
        \end{cases}
    \end{align}
    Then it has value
    \begin{align}
        v_c^t = \begin{cases}
            i_1 &\text{ if }t\in[(i_1-1)\Tcal_U+1, i_1\Tcal_U] \text{ and }i_1\leq i_0^*+1;\\
            i_0^*+1&\text{ otherwise.}
        \end{cases} 
    \end{align}
    We then change the initialization of $u_0$ to be 
    \[
    u_0^1 = \begin{cases}
        1 & \text{ if }i_0^* = 0;\\
        1-i_0^* + k & \text{ otherwise.}
    \end{cases}
    \]
    This ensures that when $i_1(t)=i_0^*+1$, the node $u_0^t=1$, and matches the initialization in Claim~\ref{claim:lem_proof_counters}.

    Next, we revise the update rule for the node set $H$. Instead of storing the hidden states with the anchor string $x_{:i_0(t)+1}$. When $i_1\leq i_0^*$, we let $H$ stores the hidden states with the input string $x_{:i_1(t)+1}$. For each $h\in H$, denote $f_{hQ}$ as the transition function of its corresponding node $h_Q\in H_Q$, and $f_{h}$ as the transition function of $h$ defined in Claim~\ref{claim:lem_proof_H} (Equation~\eqref{eq:lem_iter_h_update}).
    Let $h$ initialize with the same value as $h_Q\in H_Q$.
    Then we define its new update rule as
    \begin{align}
        \label{eq:lem_iter_h_new_update}
        h^t = \begin{cases}
            f_{hQ}(v_{U,\txtin}^{t-1},\, H^{t-1}) & \text{ if }w_0^{t-1}\leq \Tcal_Q \text{ and }v_c^{t-1}\leq i_0^*;\\
            h^{t-1} & \text{ if }w_0^{t-1}\in [\Tcal_Q+1,\Tcal_U]\text{ and }v_c^{t-1}< i_0^* ;\\
            f_h(y_j^{t-1},\, N^{t-1}(h))& \text{ otherwise.}
        \end{cases}
    \end{align}
    Now we compute its value. We first analyze the value of $h$ for an input loop where $i_1\leq i_0^*$. During each input loop,
    $h$ is updated using $h_Q$ for $\Tcal_Q$ time steps. That is when $w_0^t \in [2,\Tcal_Q+1]$,
    \begin{align}
        \label{eq:lem_iter_h_new_value1}
        h^t = h_q\left(x_{:i_1(t)+1}\right)^{\langle w_0^t-1\rangle}.
    \end{align}
    Followed by this, it retains its value  $h_q(x_{:i_1(t)+1})^{\langle \Tcal_Q\rangle}$ until the first step of the next input loop.

    When $i_1 \geq i_0^*$, the initialization of $h$ is 
    \[
    h^{i_0^*\cdot \Tcal+1}
    =h_q\left(x_{:i_0^*+1}\right)^{\langle \Tcal_Q\rangle}
    =h_q\left(x_{:i_0(i_0^*\cdot \Tcal_U+1)+1}\right)^{\langle \Tcal_Q\rangle}.
    \]
    Followed by this, $h$ is updated using $f_{h}$.
    Note that the counter nodes $w_0,w,u$ cycle with periods $\Tcal_U, \tau$ and $k$, respectively. So within each input loop spanning $\Tcal_U$ time steps, these three nodes complete full cycles and thus their values remain unchanged for each input loop. Also, $u_0$ achieves the same initialization as in Claim~\ref{claim:lem_proof_H}. Thus, $h$ has the same value as in Claim~\ref{claim:lem_proof_H}.

    Next, we will revise $\tilde{H}$ such that when $i_1(t) \geq i_0^*+1$, it has the correct initialization. Specifically,
    \[
    \tilde{H}^{i_0^*\cdot \Tcal_U+1}
    =
    \tilde{H}_q\left(x_{:i_0^*+1}\right)^{\langle \Tcal_Q\rangle}.
    \]
    For each node $\tilde{h}\in\tilde{H}$, denote $f_{\tilde{h}}$ as its original transition function in Claim~\ref{claim:lem_proof_Htilde}. Then its new transition function is
    \begin{align}
        \tilde{h}=
        \begin{cases}
            h^{t-1} &\text{ if }v_c^{t-1}\leq i_0^*;\\
            f_{\tilde{h}}\left(N^{t-1}(\tilde{h})\right) &\text{ otherwise.}
        \end{cases}
    \end{align}
    This ensures its initialization at time $i_0^*\cdot \Tcal_U+1$, and the following update rule is the same as Claim~\ref{claim:lem_proof_Htilde}.

     Finally, we revise the update rule for each node $v_R\in R$. Since the input node $v_{U,\txtin}$, the node set $H$, and the node set $R$ consist of a graph equivalent to the graph $G_Q$, we add edges between them that correspond to their relationship in $G_Q$. We denote by $f_{v_R,1}$ and $N_1(v_R)$ the transition function and the incoming neighbors used to update the node $v_R$ based on the input node $v_{U,\txtin}$ and the hidden node set $H$. Similarly, we denote by $f_{v_R,2}$ and $N_2(v_R)$ the transition function and the incoming neighbors used to update the node $v_R$ as specified in Claim~\ref{claim:lem_proof_R}. Then the new update rule of node $v_R$ is defined as
    \begin{align}
    \label{eq:lem_iter_vR_new_update}
        v_R^t = 
        \begin{cases}
            f_{v_R,1}\left(N_1^{t-1}(v_R)\right) & \text{ if }w_0^{t-1}\leq \Tcal_Q-1\text{ and }v_c^{t-1}\leq i_0^*;\\
            v_R^{t-1} & \text{ if }w_0^{t-1}\in[\Tcal_Q, \Tcal_U-1] \text{ and }v_c^{t-1}\leq i_0^*;\\
            0 & \text{ if }w_0^{t-1}= \Tcal_U\text{ and }v_c^{t-1}\leq i_0^*;\\
            f_{v_R,2}\left(N_2^{t-1}(v_R)\right) & \text{ otherwise.}
        \end{cases}
    \end{align}
    Since a constant number of compositions of transition functions is a valid transition, Equation~\eqref{eq:lem_iter_vR_new_update} is a valid transition function. 
    
    Next, we will compute the value of $v_{U,\txtout}$. 
    
    \begin{itemize}
        \item When $w_0^t=1$ and $i_1(t)\leq i_0^*$, that is, when a new input token is fed into the RNN. Since $w_0^{t-1}=\Tcal_U$ and $v_c^{t-1}\leq i_1(t)\leq i_0^*$, by the update rule \eqref{eq:lem_iter_vR_new_update},
        \[
        v_{U,\txtout}^t = 0.
        \]
        
        \item When $w_0^t \in [2,\Tcal_Q]$ and $i_1(t) \leq i_0^*$, we have $w_0^{t-1} \in [1,\Tcal_Q-1]$ and $v_c^{t-1}=i_1(t-1)\leq i_0^*$. Then by the update rule \eqref{eq:lem_iter_vR_new_update},
        \[
        v_{U,\txtout}^t = f_{v_{U,\txtout},1}\left(N_1^{t-1}(v_{U,\txtout})\right).
        \]
        By Equation~\eqref{eq:lem_iter_h_new_value1}, each node $h$ in
        the node set $H$ has value
        \[
        h^t = h_q\left(x_{:i_1(t)+1}\right)^{\langle w_0^t-1\rangle}.
        \]
        Combined with the input node value $v_{U,\txtin}^t = x_{i_1(t)}$, we have
        \begin{align}
            \label{eq:claim_general_case2_vout}
            v_{U,\txtout}^t = v_{Q,\txtout}\left(x_{:i_1(t)+1}\right)^{\langle w_0^t-1\rangle}.
        \end{align}

        \item When $w_0^t \in [\Tcal_Q+1, \Tcal_U]$ and $v_c^{t-1}\leq i_0^*$, we have $w_0^{t-1} \in [\Tcal_Q,\Tcal_U-1]$ and $v_c^{t-1}=i_1(t-1)\leq i_0^*$. Then by the update rule \eqref{eq:lem_iter_vR_new_update}, the node $v_{U,\txtout}$ retains the same. Thus,
        \begin{align}
            \label{eq:claim_general_case3_vout}
            v_{U,\txtout}^t = v_{Q,\txtout}\left(x_{:i_1(t)+1}\right)^{\langle \Tcal_Q-1\rangle}.
        \end{align}

        \item When $w_0^t=1$ and $i_1(t)=i_0^*+1$, this is the first time step for the input $x_{i_0^*+1}$. Since $w_0^{t-1}=\Tcal_U$ and $i_1(t-1)=i_0^*$, by the update rule \eqref{eq:lem_iter_vR_new_update}, the node $v_{U,\txtout}$ has value
        \[
        v_{U,\txtout}^t=0.
        \]

        \item Otherwise, we have $t\geq i_0^*\cdot \Tcal_U+2$. Since $v_c^{t-1}\geq i_0^*+1$, the node $v_{U,\txtout}$ will update use $f_{v_{U,\txtout},2}$. Since its initial value is $0$ at time $t=i_0^*\cdot \Tcal_U+1$, which is the same as the initialization in Claim~\ref{claim:lem_proof_R}. Furthermore, all counter nodes have the same initialization. Thus, by the same update rule, it will have the same value.

    \end{itemize}
    To sum up, when $t\leq i_0^*\cdot \Tcal_U$, by Equation~\eqref{eq:claim_general_case2_vout} and Equation~\eqref{eq:claim_general_case3_vout}, for any $t\in [(i_1-1)\Tcal_U+\Tcal_Q, i_1\Tcal_U]$, the output node
    \[
    v_{U,\txtout}^t = v_{Q,\txtout}\left(x_{:i_1(t)+1}\right)^{\langle \Tcal_Q-1\rangle}
    =g_Q\left(x_{:i_1+1}\right).
    \]
    For $t\geq i_0^*\cdot \Tcal_U+1$, by Equation~\eqref{eq:proof_lemma_vuout_i0*1}, for $t\in [(i_1-1)\Tcal_U+(j_1-1)k\tau+(r_1-1)\tau+\Tcal_Q,(i_1-1)\Tcal_U+(j_1-1)k\tau+r_1 \tau]$, the output node value satisfies that
    \[
    v_{U,\txtout}^t 
    =g_Q\left(x_{:i_0(i_1)+1}\cdot z^{(j_1)}_{:r_1+1}\right).
    \]
  
    \end{proof}

    Next, we will bound the size and hidden node set size of the RNN $U$.

    \begin{claim}[Complexity]
        \label{claim:complexity}
        The constructed RNN $U$ has a size of $|Q|+|H_Q|+2k+6$ and a hidden node set size of $|H_Q|+2k+6$.
    \end{claim}
    \begin{proof}[Proof of Claim~\ref{claim:complexity}]

        By the construction, the total number of nodes in the RNN $U$ is 
        \[
        |U| = 5 +|Y|+ |E|+|H|+|\tilde{H}| + |R| 
        =5+k+(k+1)+|H_Q|+|Q|
        = |Q|+|H_Q|+2k+6.
        \]
        We next show that the node set $H_U:=\{w_0,u_0,w,u,v_c\} \bigcup Y \bigcup E \bigcup H  $ is a hidden node set of the constructed RNN $U$.
        By update rules \eqref{eq:lem_iter_vpt_update}, \eqref{eq:lem_iter_u0_update}, \eqref{eq:lem_iter_wt_update}, \eqref{eq:lem_iter_ut_update}, \eqref{eq:lem_proof_y_set_update},
        \eqref{eq:lem_iter_xer_update}, \eqref{eq:lem_iter_ve_update},
        \eqref{eq:lem_iter_h_update}, \eqref{eq:lem_iter_vc_update}, \eqref{eq:lem_iter_h_new_update}
        , each node in $H_U$ is updated using solely the values of $\{v_{U,\txtin}\}\bigcup H_U$ in the previous time step. Thus, the first requirement for the hidden node set is satisfied.

        It remains to show that the output node at the time of interest depends only on the input node and the node set $H_U$ at the beginning of the current loop. We consider two cases based on whether $i_1 \leq i_0^*$.

        If $i_1\leq i_0^*$, $v_{U,\txtout}$ is updated using $v_{U,\txtin}$ and $H$. 
        By Claim~\ref{claim:general_i0*},
        $H$ has the same update rule as $H_q$ for $w_0^{t-1} \leq \Tcal_Q$; $R$ and $R_q$ is updated using the same transition function when $w_0^{t-1}\leq \Tcal_Q-1$. When $w_0^{t-1}\in[\Tcal_Q, \Tcal_U-1]$, $R$ retains its value.
        Thus,
        \[
        v_{U,\txtout}^{i_1\Tcal_U} = v_{U,\txtout}^{(i_1-1)\Tcal_U+\Tcal_Q}
        = \psi_{Q,H}\left(v_{U,\txtin}^{(i_1-1)\Tcal_U+1},\, H^{(i_1-1)\Tcal_U+1}\right).
        \]
        If $i_1\geq i_0^*+1$, for any $i_1\in[i_0^*+1,n],j_1\in[1,2^k],r_1\in[1,k]$, denote $T_0 = (i_1-1)\Tcal_U+(j_1-1)k\tau+(r_1-1)\tau$. We will show that for $t\in [T_0+\Tcal_Q, T_0+\tau]$, the output node $v_{U,\txtout}^t$ depends only on $H_U^{T_0+1}$ and $x_{i_1}$. According to the update rule Equation~\eqref{eq:lem_iter_vr_update} in Claim~\ref{claim:lem_proof_R} and the property of hidden node set,
        \[
        v_{U,\txtout}^t = \psi_{Q,H}\left(v_e^{T_0+1}, \tilde{H}^{T_0+1}\right).
        \]
    For $\tilde{H}$, by its update rule \eqref{eq:lem_iter_tildeht_update}, each $\tilde{H}^t$ can be expressed using $w^{t-1},u^{t-1},w_0^{t-1},E^{t-1},H^{t-1}$ and $\tilde{H}^{t-1}$. Note that $\tilde{H}^{T_0+1}=H^{T_0}$ from Equation~\eqref{eq:lem_iter_tildeht_update}, and also $H^{T_0} = H^{T_0+1}$ from Equation~\eqref{eq:lem_iter_ht_value}. Thus we know $\tilde{H}^{T_0+1}=H^{T_0+1}$. Thus there exists a function $\psi_2:\{0,1\} \times \R^{|H_U|}\rightarrow \{0,1\}$ such that
    \[
    \tilde{H}^{t} = \psi_2\left(x_{i_1}, H_U^{T_0+1}\right) = \psi_2\left(v_{U,\txtin}^{T_0+1}, H_U^{T_0+1}\right).
    \]
    Finally, to bound the size of $H_U$,
    \[
    |H_U| = 5 + |Y|+|E|+|H| = 5+k+(k+1)+|H_Q| = |H_Q|+2k+6.
    \]
    \end{proof}

\end{proof}

\subsubsection{Proofs for the RNN Component Functions $f_1$ (Lemma~\ref{lemma:rnn_prod_q}) , $f_2$
(Lemma~\ref{lemma:rnn_exp_distinguisher}) and $g_1,g_2$
(Lemma~\ref{lemma:rnn_indicator_first_digit})}
\label{sec:proof_lemma789}
Next, we will prove Lemma~\ref{lemma:rnn_prod_q}, Lemma~\ref{lemma:rnn_exp_distinguisher} and Lemma~\ref{lemma:rnn_indicator_first_digit}, which show the construction of RNNs that realize functions $f_1,f_2$ and $g_1,g_2$ respectively. These proofs are built based on Lemma~\ref{lemma:iterate_compose}.
\lemmafone*
\begin{proof}
    By Lemma~\ref{lemma:iterate_compose}, there exists an RNN $U$ with RNN-time $\Tcal_{U}=(2^k+1)k\tau$, size $|Q|+|H_Q|+2k+6$ and hidden node set size $|H_Q|+2k+6$. Its output node $v_q$ satisfies that 
    \begin{itemize}
        \item If $1\leq i\leq i_0^*$,
        \begin{align}
            v_q^{i\Tcal_U-2} = q(x_i \mid x_{:i}).
        \end{align}
        \item If $i > i_0^*$, for any $1\leq j\leq 2^k$, $1\leq r\leq k$,
        \begin{align}
            \label{eq:lem_f1_old_v_value}
            v_q^{(i-1)\Tcal_U+(j-1)k\tau+r\tau-2}
            = q(z^{(j)}_r \mid x_{:i_0(i)+1}\cdot z^{(j)}_{:r}).
        \end{align}
    \end{itemize}
    We now augment the RNN $U$ with another node $v_{U,\txtout}$, and let it be the output node. It is initialized as $1$, and updates as follows.
    \begin{itemize}
        \item When $1\leq i\leq i_0^*$, the node $v_{U,\txtout}$ copies the value from $v_q$ at the third-to-last time step of the current input loop. At all other time steps, it retains its previous value.
        \item When $i > i_0^*$, during each digit loop,
        the node $v_{U,\txtout}$ remains fixed except at the third-to-last time step, at which point it multiplies its current value with the value of $v_q$ at the last time step.
    \end{itemize}
    Recall the input loop time-step counter node $w_0\in[1,\Tcal_U]$, the digit loop time-step counter $w\in[1,\tau]$, the digit counter $u\in[1,k]$
    defined in Claim~\ref{claim:lem_proof_counters} and the node $v_c \in[1,i_0^*+1]$ defined in Claim~\ref{claim:general_i0*}. The output node is defined as
    \begin{align}
        v_{U,\txtout}^t = \begin{cases}
            v_q^{t-1}
            & \text{ if }w_0^{t-1}=\Tcal_U-2 \text{ and } v_c^{t-1}\leq i_0^*;\\
            1
            & \text{ if }w^{t-1}=\tau \text{ and }
            u^{t-1}=1 \text{ and } 
            v_c^{t-1}\geq i_0^*;\\
            v_{U,\txtout}^{t-1}\cdot v_q^{t-1}
            & \text{ if }w^{t-1} = \tau-2 \text{ and }
            v_c^{t-1}> i_0^*;\\
            v_{U,\txtout}^{t-1}& \text{ otherwise.}
        \end{cases}
    \end{align}
    This is a valid transition function by Lemma~\ref{lemma:transition_function}.
    
    When $i\leq i_0^*$, we have $v_c^{t-1}\leq i\leq i_0^*$. Thus, by the update rule
    \[
    v_{U,\txtout}^{i\Tcal_U-1} = v_{U,\txtout}^{i\Tcal_U-2}=q(x_i\mid x_{:i}).
    \]
    
    When $i> i_0^*$, since $w^{t-1}=\tau,u^{t-1}=1$ implies that $t$ is in the first step of some string loop. That is, $t=(i-1)\Tcal_U+(j-1)k\tau+1$. By the update rule,
    \[
    v_{U,\txtout}^{(i-1)\Tcal_U+(j-1)k\tau+1}=1.
    \]
    It retains its value except when $w^{t-1}=\tau-2$. Combining with Equation~\eqref{eq:lem_f1_old_v_value}, we have
    \[
        v_{U,\txtout}^{(i-1)\Tcal_U+(j-1)k\tau+r\tau-1}
            = \prod_{r_1=1}^r q(z^{(j)}_{r_1} \mid x_{:i_0(i)+1}\cdot z^{(j)}_{:r_1})
            = q(z^{(j)}_{:r+1} \mid x_{:i_0(i)+1}).
    \]
    Therefore, when $r=k$, we have
    \[
     v_{U,\txtout}^{(i-1)\Tcal_U+jk\tau-1}
     =q(z^{(j)}_{:k+1} \mid x_{:i_0(i)+1})
     =q(z^{(j)} \mid x_{:i_0(i)+1}).
    \]
    The size of the RNN $U$ is 
    \[
    |U|
    =|Q|+|H_Q|+2k+7.
    \]
    Since $v_{U,\txtout}$ is reset to $1$ for each string loop, the output node only depends on the hidden node set of $H$. Thus $H_{U'}$ is also a hidden node set of $U$. The hidden node set size is
    \[
    |H_U|=
    |H_{U'}|=|H_Q|+2k+6.
    \]
    
\end{proof}

\lemmaftwo*
\begin{proof}
    By Lemma~\ref{lemma:iterate_compose}, there exists an RNN $W$ with RNN-time $\Tcal_U=(2^k+1)k\tau$ such that when the input index $i>i_0^*$, for any $j\in[1,2^k]$, and $t\in [(i-1)\Tcal_W+(j-1)k\tau+(k-1)\tau+\Tcal_Q,  (i-1)\Tcal_W+  jk\tau]$,
    the output node $v_{W',\txtout}$ has value
    \[
    v_{W',\txtout}^t = d\left(i_0(i)+1, x_{:i_0(i)+1}\cdot z^{(j)}\right).
    \]
    Now we add node $v_{W,\txtout}$ and treat it as the new output node of the RNN $W$. It is updated as
    \[
    v_{W,\txtout}^t = \exp(-\alpha v^{t-1}_{W',\txtout}).
    \]
    For any $x\in\{0,1\}$, by Lemma~\ref{lemma:transition_function}, the function $f(x)=e^{-\alpha x}$ is a transition function.
    Thus for any $t=(i-1)\Tcal_W+jk\tau-1$,
    \[
    v_{W,\txtout}^t = \exp\left(-\alpha d\left(i_0(i)+1, x_{:i_0(i)+1}\cdot z^{(j)}\right)\right).
    \]
    Since $v_{W,\txtout}$ only depends on $v_{W',\txtout}$ in the previous time step. Thus, $H_{W}$ remains a hidden node set of $W$ after modification. The size and hidden node set size of $W$ is
    \[
    |W|=|D|+|H_D|+2k+7,\quad
    |H_W|= |H_D|+2k+6.
    \]

\end{proof}

\lemmag*
\begin{proof}
    To construct the RNN $O$, we include the input node $v_{O,\txtin}$, counter nodes $w_0,u_0,w,u$ in Claim~\ref{claim:lem_proof_counters}, the input storage node set $Y$ in Claim~\ref{claim:lem_proof_input_set},
    the enumerator node set $E$ in Claim~\ref{claim:lem_proof_enumerater}
    from Lemma~\ref{lemma:iterate_compose}, a node set $\{w_l\}_{l=1}^k$, nodes $v_1,v_2$
    and the output node $v_{O,\txtout}$. By Equation~\eqref{eq:lem_proof_y_set_value}, for any time $t=(i-1)\Tcal_O+(j-1)k\tau+2$, the input storage nodes
    \begin{align}
        y_l^t = \begin{cases}
            x_{i_0(i)+l}&\text{ if }i-i_0(i) \geq l;\\
            0&\text{ otherwise.}
        \end{cases}
    \end{align}
    By Equation~\eqref{eq:lem_iter_ve_value}, for any time $t=(i-1)\Tcal_O+(j-1)k\tau+2$,
     \begin{align}
        (z_{e1},\cdots,z_{ek})^t = z^{(j)}.
    \end{align}
    Now we add $k$ nodes $w_l$, $1\leq l\leq k$, which update as
    \[
    w_l^t = \ind{y_l^{t-1} = z_{el}^{t-1}}.
    \]
    Thus when $t = (i-1)\Tcal_O+(j-1)k\tau+3$, the node $w_l$ for $l\leq i-i_0(i)$ has value
    \[
    w_l^t = \ind{x_{i_0(i)+l} = z^{(j)}_{el}}.
    \]
    Now we add two nodes $v_1,v_2$ to the RNN $O$, and has the following update rule:
    \begin{align}
        v_1^t = \begin{cases}
            \sum_{l=1}^k w_{l}^{t-1}\cdot u_0^{t-1}\cdot \ind{l\leq u_0^{t-1}}
            & \text{ if }w^{t-1} = 3 \text{ and }u^{t-1}=1 
            \text{ and }w_0^{t-1}\leq 2^kk\tau
            ;\\
            v_1^{t-1} & \text{ otherwise.}
        \end{cases}
    \end{align}
    \begin{align}
        v_2^t = \begin{cases}
            \sum_{l=1}^k w_{l}^{t-1}\cdot u_0^{t-1}\cdot \ind{l\leq u_0^{t-1}-1}
            & \text{ if }w^{t-1} = 3\text{ and }u^{t-1}=1
            \text{ and }w_0^{t-1}\leq 2^kk\tau;\\
            v_1^{t-1} & \text{ otherwise.}
        \end{cases}
    \end{align}
    Thus, for any $t\in[(i-1)\Tcal_O+(j-1)k\tau+4, (i-1)\Tcal_O+jk\tau]$,
    \[
    v_1^t = \ind{z^{(j)}_{:i-i_0(i)+1} = x_{i_0(i)+1:i+1}},
    \quad
    v_2^{t} =
        \ind{z^{(j)}_{:i-i_0(i)} = x_{i_0(i)+1:i}}.
    \]
    The RNN $O$ has a size of
    \[
    |O| = 4+k+k+1+k+2+1 = 3k+8.
    \]
    Since the nodes $v_1,v_2$ update their values at the third step of each string loop, based on the value of $w_l,u_0$ at the second step of the string loop, which depends on the value of $y_l$ and $z_{el}$ at the first step of the string loop, thus the output node at the end of each string loop only depends on the values of $H_):=\{w_0,u_0,w,u,Y,E\}$ at the first step of the string loop. As a result, the hidden node set size of $O$ is
    \[
    |H_O|= 4+k+k+1= 2k+5.
    \]
\end{proof}

\subsubsection{Proof of Lemma~\ref{lemma:rnn_boosting_construction}: RNN Boosting}
\label{sec:proof_lemma3}

Based on Lemma~\ref{lemma:rnn_prod_q}, Lemma~\ref{lemma:rnn_exp_distinguisher} and Lemma~\ref{lemma:rnn_indicator_first_digit}, we can prove Lemma~\ref{lemma:rnn_boosting_construction}, which realizes the boosted RNN language model. 
For the reader's convenience, we restate Lemma~\ref{lemma:rnn_boosting_construction} here.
\lemmarnn*
\begin{proof}
    Let $\tau = \max \{\Tcal_Q, \Tcal_{D}\}+4$. By Lemma~\ref{lemma:rnn_prod_q}, Lemma~\ref{lemma:rnn_exp_distinguisher} and Lemma~\ref{lemma:rnn_indicator_first_digit}, there exists an RNN $Q'$ with RNN-time $\Tcal = (2^k+1)k\tau$ and has four nodes $u_1,u_2,v_1,v_2$ with values computed as follows:
    \begin{itemize}
        \item When input index $1\leq i\leq i_0^*$, 
        \[
        u_1^{i\Tcal-1}=q(x_i \mid x_{:i}).
        \]
        \item When input index $i >i_0^*$,
        \begin{gather*}
            u_1^{(i-1)\Tcal + jk\tau-1} = q(z^{(j)} \mid x_{:i_0(i)+1})=f_1(z^{(j)},i).\\
            u_2^{(i-1)\Tcal + jk\tau-1} =\exp\left(-\alpha
            d\left(i_0(i)+1, x_{:i_0(i)+1}\cdot z^{(j)}\right)\right)
            =f_2(z^{(j)},i)
            .\\
            v_1^{(i-1)\Tcal + jk\tau-1} =
            \ind{z^{(j)}_{:i-i_0(i)+1} = x_{i_0(i)+1:i+1}}
            =g_1(z^{(j)},i)
            .\\
            v_2^{(i-1)\Tcal + jk\tau-1} =
            \ind{z^{(j)}_{:i-i_0(i)} = x_{i_0(i)+1:i}}
            =g_2(z^{(j)},i).
        \end{gather*}
    \end{itemize}
    Recall the counter nodes $w_0$ in Claim~\ref{claim:lem_proof_counters} with value $w_0^t=(t-1)\bmod(\Tcal)+1$
    and $v_c$ in Claim~\ref{claim:general_i0*} with value
    \[v_c^t = \begin{cases}
            i_1 &\text{ if }t\in[(i_1-1)\Tcal+1, i_1\Tcal] \text{ and }i_1\leq i_0^*+1;\\
            i_0^*+1&\text{ otherwise.}
            \end{cases}\]          
    We add nodes $w_1,w_2$ and the output node $v_{\txtout}$ to the RNN $Q'$, which are initialized as $0$ and updated as follows.
    \begin{gather}
    \label{eq:lem_main_w1}
        w_1^t = w_1^{t-1}+
        \left( u_1^{t-1}\cdot u_2^{t-1} \cdot v_1^{t-1}\right)
        \cdot\ind{w_{0}^{t-1}=jk\tau-1}.\\
        \label{eq:lem_main_w2}
        w_2^t = w_2^{t-1} + 
        \left(u_1^{t-1}\cdot u_2^{t-1} \cdot v_2^{t-1}\right)
        \cdot\ind{w_{0}^{t-1}=jk\tau-1}.\\
        \label{eq:lem_main_vout}
        v_{\txtout}^t = \begin{cases}
            u_1^{t-1} & \text{ if }v_c^{t-1}\leq i_0^*;\\
            w_1^{t-1}/w_1^{t-1} & \text{ otherwise}.
        \end{cases}
    \end{gather}
    Thus, when the input index $1\leq i\leq i_0^*$, the output node 
    \[
    v_{\txtout}^{i\Tcal} = u_1^{i\Tcal-1}=q(x_i \mid x_{:i}).
    \]
    When the input index $i > i_0^*$, the nodes $w_1$ and $w_2$ has values
    \begin{gather*}
        w_1^{(i-1)\Tcal + jk\tau} = \sum_{j'=1}^{j}
        f_1(z^{(j')},i)f_2(z^{(j')},i)g_1(z^{(j')},i)\\
         w_2^{(i-1)\Tcal + jk\tau} = 
         \sum_{j'=1}^j
         f_1(z^{(j')},i)f_2(z^{(j')},i)g_2(z^{(j')},i).
    \end{gather*}
    Thus,
    \[
    v_{\txtout}^{i\Tcal} 
    =v_{\txtout}^{(i-1)\Tcal+2^kk\tau+1}    
    = \frac{
    \sum_{j'=1}^{2^k}f_1(z^{(j')},i)f_2(z^{(j')},i)g_1(z^{(j')},i)
    }{
    \sum_{j'=1}^{2^k}f_1(z^{(j')},i)f_2(z^{(j')},i)g_2(z^{(j')},i)
    }
    =q'(x_i\mid x_{:i}).
    \]
    By Lemma~\ref{lemma:self_boost_next_token_prob}, the output of of the constructed RNN $Q'$ at time $i\Tcal$ realizes the conditional probability $q'(x_i\mid x_{:i})$ in Equation~\eqref{eq:q_update_k_token}. By Lemma~\ref{lemma:lm_update_decrease_kl}, it satisfies that
    \[
    \KL(\bp\| \bq')\leq \KL(\bp\|\bq)-\frac{\alpha^2 n}{4k}.
    \]
    By Lemma~\ref{lemma:ntp_kl}, it is equivalent to say
    \[
    L(q')-L(q)\leq -\frac{\alpha^2}{4k}.
    \]
    The size of the constructed RNN $Q'$ is
    \[
    |Q'|= |Q|+|H_Q|+2k+7+|D|+|H_D|+2k+7+3k+8+3
    =|Q|+|D|+|H_Q|+|H_D|+7k+25.
    \]
    The hidden node set size of the constructed RNN $Q'$ is 
    \[
    |H_{Q'}| = |H_Q|+2k+6+|H_D|+2k+6+2k+5
    =|H_Q|+|H_D|+6k+17.
    \]
   
\end{proof}

\section{Proof of the Main Theorem}
\label{sec:main_results}

Our main result states that minimizing next-token loss yields an LM that is $\eps$-indistinguishable from the true data distribution for any RNN-based next-$k$-token distinguisher of size up to $\dcal$. 
\thmmainminloss*
The result naturally extends to any alphabet size $|\Sigma|$, as shown in the theorem proof.
\paragraph{Choice of Hyperparameters.}

We first choose an integer index $j_0$ uniformly from the range $[4k\log|\Sigma|/\eps^2, 44k\log|\Sigma|/\eps^2]$. Then for $i\geq 1$, we choose the the $i$-th set of hyperparameters by the network size $N_i = 17(\dcal+k) (j_0+i)^2$, hidden node set size $H_i = 12(\dcal+k) (j_0+i)$ and RNN-time $\Tcal_i = \left(
8k|\Sigma|^k
\right)^{j_0+i-1}\tau$.
For the $i$-th set of hyperparameters of RNN, we train an LM by minimizing the next-token log loss. Let $\hat{q}_i$ be the optimal solution and $L_i:=L(\hat{q}_i)$ be its loss.  
We terminate the procedure when the loss decreases by at most $\epsilon^2/4k$ compared to the previous model. That is, it terminates and outputs $q_i$ if $L_i-L_{i+1}<\eps^2/4k$.
The full procedure is formally described in Algorithm~\ref{algo:min_loss_prac}. 

\begin{algorithm}[h]
   \caption{Minimizing Next-token Loss, Practically}
   \label{algo:min_loss_prac}
   \KwIn{Token set $\Sigma$, distinguisher window size $k$, distinguisher size bound $\dcal$, distinguisher RNN-time bound $\tau$, distinguisher advantage bound $\eps$.}
   \KwOut{An indistinguishable language model.}

   Choose an index $j_0$ uniformly from $[4k\log|\Sigma|/\eps^2, 44k\log|\Sigma|/\eps^2]$.\;

   \For{$i =j_0+1,j_0+2,\cdots$}{
    Minimize the next-token loss on an RNN of size $N_{i}=17(\dcal+k)i^2$, hidden node set size $H_{i}=12(\dcal+k)i$ and RNN-time $\Tcal_{i}=(8k|\Sigma|^k)^{i-1}\tau$.\;
    
    Let the optimal LM in size $N_{i}$ be $q_{i}$ with loss $L_{i}$.\;

    \DontPrintSemicolon
    \lIf{$L_{i-1}-L_{i}<\eps^2/4k$ and $i\geq 2$}{\Return the LM $q_{i-1}$.}
    
   }
\end{algorithm}

When optimizing over RNNs under structural constraints, even though the exact network structure is not known a priori, we can define a universal graph that encompasses all valid architectures within the constraint class. Specifically, suppose we are constrained to RNNs of size at most $N$, hidden node set size at most $H$ and RNN-time at most $\Tcal$ with $N,H,\Tcal\in \N$ and $H<N$. We can construct a fixed graph with $N$ nodes that allows searching over all valid RNNs within this class. Let $v_\txtin$ be the input node, $H\subset G\setminus \{v_\txtin\}$ be the hidden node set size, and $R=G\backslash \left(
H\bigcup \{v_\txtin\}
\right)$ be the remaining $N-H-1$ nodes, used for stateless computation.
We then define the edges of $G$ as follows.

\begin{itemize}
    \item For each pair of nodes $h_1,h_2\in H$, we add two edges $(h_1,h_2)$ and $(h_2,h_1)$ to the graph.
    \item For each node $h\in H$, we add an edge $(v_\txtin, h)$ to the graph.
    \item For each node $r\in H\cup R$, and any node $r\in R$, we add an edge $(v,r)$.
    \item Each node $r\in R$ is reset to $0$ whenever a new input is fed into the model.
\end{itemize}
Since the weights between nodes can be zero, the constructed graph is general enough to represent any RNN with constraints on size, hidden node set size and RNN-time. Figure~\ref{fig:universal_graph} gives a sketch of this universal graph.

\begin{figure}
    \centering
    \includegraphics[width=0.5\linewidth]{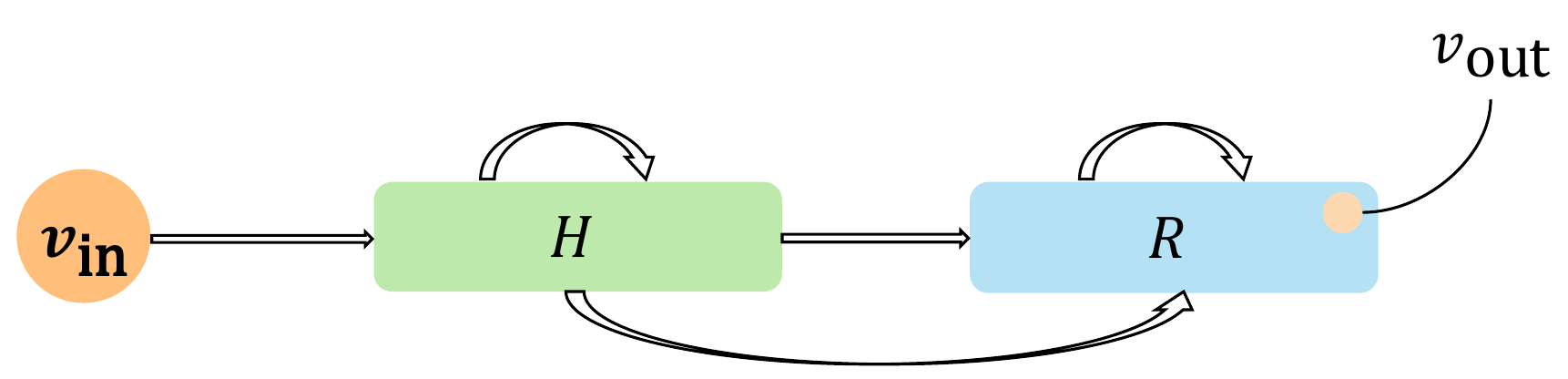}
    \caption{A universal graph that encompasses all RNNs with constraints on size, hidden node set size, and RNN-time.}
    \label{fig:universal_graph}
\end{figure}

Before proving Theorem~\ref{thm:main2}, we first prove Lemma~\ref{lemma:self_boosting_criticizer}, which formulates self-boosting by loss minimization. We restate the lemma here for the reader's convenience.
\lemmaselfboost*
\begin{proof}
    Let $N_j \in \{N_i\}_{i\geq 1}\setminus B_\eps$ and $\hat{q}$ is the minimizer of $L(q)$ over $C_j$. Then for any function $c\in\Ccal$, there is an $q'\in\Qcal$ with size
    \[
    |q'| \leq |\hat{q}| +h(\hat{q})  +\beta
    \leq N_j+H_j+ \beta = N_{j+1},
    \]
    time
    \[
    \Tcal(q')\leq \gamma \Tcal(\hat{q})+\delta
    \leq \gamma H_j+\delta = \Tcal_{j+1},
    \]
    value
    \[
    h(q')\leq \theta h(\hat{q})+\zeta
    \leq \theta H_j+\delta = H_{i+1},
    \]
    and loss 
    \[
    L(q')\leq L(\hat{q})-c(\hat{q}) = L_j-c(\hat{q}).
    \]
    Since $q'$ satisfies the constraints for $N+\beta$ in size, time and value, we know $q'\in C_{j+1}$. Thus,
    \begin{align}
    \label{eq:lem_proofboost_loss}
        L_{j+1} \leq L(q')\leq L_j-c(\hat{q}).
    \end{align}
    By the definition of $B_\eps$, we know 
    $L_{j+1} \geq L_j-\eps$. Combining with Equation~\eqref{eq:lem_proofboost_loss}, we have
    \[
    L_j-\eps \leq L_{j+1} \leq L_j- c(\hat{q})
    \]
    This leads to $c(\hat{q})\leq \eps$ for any $c\in\Ccal$.
    
    Next, we will bound the size of $B_\eps$. Observe that as $i$ increases, the feasible region of the optimization problem expands. As a result, the optimal loss $L_i$ is monotonically decreasing. Since the loss function is always non-negative, the number of $j$'s such that $L_{j+1}<L_j-\epsilon$ is no more than $L_1/\eps$. In other words, $|B_\eps| \leq L_1/\eps$.
    
\end{proof}

Now we are ready to prove Theorem~\ref{thm:main2} using Lemma~\ref{lemma:rnn_boosting_construction} and Lemma~\ref{lemma:self_boosting_criticizer}.

\begin{proof}[Proof of Theorem~\ref{thm:main2}]
    The sequences $\{N_i\}_{i\geq 1}, \{H_i\}_{i\geq 1}, \{\Tcal_i\}_{i\geq 1}$ are defined in Algorithm~\ref{algo:min_loss_prac}. By computation, the network size for index $i+1$ is
    \begin{align*}
        N_{i+1} =& 17(\dcal+k)(j+i+1)^2\\
        =&17(\dcal+k)(j+i)^2 + 34(\dcal+k)(j+i)+17(\dcal+k)\\
        \geq& N_i + H_i+2\dcal+7k+25,
    \end{align*}
    and its hidden node set size is
    \begin{align*}
        H_{i+1} =& 12(\dcal+k)(j+i+1)\\
        =& 12(\dcal+k)(j+i) + 12(\dcal+k)\\
        \geq & H_i + \dcal+6k+17.
    \end{align*}
    Combining with Lemma~\ref{lemma:rnn_boosting_construction}, we know for any language model $q$ constructed with RNN $Q$ with size $N_i$, hidden node set size $H_i$ and RNN-time $\Tcal_i$,  if there exists a distinguisher RNN $D$ with size $|D|\leq \dcal$ that has advantage $\alpha$, then we can construct a language model $q'$ implemented by an RNN $Q'$ with size no more than $N_{i+1}$, hidden node set size no more than $H_{i+1}$ and RNN-time no more than  $\Tcal_{i+1}$, such that
    \[
    L(q')-L(q) \leq -\frac{\alpha^2}{4k}.
    \]
    Then we apply Lemma~\ref{lemma:self_boosting_criticizer} by using the value function $h(q)$ as the number of hidden node set size of the model $q$. Let the set $B_\eps = \{N_j \in \{N_i\}_{i\geq 1}\mid L_{j+1}<L_j - \eps^2/4k\}$. Then for any $\eps>0$, for all $N_i \in \{N_i\}_{i\geq 1}\setminus B_\eps$, every $\hat{q}$ which minimizes $L(q)$ over $C_j=\{q\mid |q|\leq N_j, \Tcal(q)\leq \Tcal_j, h(q)\leq H_j\}$ satisfies that for any next-$k$-token RNN distinguisher $d$ of size $|d|\leq \dcal$ and RNN-time $\Tcal(d)\leq \tau$,
    \[
    \frac{\alpha^2 }{4k}\leq \frac{\eps^2}{4k}.
    \]
    Equivalently, the advantage for any distinguisher is at most $\eps$.

    The algorithm terminates with a network size $N_i\notin B_\eps$. Consequently, for any next-$k$-token distinguisher of size up to $\dcal$ and RNN-time up to $\tau$, the output LM is $\eps$-indistinguishable.

    Next, we will bound the size of $B_\eps$. 
    Let $\bq_0$ be the uniform distribution, which can be realized by an RNN of size, hidden node set size and RNN-time one (with constant output on the next-token probability). Then the KL divergence between $\bp$ and $\bq_0$ can be bounded by
    \[
    \KL(\bp\|\bq_0) = \Eop\limits_{x\sim \bp}\log \frac{\bp(x)}{\bq_0(x)}
    = \int \bp(x) \log \frac{\bp(x)}{1/|\Sigma|^n} = \int \bp(x)\log \bp(x) + n\log |\Sigma|
    \leq n\log |\Sigma|
    \]
    Since KL divergence is nonnegative, the size of $B_\eps$, or equivalently, the number of $j_0$'s where the KL decreases by at least $\eps^2n/4k$ can be bounded by
    \[
    \frac{\KL(\bp\|\bq_0)}{\eps^2n/4k}
    \leq \frac{n\log|\Sigma|}{\eps^2n/4k}
    = \frac{4k\log|\Sigma|}{\eps^2}
    \]

    Since the index $j_0$ is uniformly chosen from $[4k\log|\Sigma|/\eps^2,44k\log|\Sigma|/\eps^2]$, the probability that $N_{j_0}\in B_\eps$ is less than $0.1$. Therefore, the probability that we only try two model sizes, that is, if the first model size that we try is not in $B_\eps$, is 0.9.

    Assume that the algorithm terminates with size $N_{i'}$. This implies that for each $i\in[j_0+1,i')$, we have $\KL(\bp\|\bq_i)-\KL(\bp\|\bq_{i+1})>\eps^2n/4k$. Summing over $i$, we obtain $\KL(\bp\|\bq_{j_0+1}) - \KL(\bp\|\bq_{i'})> (i'-j_0-1)\eps^2n/4k$. 
    Since $\KL(\bp\|\bq_{j_0+1})\leq \KL(\bp\|\bq_{0})\leq n\log|\Sigma|$, we have 
    \[
    i'-j_0-1 < \frac{\KL(\bp\|\bq_{j_0+1})}{\eps^2n/4k}\leq \frac{n\log|\Sigma|}{\eps^2n/4k} = \frac{4k\log|\Sigma|}{\eps^2}.
    \]
    Thus, the final output LM has size bounded by
    \begin{align*}
         N_{i'} 
         =& 17(\dcal+k){i'}^2\\
        \leq&17(\dcal+k)\left(\frac{44k\log|\Sigma|}{\eps^2}+\frac{4k\log|\Sigma|}{\eps^2}+1\right)^2\\
        =&O\left(
        \frac{1}{\eps^4}(d+k)k^2\log^2|\Sigma|
        \right)
    \end{align*}
    Its RNN-time is bounded by
    \begin{align*}
        \Tcal_{N_{i'}}
        =\left(8k|\Sigma|^k\right)^{i'-1}\cdot \tau
        \leq  \left(8k|\Sigma|^k\right)^{\frac{48k\log|\Sigma|}{\eps^2}}\cdot \tau
        =\left(k|\Sigma|^k\right)^{O\left(\frac{k\log|\Sigma|}{\eps^2}\right)}\cdot \tau
    \end{align*}

\end{proof}

\subsection{Bounding the bit size}
\label{subsec:bitsize}
In this section, we bound the total space required by the learned language model in terms of the bit sizes of the numbers it maintains.  
The key technical ingredient is  Lemma~\ref{lemma:rnn_selfboost_bitsize}, which generalizes Lemma~\ref{lemma:rnn_boosting_construction} to the setting of bounded bit size. We first recall the bit size of an RNN.
\begin{defn}[Bit Size of the RNN]
    The bit size of an RNN $Q$ is the number of bits needed to encode the value stored in each node at any time step, denoted by $\langle Q\rangle=1+\langle Q\rangle_I+\langle Q\rangle_F$. Formally, we fix a signed fixed-point representation with one sign bit, $\langle Q\rangle_I$ integer bits, and $\langle Q\rangle_F$ fractional bits. Each real number $r$ stored in a node is represented as $r=\sign(r)\left(r_I+r_F\right)$, where $\sign(r)\in\{+1,-1\}$ is the sign of $r$, $r_I$ is its integer part in the range $[1,2^{\langle Q\rangle_I}]$, and $r_F\in [0,1)$ is a its fractional part as a multiple of $2^{-\langle Q\rangle_F}$.
\end{defn}
For each real number $x\in\R$, where $x=x_I+x_F$ with $x_I\in\Z$ and $x_F\in[0,1)$. We define the quantizer
\[
\Qcal_{b}(x):=\min\left(x_I, 2^{b_I}\right) + 2^{-b_F} \cdot \left\lfloor\frac{x_F}{2^{-b_F}}\right\rfloor.
\]
We say that the quantizer $\Qcal_{b_F}(x)$ induces an absolute (additive) error $\delta$ if $\left| \Qcal_{b_F}(x)-x\right| \leq \delta$.

\begin{algorithm}[htp]
   \caption{Minimizing Next-token Loss with Bounded Bit Size}
   \label{algo:min_loss_bitsize}
   \KwIn{Vocabulary set size $|\Sigma|$, distinguisher window size $k$, distinguisher size bound $\dcal$, distinguisher RNN-time bound $\tau$, distinguisher bit size bound $b_D$, distinguisher advantage bound $\eps$.}
   \KwOut{An indistinguishable language model.}

   Choose an index $j_0$ uniformly from $[16k\log|\Sigma|/\eps^2, 176k\log|\Sigma|/\eps^2]$.\;

   \For{$i =j_0+1,j_0+2,\cdots$}{
    Minimize the next-token loss on an RNN of size $N_{i}=17(\dcal+k)i^2$, hidden node set size $H_{i}=12(\dcal+k)i$, RNN-time $\Tcal_{i}=(8k|\Sigma|^k)^{i-1}\tau$, bit size $b_i=b_D+3k\log|\Sigma|i^2+i\log\tau+772\left(\frac{k^2}{\eps^2}\log|\Sigma|+\log\frac{1}{\eps}\right)$, and next-token probability lower bound $\ell_i = \frac{0.99}{|\Sigma|4^{i-1}}$.
    \;
    
    Let the optimal LM in size $N_{i}$ be $q_{i}$ with loss $L_{i}$.\;

    \DontPrintSemicolon
    \lIf{$L_{i-1}-L_{i}<\eps^2/8k$ and $i\geq j_0+2$}{\Return the LM $q_{i-1}$.}
    
   }
\end{algorithm}
We show in the following theorem that minimizing next-token loss yields an $\eps$-indistinguishable LM with bounded bit size.
\begin{thm}[Main result with Bounded Bit Size]
\label{thm:main_bit}
For any $0<\eps<1,  b_D,k,\tau,\dcal\in\N$, Algorithm~\ref{algo:min_loss_bitsize} outputs an LM $q$ with the following properties:  
\begin{enumerate} 
\item The model $q$ is $\eps$-indistinguishable from the training  distribution $p$ for any next-$k$-token distinguisher $d:[n]\times \Sigma^n \rightarrow\{0,1\}$ realized by an RNN of size $|d|\leq \dcal$, RNN-time $\Tcal(d) \leq \tau$ and bit size $\langle d\rangle \leq b_D$.
\item The model $q$ has size $O\left(
        \frac{k^2}{\eps^4}(\dcal+k)\log^2|\Sigma|
        \right)$, RNN-time $
        O\left(
        \tau \cdot (k|\Sigma|^k)^{\frac{48k\log|\Sigma|}{\eps^2}}
        \right)$,
        and bit-size $O\left(
        b_D + \frac{k^3\log^2|\Sigma|}{\eps^4}
        +\frac{k}{\eps^2}\log\tau\log|\Sigma|
        \right)$.
\end{enumerate}
Moreover, with probability at least 0.9, the number of model sizes attempted is two.
\end{thm}
Similar to the proof of Theorem~\ref{thm:main2}, we use
Lemma~\ref{lemma:rnn_selfboost_bitsize} to construct a boosted LM using a distinguisher with controlled bit size, and then apply Lemma~\ref{lemma:self_boosting_criticizer} to achieve indistinguishability. We state and prove Lemma~\ref{lemma:rnn_selfboost_bitsize} before proving the theorem.

\begin{lemma}[Boosted RNN with Bounded Bit Size]
\label{lemma:rnn_selfboost_bitsize}
    Let $0<\alpha,\ell,\delta\leq 1$, and integers $k,b,b_I,b_F,$ $b_D,b_{D,I},b_{D,F} ,\Tcal_D\in\N$ satisfying  $b=1+b_I+b_F,
    b_D=1+b_{D,I}+b_{D,F}$, 
    $b_I\geq  b_{D,I}+k\log|\Sigma|+\log \left(k\Tcal_D\right)+1$,
    $b_F\geq b_{D,F}$,
    $2^{-b_F}\leq \frac{\alpha^2\ell^{k+1}}{1088k^2}$.
    Let $q$ be any language model represented by an RNN $Q$ with bit size $b$ (with $b_I$ integer bits and $b_F$ fractional bits). Its next-token conditional probabilities are all at least $\ell$. 
    Suppose there exists a next-$k$-token distinguisher $d$ for $q$ with advantage $\alpha$, implemented by an RNN $D$ of bit size $b_D$ (with $b_{D,I}$ integer bits and $b_{D,F}$ fractional bits) and RNN-time $\Tcal_D$.
    Then there exists a language model $q'$, implemented by an RNN $Q'$ with size and hidden node set size
    \[
    |Q'|=|Q|+|H_{Q}|+|D|+|H_D|+7k+25,\quad
    |H_{Q'}|=|H_{Q}|+|H_D|+6k+17,
    \]
    bit size 
    \[
    \langle Q'\rangle = b+\log(\Tcal_Q),
    \quad
    \langle Q'\rangle_I = b_I + \log(\Tcal_Q),
    \quad
    \langle Q'\rangle_F = b_F.
    \]
    RNN-time
    \[
    \Tcal_{Q'} =(|\Sigma|^k+1)k(\max\{\Tcal_{Q},\Tcal_D\}+4),
    \]
    and all next-token conditional probabilities are lower bounded by
    \[
    \ell'=\ell/4,
    \]
    such that its next-token loss satisfies
    \[
    L(q')-L(q) \leq -\alpha^2/8k.
    \]
\end{lemma}

\begin{proof}
    Let $b' = 1+b_I'+b'_F$, where we assign the first bit to denote the sign, $b_I'$ bits for integers, and $b_F$ bits for fractions. 
    
    By Lemma~\ref{lemma:rnn_boosting_construction}, we can boost the LM $q$ by the distinguisher $d$ to an LM $\tilde{q}'$ realized by an unbounded-bit-size RNN $\tilde{Q}'$ such that
    \begin{align}
        L(\tilde{q}')-L(q)\leq -\alpha^2/4k.
    \end{align}
    The RNN $\tilde{Q}'$ satisfies
    \[
    |\tilde{Q}'|=|Q|+|H_{Q}|+|D|+|H_D|+7k+25,\quad
    |H_{\tilde{Q}'}|=|H_{Q}|+|H_D|+6k+17,
    \]
    and has RNN-time
    \[
    \Tcal_{\tilde{Q}'}=(|\Sigma|^k+1)k(\max\{\Tcal_{Q},\Tcal_D\}+4).
    \]
    We now construct the RNN $Q'$ by restricting the bit size of the RNN $\tilde{Q}'$ to $b'=1+b'_I+b_F'$. Because the quantization does not change the size, the hidden node set size, or the RNN-time of an RNN, we immediately obtain
    \[
    |Q'|=|\tilde{Q}'|=|Q'|+|H_Q|+|D|+|H_D|+7k+25,\quad
    |H_Q'|= |H_{\tilde{Q}'}|= |H_Q|+|H_D|+6k+17,
    \]
    and
    \[
    \Tcal_{Q'} =\Tcal_{\tilde{Q}'}  =(|\Sigma|^k+1)k(\max\{\Tcal_Q,\Tcal_D\}+4).
    \]
    To analyze the bit size and bound the output error induced by quantization, we compute the integer part and the fractional part in Claim~\ref{claim:bitsize_integer} and Claim~\ref{claim:bitsize_fraction}, respectively. We recall from Equation~\eqref{eq:self_boost_next_token_prob}, there exists an index $i_0^*\in[0,k-1]$ and a set $I=\{i\in[n] \mid i\equiv i_0^*\Mod{k}\}$
    such that $q'$ is computed as
    \[
    q'(x_i \mid x_{:i})=\begin{cases}
            \displaystyle \Qcal_{b'}\left(q(x_i \mid x_{:i})\right) & \text{ if }i\leq i_0^*;\\
            \displaystyle 
            \Qcal_{b'}\left(
            \frac{
            \Qcal_{b'}\left(
            \sum\limits_{s\in\Sigma^{k-r_0}}
            \Qcal_{b'}\left(
            q(x_{i_0+1:i+1}\cdot s\mid x_{:i_0+1})e^{-\alpha d_{i_0+1}(x_{:i+1}\cdot s)}\right)\right)
            }{
            \Qcal_{b'}\left(
            \sum\limits_{s\in\Sigma^{k-r_0+1}}
            \Qcal_{b'}\left(
            q(x_{i_0+1:i}\cdot s\mid x_{:i_0+1})e^{-\alpha d_{i_0+1}(x_{:i}\cdot s)}\right)\right)
            }
            \right)
            & 
            \begin{aligned}
                &\text{ if }i=i_0+r_0,\\
                & \text{ }i_0\in I,\\
                &\text{ }1\leq r_0\leq k.
            \end{aligned}
        \end{cases}
    \]
    Here, the quantization step $\Qcal_{b'}(v)$ is applied for each node $v$ computed in the RNN $Q'$.

    \begin{claim}[Integer Bit Size]
    \label{claim:bitsize_integer}
        The integer bit size of $Q'$ can be bounded by
        \[
        b_I'= b_I+\log(\Tcal_Q).
        \]
    \end{claim}
    \begin{proof}
        We bound the integer bit size by finding the upper bound of the node values in the RNN $Q'$. Our construction of the RNN $Q'$ as in Lemma~\ref{lemma:rnn_boosting_construction} is based on Lemma~\ref{lemma:rnn_prod_q}, Lemma~\ref{lemma:rnn_exp_distinguisher} and Lemma~\ref{lemma:rnn_indicator_first_digit}. 

        The main part of the three lemmas is based on Lemma~\ref{lemma:iterate_compose}.
        Specifically, in Claim~\ref{claim:lem_proof_counters}, the counter nodes are bounded by $\Tcal_{Q'}=(|\Sigma|^k+1)k\max\left\{\Tcal_Q,\Tcal_D\right\}$, thus they can be represented using $k\log|\Sigma|+1+\log(k\max\left\{\Tcal_Q,\Tcal_D\right\})$ bits. In Claim~\ref{claim:lem_proof_input_set} and Claim~\ref{claim:lem_proof_enumerater}, each node in $Y$ and $E$ is upper bounded by $|\Sigma|$, and thus can be represented using $\log|\Sigma|$ bits. In Claim~\ref{claim:lem_proof_H}, Claim~\ref{claim:lem_proof_Htilde} and Claim~\ref{claim:lem_proof_R}, each node in $H, \tilde{H}$ and $R$ has a corresponding node in the RNN $Q$, and thus are upper bounded by $\max\{2^{b_I},2^{b_{D,I}}\}$, with $\max\{b_I,b_{D,I}\}$ integer bits. The extra counter node in Claim~\ref{claim:general_i0*} is bounded by $k$, which can be represented with $\log(k)$ bits. Thus, the integer bit size of these three lemmas is bounded by $\max\{k\log|\Sigma|+\log k+\log(\max\left\{\Tcal_Q,\Tcal_D\right\})+1,
        \log|\Sigma|,
        b_I,b_{D,I},\log(k).$
        Since $b_I \geq b_{D,I}+k\log|\Sigma|+\log(k\Tcal_D)+1$, the integer bit size can be upper bounded by
        \[
        \max\{k\log|\Sigma|+\log k+\log(\max\left\{\Tcal_Q,\Tcal_D\right\})+1,
        \log|\Sigma|,
        b_I,b_{D,I},\log(k)\}
        \leq b_I+ \log(\Tcal_Q).
        \]
        For the additional nodes introduced in the proof of Lemma~\ref{lemma:rnn_boosting_construction}, the nodes $w_1$ and $w_2$ compute the summation of reweighted probabilities and are therefore bounded above by $e^{\alpha}$. The final output node, which represents a probability, is bounded by $1$. Thus, their integer part requires only a constant number of bits.

        Thus, we can set the integer bits as follows.
        \[
        b_I'= b_I + \log(\Tcal_Q).
        \]
    \end{proof}
    \begin{claim}[Fractional Bit Size]
    \label{claim:bitsize_fraction}
        By using $b'_F=b_F$ fractional bits, the output probability $q'$ satisfies that for any $x\in\Sigma^*$ and any $i\in\N$,
        \[
        \left|q'(x_i\mid x_{:i})-\tilde{q}' (x_i\mid x_{:i})\right|\leq \frac{
            17k2^{-b_F}
            }{
            \ell^{k}
            }.
        \]
    \end{claim}
    \begin{proof}
        Firstly, for the input index $i \leq i_0^*$, it outputs the same as $q$, and thus has no quantization error by using $b_F$ bits. We next analyze the case when $i > i_0^*$.
        Define $f_1, f_2:\Sigma^{k-r_0}\rightarrow [0,1]$  as
        \[
        f_1(s) = q\left(x_{i_0+1:i+1}\cdot s\mid x_{:i_0+1}\right),\quad
        f_2(s) = \exp\left(-\alpha d_{i_0+1}(x_{:i+1}\cdot s)\right).
        \]
        Similarly, we define $h_1, h_2:\Sigma^{k-r_0+1}\rightarrow [0,1]$ defined as
        \[
        h_1(s) = q\left(x_{i_0+1:i}\cdot s\mid x_{:i_0+1}\right),\quad
        h_2(s) = \exp\left(-\alpha d_{i_0+1}(x_{:i}\cdot s)\right).
        \]
        Since we do quantization using $b'$ bits on each node, for $i\geq i_0^*+1$,
        \begin{align}
            q'(x_i\mid x_{:i}) = 
            \Qcal_{b'}\left(
            \frac{
            \Qcal_{b'}\left(\sum\limits_{s\in\Sigma^{k-r_0}} 
            \Qcal_{b'}\left(f_1(s)f_2(s)\right)\right)
            }{
            \Qcal_{b'}\left(\sum\limits_{s\in\Sigma^{k-r_0+1}} 
            \Qcal_{b'}\left(h_1(s)h_2(s)\right)\right)
            }
            \right).
        \end{align}
        Since the quantization using $b_F$ fractional bits induces $\delta_1:= 2^{-b_F}$ absolute error, for each $s\in \Sigma^{k-r_0}$, we have
        \[
        \left|\Qcal_{b'}(f_1(s)f_2(s)) - f_1(s)f_2(s)\right| \leq \delta_1.
        \]
        When summing over $s$, we have
        \begin{align*}
            &\left|
            \sum\limits_{s\in\Sigma^{k-r_0}}
            \Qcal_{b'}\left(
            f_1(s)f_2(s)
            \right)
            - \sum\limits_{s\in\Sigma^{k-r_0}}
            f_1(s)f_2(s)
            \right|\\
            \leq &
            \sum\limits_{s\in\Sigma^{k-r_0}}
            \left|
            \Qcal_{b'}\left(
            f_1(s)f_2(s)
            \right)
            -f_1(s)f_2(s)
            \right|\\
            \leq &|\Sigma|^{k-r_0}
            (2k+1)\delta_1.
        \end{align*}
        By a further step of quantization, we have
        \begin{align}
        \label{eq:lemma_bitsize_qqerror1}
            \left|
            \Qcal_{b'}\left(
            \sum\limits_{s\in\Sigma^{k-r_0}}
            \Qcal_{b'}\left(
            f_1(s)f_2(s)
            \right)\right)
            - \sum\limits_{s\in\Sigma^{k-r_0}}
            f_1(s)f_2(s)
            \right|
            \leq |\Sigma|^{k-r_0}(2k+1)\delta_1+\delta_1
            \leq |\Sigma|^{k-r_0}(2k+2)\delta_1.
        \end{align}
        Similarly, we have
        \begin{align}
        \label{eq:lemma_bitsize_qqerror2}
            \left|
            \Qcal_{b'}\left(
            \sum\limits_{s\in\Sigma^{k-r_0+1}}
            \Qcal_{b'}\left(
            h_1(s)h_2(s)
            \right)\right)
            - \sum\limits_{s\in\Sigma^{k-r_0+1}}
            h_1(s)h_2(s)
            \right|
            \leq |\Sigma|^{k-r_0+1}(2k+2)\delta_1.
        \end{align}
        Combining Equation~\eqref{eq:lemma_bitsize_qqerror1} and Equation~\eqref{eq:lemma_bitsize_qqerror2}, we have 
        \begin{align}
        \label{eq:lemma_bitsize_tildeq_error}
            \frac{
            \Qcal_{b'}\left(
            \sum\limits_{s\in\Sigma^{k-r_0}}
            \Qcal_{b'}\left(
            f_1(s)f_2(s)
            \right)\right)
            }
            {
            \Qcal_{b'}\left(
            \sum\limits_{s\in\Sigma^{k-r_0+1}}
            \Qcal_{b'}\left(
            h_1(s)h_2(s)
            \right)\right)
            }
            \leq 
            \frac{
            \sum\limits_{s\in\Sigma^{k-r_0}}
            f_1(s)f_2(s)
            +|\Sigma|^{k-r_0+1}(2k+2)\delta_1
            }{
            \sum\limits_{s\in\Sigma^{k-r_0+1}}
            h_1(s)h_2(s)
            -|\Sigma|^{k-r_0+1}(2k+2)\delta_1
            }.
        \end{align}
        Since the next-token probability of $q$ is lower bounded by $\ell$ we have
        \begin{align}
            \label{eq:lemma_bitsize_h1h2lower}
            \sum_{s\in\Sigma^{k-r_0+1}}h_1(s)h_2(s)
            \geq e^{-\alpha} \sum_{s\in\Sigma^{k-r_0+1}} q(x_{i_0+1:i}s\mid x_{:i_0+1})
            \geq e^{-\alpha} |\Sigma|^{k-r_0+1}\ell^k.
        \end{align}
        Since $\delta_1=2^{-b_I}\leq \frac{\alpha^2\ell^{k+1}}{1088k^2}\leq \frac{e^{-\alpha}\ell^k}{4(k+1)}$, we have
        \[
        e^{-\alpha}|\Sigma|^{k-r_0+1}\ell^k
        \geq
        2|\Sigma|^{k-r_0+1}(2k+2)\delta_1.
        \]
        Given the lower bound of $\sum_{s\in\Sigma^{k-r_0+1}}h_1(s)h_2(s)$ as in Equation~\eqref{eq:lemma_bitsize_h1h2lower}, we can apply Lemma~\ref{lemma:error_fraction} as follows.
        \begin{align*}
            &\frac{
            \sum\limits_{s\in\Sigma^{k-r_0}}
            f_1(s)f_2(s)
            +|\Sigma|^{k-r_0+1}(2k+2)\delta_1
            }{
            \sum\limits_{s\in\Sigma^{k-r_0+1}}
            h_1(s)h_2(s)
            -|\Sigma|^{k-r_0+1}(2k+2)\delta_1
            }
            -
            \frac{
            \sum\limits_{s\in\Sigma^{k-r_0}}
            f_1(s)f_2(s)
            }{
            \sum\limits_{s\in\Sigma^{k-r_0+1}}
            h_1(s)h_2(s)
            }
            \\
            \leq &
            \frac{
            2|\Sigma|^{k-r_0+1}(2k+2)\delta_1
            }{        
            e^{-\alpha}|\Sigma|^{k-r_0+1}\ell^k - |\Sigma|^{k-r_0+1}(2k+2)\delta_1
            }\\
            \leq&
            \frac{
            2|\Sigma|^{k-r_0+1}(2k+2)\delta_1
            }{
            \frac{1}{2}e^{-\alpha} |\Sigma|^{k-r_0+1} \ell^{k}
            }\\
            \leq &
            \frac{
            8(k+1)\delta_1e^\alpha
            }{
            \ell^{k}
            }
        \end{align*}
        Combining this inequality with Equation~\eqref{eq:lemma_bitsize_tildeq_error}, we have
        \[
            \frac{
            \Qcal_{b'}\left(
            \sum\limits_{s\in\Sigma^{k-r_0}}
            \Qcal_{b'}\left(
            f_1(s)f_2(s)
            \right)\right)
            }
            {
            \Qcal_{b'}\left(
            \sum\limits_{s\in\Sigma^{k-r_0+1}}
            \Qcal_{b'}\left(
            h_1(s)h_2(s)
            \right)\right)
            }
            \leq 
            \tilde{q}'(x_i\mid x_{:i}) + \frac{
            8(k+1)\delta_1e^\alpha
            }{
            \ell^{k}
            }.
        \]
        In the same way, we can show that 
        \[
            \frac{
            \Qcal_{b'}\left(
            \sum\limits_{s\in\Sigma^{k-r_0}}
            \Qcal_{b'}\left(
            f_1(s)f_2(s)
            \right)\right)
            }
            {
            \Qcal_{b'}\left(
            \sum\limits_{s\in\Sigma^{k-r_0+1}}
            \Qcal_{b'}\left(
            h_1(s)h_2(s)
            \right)\right)
            }
            \geq 
            \tilde{q}'(x_i\mid x_{:i}) -
            \frac{
            8(k+1)\delta_1e^\alpha
            }{
            \ell^{k}
            }.
        \]
        By one more step of quantization on the output node, we get
        \begin{align*}
            \left|q'(x_i\mid x_{:i}) - \tilde{q}'(x_i\mid x_{:i})\right|
            \leq \frac{
            8(k+1)\delta_1e^\alpha
            }{
            \ell^{k}
            }+\delta_1
            \leq  \frac{
            17k\delta_1
            }{
            \ell^{k}
            }.
        \end{align*}
    \end{proof}
    From Claim~\ref{claim:bitsize_integer} and Claim~\ref{claim:bitsize_fraction}, by using $b'=1+b_I'+b_F'$ bits, the quantized LM $q'$ has absolute error $ \delta_2 = 17k\delta_1/\ell^{k}$.

    By Lemma~\ref{lemma:lower_bound_q'}, for any $x\in\Sigma^*$ and $i\in \N$,
    \[
    \tilde{q}'(x_i\mid x_{:i})\geq \ell/3.
    \]
    Thus,
    \[
    q'(x_i\mid x_{:i})
    \geq  \ell/3 - 17k\delta_1/\ell^{k}
    \geq \ell/3-\ell/12
    = \ell/4=\ell'
    .
    \]
    The last inequality holds because $\delta_1 \leq \frac{\alpha^2\ell^{k+1}}{1088k^2} \leq 
    \frac{\ell^{k+1}}{204k}$.

    Finally, we will compute the improvement of next-token loss. By Lemma~\ref{lemma:bitsize_lossimprove_quantized},
    \[
    \Eop_{x\sim\bp}
    \log\frac{\bar{q'}(x)}{\bar{\tilde{q}}'(x)}
    \leq n\cdot \frac{\delta_2}{\ell'-\delta_2}.
    \]
    Also by Lemma~\ref{lemma:rnn_boosting_construction}, we have
    \[
    \frac{1}{n}\Eop_{x\sim \bp}\log \left(
    \frac{
    \bar{\tilde{q}}'(x)
    }{
    \bq(x)
    }
    \right)
    =
    L(\tilde{q}')-L(q)
    \leq -\frac{\alpha^2}{4k}.
    \]
    Thus, by the definition of next-token loss, we have
    \begin{align*}
        L(q')-L(q)
        =&\Eop_{x\sim \bp}\left[-
        \frac{1}{n}\sum_{i=1}^n\log q'(x_i\mid x_{:i})
        +
        \frac{1}{n}\sum_{i=1}^n\log q(x_i\mid x_{:i})
        \right]\\
        =& \frac{1}{n}\Eop_{x\sim \bp}
        \log \left(
        \frac{\bar{q}'(x)}{
        \bar{q}(x)
        }
        \right)\\
        = &\frac{1}{n}
        \Eop_{x\sim \bp}
        \log \frac{
        \bar{q}'(x)
        }{
        \bar{\tilde{q}}'(x)
        }
        + 
        \frac{1}{n}
        \Eop_{x\sim \bp}
        \log \frac{
        \bar{\tilde{q}}'(x)
        }{
        \bar{q}(x)
        }\\
        \leq &
        \frac{\delta_2}{\ell'-\delta_2}
        -\frac{\alpha^2}{4k}
    \end{align*}
    Since $\delta_1 < \frac{\ell^{k+1}\alpha^2}{3172k^2}$, we have $\delta_2=\frac{17k\delta_1}{\ell^k}< \frac{\alpha^2\ell}{64k}$. Thus, 
    \[
    \frac{\delta_2}{\ell'-\delta_2} < \frac{\alpha^2}{8k}.
    \]
    Thus, we have,
    \[
    L(q') -L(q) < -\frac{\alpha^2}{8k}.
    \]
    The bit size of $q'$ is
    \begin{align*}
         b'=1+b_I'+b_F'
         =1+b_I+\log(\Tcal_Q) + b_F
         = b+\log(\Tcal_Q).
    \end{align*}
    
\end{proof}

We are now ready to prove the main theorem in this section.
\begin{proof}[Proof of Theorem~\ref{thm:main_bit}]
    The proof of size, hidden node set size and RNN-time follows the same as the proof of Theorem~\ref{thm:main2}. We here focus on the bit size. The proof idea is a combination of Lemma~\ref{lemma:rnn_selfboost_bitsize} and Lemma~\ref{lemma:self_boosting_criticizer}.
    We first show in the following claim that the sequences of hyperparameters $\{b_i\}_{i\geq 1}, \{I_i\}_{i\geq 1},\{F_i\}_{i\geq 1},\{\ell_i\}_{i\geq 1}$ satisfy the conditions in Lemma~\ref{lemma:rnn_selfboost_bitsize}.
    \begin{claim}
    \label{claim:proof_bitsize}
        For $1\leq i\leq \frac{193k}{\eps^2}\log|\Sigma|$, the following equations hold.
        \begin{gather*}
            \ell_{i+1} = \frac{\ell_i}{4},\quad 
            I_i \geq b_{D,I}+k\log|\Sigma|+\log(k\tau)+1,\\
            I_{i+1} \geq I_i + \log(\Tcal_i),\quad
            F_i \geq b_{D,F},\quad
            2^{-F_i}\leq \frac{\eps^2\ell_i^{k+1}}{1088k^2}.
        \end{gather*}
    \end{claim}
    \begin{proof}
        We check them sequentially. For the lower bound sequence $\{\ell_i\}$,
        \begin{align*}
            \ell_{i+1} = \frac{0.99}{|\Sigma|4^{i}} = \frac{\ell_i}{4}.
        \end{align*}
        Next, by the definition of $I_i$,
        \begin{align*}
            I_{i+1}-I_i
            \geq 3k\log|\Sigma|\cdot 2i+\log\tau.
        \end{align*}
        By the definition of $\Tcal_i$, we have
        \[
        \log(\Tcal_i) = \log\tau + (i-1)\left(3+\log k + k\log|\Sigma|\right)
        \leq \log\tau + i\left(6k\log|\Sigma|\right)
        \leq I_{i+1}-I_i.
        \]
        By the definition of $F_i$, it directly gives $F_i\geq b_{D,F}$.
        To prove the last inequality, it suffices to show that
        \[
        F_i\geq \log(1088)+2\log k + (k+1)\log \frac{1}{\ell_i} + 2\log \frac{1}{\eps}.
        \]
        We compute the right-hand side.
        \begin{align*}
            &\log(1088)+2\log k + (k+1)\log \frac{1}{\ell_i^{k+1}} + 2\log \frac{1}{\eps}\\
            \leq& 11 + 2\log k + (k+1)\left(\log|\Sigma| + 2(i-1)+0.02\right) + 2\log \frac{1}{\eps}\\
            \leq & 11 + 2\log k + (k+1)\left(\log|\Sigma| + 2\left(\frac{192k\log|\Sigma|}{\eps^2}\right)+0.02\right)+ 2\log \frac{1}{\eps}\\
            \leq & 772\left(\frac{k^2\log|\Sigma|}{\eps^2}+\log\frac{1}{\eps}\right)\\
            \leq&  F_i.
        \end{align*}
    \end{proof}

    By Lemma~\ref{lemma:rnn_selfboost_bitsize} and Claim~\ref{claim:proof_bitsize}, we know for any language model constructed with an RNN $Q$ with size $N_i$, hidden node set size $H_i$, RNN-time $\Tcal_i$, bit size $b_i=1+I_i+F_i$ and lower bound $\ell_i$, satisfying $I_i\geq b_{D,I}+k\log|\Sigma|+\log(k\tau)+1,F_i\geq b_{D,F},2^{-F_i}\leq \frac{\eps^2\ell^{k+1}_i}{1088k^2}$, if there exists a distinguisher RNN $D$ with size $|D|\leq \mathcal{d}$, RNN-time $\Tcal_D\leq \tau$ and bit size $\langle D\rangle \leq b_D$ whose advantage is $\eps$, then we can construct a language model $q'$ implemented by an RNN $Q'$ with size no more than $N_{i+1}$, hidden node set size no more than $H_{i+1}$, RNN-time no more than $\Tcal_{i+1}$, integer bit size no more than $I_{i+1}$, fractional bit size no more than $F_{i+1}$ and the output next-token probability lower bounded by $\ell'$, such that
    \[
    L(q')-L(q)\leq - \frac{\eps^2}{8k}.
    \]
    Then we apply Lemma~\ref{lemma:self_boosting_criticizer} by using the value function $h(q)$ as the bit size and the reciprocal of the next-token probability lower bound. Let the set $B_\eps = \{N_j \in \{N_i\}_{i\geq 1} \mid L_{j+1}<L_j-\eps^2/8k\}$. Then for any $\eps>0$, for all $N_i \in \{N_i\}_{i\geq 1}\backslash B_\eps$, every $\hat{q}$ that minimizes $L(q)$ over $C_j = \{q \mid  |q|\leq N_j, \Tcal(q)\leq \Tcal_j, h(q)\leq H_j, \langle q\rangle \leq b_j, \frac{1}{q}\leq \frac{1}{\ell_j}\}$ satisfies that for any next-$k$-token RNN distinguisher of size $|d|\leq \dcal$, RNN-time $\Tcal(d)\leq \tau$ and bit size $\langle d\rangle \leq b_D$, the distinguisher advange of $\hat{q}$, denoted as $\alpha$, satisfies
    \[
    \frac{\alpha^2}{8k}\leq \frac{\eps^2}{8k}.
    \]
    Equivalently, the advantage is at most $\eps$.
    The algorithm terminates with a network $N_i \notin B_\eps$. Consequently, the output of this algorithm is $\eps$-indistinguishable.

    Next, we will bound the size of $B_\eps$. Let $\bar{q}_0$ be the uniform distribution over the alphabet, which can be realized by an RNN of size $1$, hidden node set size $1$, RNN-time $1$ and integer bit size $0$ and fractional bit size $b_1$. Thus, for each $x\in \Sigma^*$ and $i\in\N$, $q_0(x_i\mid x_{:i})\geq \frac{1}{|\Sigma|}-2^{-F_1}$. Since $F_1\leq \log|\Sigma|+7$, we have $2^{-F_1}\leq \frac{0.01}{|\Sigma|}$, and $q_0(x_i\mid x_{:i})\geq \frac{0.99}{|\Sigma|}$. Then the KL divergence between $\bp$ and $\bq_0$ can be bounded by
    \[
    \KL(\bp \| \bq_0) = \Eop\limits_{x\sim \bp} \log \frac{\bp(x)}{\bq_0(x)}
     \leq \int \bp(x)\log \frac{\bp(x)}{\left(\frac{0.99}{|\Sigma|}\right)^n}
     \leq n\log \left(\frac{|\Sigma|}{0.99}\right)\leq 2n\log|\Sigma|.
    \]
    Because KL divergence is nonnegative, the size of $B_\eps$, or equivalently, the number of indices where the KL decreases by at least $\eps^2n/8k$ can be bounded by
    \[
    \frac{\KL(\bp \| \bq_0)}{\eps^2n/8k}\leq \frac{2n\log|\Sigma|}{\eps^2n/8k}
    =\frac{16k\log|\Sigma|}{\eps^2}.
    \]
    Since the index $j_0$ is uniformly chosen from $[16k\log|\Sigma|/\eps^2, 176k\log|\Sigma|/\eps^2]$, the probability that $N_j\in B_\eps$ is less than 0.1. Therefore, the expected number of trials is $1+1/0.9\approx 2.11<3.$

    Finally, we will bound the bit size of the output model. Suppose that the algorithm terminates with size $N_{i'}$. This implies that for each $i\in[j_0+1,i')$, we have $\KL(\bp \|\bq_i) - \KL(\bp \| \bq_{i+1})> \eps^2n/8k$. Summing over $i$, we obtain that $\KL(\bp \| \bq_{j_0+1}) - \KL(\bp \| \bq_i) > (i'-j_0-1)\eps^2n/8k$. Since $\KL(\bp \| \bq_{j_0+1})\leq \KL(\bp \|\bq_0)\leq 2n\log|\Sigma|$, we have
    \[
    i'-j_0-1 < \frac{\KL(\bp \| \bq_1)}{\eps^2n/8k}
    \leq \frac{2n\log|\Sigma|}{\eps^2n/8k}
    =\frac{16k\log|\Sigma|}{\eps^2}.
    \]
    So we know
    \[
    i'\leq j_0+\frac{16k\log|\Sigma|}{\eps^2}
    \leq \frac{192k\log|\Sigma|}{\eps^2}.
    \]
    Thus, the bit size of the final output LM is bounded by
    \begin{align*}
        b_{i'} 
        = &b_D +772\left(\frac{k^2}{\eps^2}\log|\Sigma|+\log\frac{1}{\eps}\right)+3k\log|\Sigma|{i'}^2+i'\log\tau\\
        =& O\left(
        b_D +\frac{k^2}{\eps^2}\log|\Sigma|+\log\frac{1}{\eps} 
        +k\log|\Sigma| \frac{k^2\log^2|\Sigma|}{\eps^4}
        +\frac{k}{\eps^2}\log|\Sigma|\log\tau 
        \right)\\
        =&O\left(
        b_D + \frac{k^3\log^2|\Sigma|}{\eps^4}
        +\frac{k}{\eps^2}\log|\Sigma|\log\tau
        \right)
    \end{align*}

\end{proof}

Finally, we provide the following lemmas, which are used in the proof of Lemma~\ref{lemma:rnn_selfboost_bitsize}.
\begin{lemma}[Additive Error composition for products]
\label{lemma:add_error_comp}
    Let $n\in\N$, $0<\delta<1/n$, and let $x_i,y_i \in [0,1]$ for $1\leq i\leq n$. Suppose that for any $1\leq i\leq n$, we have $|x_i-y_i|\leq \delta$, then we have 
    \[
    \left|
    \prod_{i=1}^n x_i - \prod_{i=1}^n y_i
    \right|
    \leq 2n\delta.
    \]
\end{lemma}
\begin{proof}
    On one hand, we have $x_i\leq y_i+\delta$. Thus,
    \begin{align*}
        \prod_{i=1}^n x_i
        \leq & \prod_{i=1}^n \left(x_i+\delta\right)\\
        = & \prod_{i=1}^n x_i
        +\sum_{\iota=1}^n \delta^\iota \sum_{1\leq i_1<\cdots<i_\iota\leq n}\prod_{j\notin\{i_1,\cdots,i_\iota\}}x_j\\
        \leq & \prod_{i=1}^n x_i+ \sum_{\iota=1}^n \binom{n}{\iota}\delta^\iota\\
        =& \prod_{i=1}^n x_i +\left(
        (1+\delta)^n-1
        \right)\\
        \leq & \prod_{i=1}^n x_i + e^{n\delta}-1
    \end{align*}
    Since $\delta < 1/n$, and for any $0<x<1$, we have $e^x \leq 1+x+x^2 < 1+2x$. Thus,
    \begin{align*}
        \prod_{i=1}^n x_i
        \leq  \prod_{i=1}^n y_i + 2n\delta.
    \end{align*}
    On the other hand, by swapping the roles of $x_i$ and $y_i$, we get
    \[
    \prod_{i=1}^n y_i
        \leq  \prod_{i=1}^n x_i + 2n\delta.
    \]
\end{proof}

\begin{lemma}
\label{lemma:error_fraction}
    For $x,y,\delta,\ell \in(0,1]$ with $y\geq x$ and $y\geq \ell > \delta$, we have
    \[
    \frac{x+\delta}{y-\delta} \leq \frac{x}{y}+ \frac{2\delta}{\ell -\delta}.
    \]
 \end{lemma}
 \begin{proof}
    By computation,
    \begin{align*}
        \frac{x+\delta}{y-\delta} - \frac{x}{y} 
        =& \frac{xy+\delta y-xy+\delta x}{y(y-\delta)}
        = \frac{\delta(x+y)}{y(y-\delta)}
    \end{align*}
    Since $x\leq y$, we have $(x+y)/y\leq 2$. Combining this with the condition that $y\geq \delta$, we have
    \[
    \frac{x+\delta}{y-\delta} - \frac{x}{y} 
    \leq \frac{2\delta}{y-\delta}
    \leq \frac{2\delta}{\ell-\delta}.
    \]
 \end{proof}

\begin{lemma}[Lower bound of the constructed language model.]
\label{lemma:lower_bound_q'}
    Let $\ell \in(0,1), i_0^*\in [0,k-1]$.
    Let $q:(\Sigma \cup \{\emptyset\})\times \mathcal{S}\rightarrow [0,1]$ be a language model such that for any $y\in \Sigma \cup \{\emptyset\}$ and $s\in\Scal$, $q(y\mid s)\geq \ell$. Define a new language model $q':(\Sigma \cup \{\emptyset\})\times \mathcal{S}\rightarrow [0,1]$ by  
    \begin{align}
        q'(x_i \mid x_{:i})=\begin{cases}
            \displaystyle q(x_i \mid x_{:i}) & \text{ if }i\leq i_0^*;\\
            \displaystyle \frac{
            \sum\limits_{s\in\Sigma^{k-r_0}}q(x_{i_0+1:i+1}\cdot s\mid x_{:i_0+1})e^{-\alpha d_{i_0+1}(x_{:i+1}\cdot s)}
            }{
            \sum\limits_{s\in\Sigma^{k-r_0+1}}q(x_{i_0+1:i}\cdot s\mid x_{:i_0+1})e^{-\alpha d_{i_0+1}(x_{:i}\cdot s)}
            }
            & \text{ if }i=i_0+r_0, i_0\in I, 1\leq r_0\leq k,
        \end{cases}
    \end{align}
    where $I:=\{ i\in [n] \mid i \equiv i_0^* \Mod{k} \}$. Then for any $y\in \Sigma \cup \{\emptyset\}$ and $s\in\Scal$, $q'(y\mid s)\geq \ell/3$.
\end{lemma}
\begin{proof}
    For $i \leq i_0^*$, $q'(x_i\mid x_{:i}) \geq \ell$. Next, we will show the result for $i>i_0^*$.
    Since for any $i\in\N$ and any $x\in \Scal$, the distinguisher $d_{i_0+1}(x)\in\{0,1\}$, thus,
    \begin{align*}
        q'(x_i \mid x_{:i})
        \geq 
        \frac{
        e^{-\alpha}\cdot 
        \sum\limits_{s\in\Sigma^{k-r_0}}q(x_{i_0+1:i+1}\cdot s\mid x_{:i_0+1})
        }{
        \sum\limits_{s\in\Sigma^{k-r_0+1}}q(x_{i_0+1:i}\cdot s\mid x_{:i_0+1})
        }
        =e^{-\alpha}\cdot q(x_{i+1}\mid x_{:i+1})
        > \ell/3
    \end{align*}
\end{proof}

\begin{lemma}[Loss improvement for quantized language models]
\label{lemma:bitsize_lossimprove_quantized}
    Let $n\in\N$ be the document length, $\ell \in (0,1),\delta \in (0,1)$.
    Let $\bp:\Sigma^n \rightarrow [0,1]$ be a text distribution. Let $q,\tilde{q}:(\Sigma \cup \{\emptyset\})\rightarrow[0,1]$ be two language models such that for any $x\in \Sigma^{n} ,i\in[n],$ 
    \[
    \left|
    q(x_i\mid x_{:i}) - \tilde{q}(x_i\mid x_{:i})
    \right|
    \leq \delta,
    \quad
    q(x_i\mid x_{:i}) \geq \ell.
    \]
    Then for any $\gamma\in (0,1)$,
    we have
    \[
    \Eop\limits_{x\sim \bp}\log \frac{\bar{q}(x)}{\bar{\tilde{q}}(x)}
    \leq n\cdot \frac{\delta}{\ell-\delta}.
    \]
\end{lemma}
\begin{proof}
    We write,
    \begin{align*}
        \Eop\limits_{x\sim \bp} \log\frac{\bar{q}(x)}{\bar{\tilde{q}}(x)}
        =&\int \bp(x) \log \frac{
        \prod_{i=1}^n q(x_i\mid x_{:i})
        }{
        \prod_{i=1}^n \tilde{q}(x_i\mid x_{:i})
        }\,dx\\
        =& \int \bp(x) \sum_{i=1}^n
        \left(
        \log \frac{
        q(x_i\mid x_{:i})
        }{
        \tilde{q}(x_i\mid x_{:i})
        }
        \right)
        \,dx\\
        \leq &
        \int \bp(x) \sum_{i=1}^n
        \left(
        \log \frac{
        q(x_i\mid x_{:i})
        }{
        q(x_i\mid x_{:i})-\delta
        }
        \right)
        \,dx\\
        \leq &
        \int \bp(x) \sum_{i=1}^n
        \left(
         \frac{
        \delta
        }{
        q(x_i\mid x_{:i})-\delta
        }
        \right)
        \,dx\\
        \leq &
        \int \bp(x) \sum_{i=1}^n
        \left(
         \frac{
        \delta
        }{
        \ell-\delta
        }
        \right)
        \,dx\\
        =& n\cdot \frac{\delta}{\ell-\delta}
    \end{align*}
\end{proof}

\paragraph{Acknowledgement. } We are deeply grateful to Adam Kalai for sharing his insightful ideas and suggestions and uncountably many hours of helpful and entertaining discussions. This work was supported by NSF award CCF-2106444, a Simons Investigator award and a JPMC AI PhD fellowship.

\bibliography{ref}

\appendix

\section{Indistinguishability and KL divergence}
\label{sec:appendix_indistinguish}

We show in this section that the advantage of any next-$k$ token distinguisher is upper-bounded by the square root of the KL divergence times $k/2n$. Intuitively, for each prefix $s$, each distinguisher $d_i$ can be viewed as an indicator of a subset $A_{i,s}\subseteq \Sigma^k$ of length-$k$ blocks. The advantage is the difference between the probabilities that $\bp$ and $\bq$ assign to $A_{i,s}$, which is at most the corresponding TV distance. Pinsker's Inequality then upper-bounds TV by a square root of KL.

We recall the relevant notation from Section~\ref{sec:prelim} and adapt it to a form more convenient for the subsequent proof. Let $p\in\Delta(\Sigma^n)$ be an LM that corresponds to the text distribution $\bp\in\bar{\Delta}(\Sigma^n)$. For a prefix $s\in\Sigma^i$ with $i\in[0,n]$, the \textit{marginal} probability under $\bp$ is
\[
\bp_i(s):=\sum_{z\in\Sigma^{n-i}}\bp(s\cdot z).
\]
Fix $k\in [1,n]$. For $s\in\Sigma^{\leq n-k}$ and $w\in |\Sigma|^k$,
the next-$k$ conditional distribution under $\bp$ is
\[
p^{(k)}(w \mid s) := 
\mathbb{P}_{x\sim \bp}\left(
x_{|s|+1:|s|+k+1}=w\mid x_{:|s|+1}=s
\right)
=
\frac{
\sum_{z \in \Sigma^{n-|s|-k}}p(s\cdot w\cdot z)
}{
\sum_{z \in \Sigma^{n-|s|}}p(s\cdot z)
}.
\]
Equivalently, in terms of next-token probability,
\[
p^{(k)}(w \mid s) = \prod_{t=1}^kp(w_t\mid s\cdot w_{:t}).
\]
\begin{defn}[Total Variation (TV) distance]
\label{defn:tv}
    The total variation distance between two distributions $p,q$ over the same space $\Xcal$ is defined as
    \[
    \TV(p,q)=\frac{1}{2}\sum_{x\in X}|p(x)-q(x)|
    =\sup_{A\subseteq \Xcal}|p(A)-q(A)|.
    \]
\end{defn}

Because each distinguisher $d_i(x)$ is binary and depends only on the prefix $x_{:i}$ and the next-$k$-token $x_{i:i+k}$, there is a binary function $\phi_i:|\Sigma|^i\times|\Sigma|^k \rightarrow \{0,1\}$ with 
$d_i(x)=\phi_i(x_{:i},x_{i:i+k})$. Then the advantage is
\begin{align}
\label{eq:appendix_adv_def_phi}
    a(d,\bp,\bq):=\Eop\limits_{y\sim \bp}\left[
    \frac{1}{n}\sum_{i=1}^n
    \left(
    \Eop\limits_{x\sim \bq}
    \left[
    \phi_i(x_{:i},x_{i:i+k}) \mid x_{:i}=y_{:i}
    \right]
    -\phi_i(y_{:i},y_{i:i+k})
    \right)
    \right].
\end{align}

We give the following theorem showing that the advantage of the distinguisher is upper-bounded by the square root of the KL divergence times $k/2n$.
\begin{thm}[Indistinguishability]
\label{thm:indistinguishability}
    For two text distributions $\bp,\bq \in \bar{\Delta}(\Sigma^n)$,  any next-$k$-token distinguisher $d:[n]\times\Sigma^n \rightarrow \{0,1\}$, its advantage 
    \[
    a(d,\bp,\bq)
    \leq \sqrt{
    \frac{k}{2n}\KL(\bp\|\bq)}.
    \]
\end{thm}

The proof of this theorem relies on the following lemma, which bounds the distinguisher’s advantage in terms of the average conditional KL divergence between the next-$k$-token distributions of $p$ and $q$.

\begin{lemma}
    \label{lemma:append_a_adv2kl}
    For two text distributions $\bp,\bq \in \bar{\Delta}(\Sigma^n)$,  any next-$k$-token distinguisher $d:[n]\times\Sigma^n \rightarrow \{0,1\}$, its advantage 
    \[
    a(d,\bp,\bq)
    \leq \sqrt{
    \frac{1}{2n}\sum_{i=0}^{n-1} \Eop_{s\sim \bp_i}
    \KL\left(
    p^{(k)}(\cdot \mid s)\| q^{(k)}(\cdot \mid s)
    \right)
    }.
    \]
\end{lemma}
\begin{proof}
    For any fixed $i\in[1,n]$ and prefix $s\in\Sigma^{i-1}$,
    \begin{align*}
        \Eop_{x\sim \bq}\left[\phi_i(x_{:i},x_{i:i+k})\mid x_{:i}=s\right]
        =& \sum_{w\in\Sigma^{k}}\phi_i(s,w)\cdot q^{(k)}(w\mid s).
    \end{align*}
    Thus,
    \begin{align}
    \label{eq:appendix_a_adv_first}
        \Eop_{y\sim \bp}\Eop_{x\sim \bq} \left[
        \phi_i(x_{:i},x_{i:i+k})\mid x_{:i}=y_{:i}
        \right]
        =&\Eop_{s\sim \bp_{i-1}}
        \left[\sum_{w\in\Sigma^{k}}\phi_i(s,w)\cdot q^{(k)}(w\mid s)\right].
    \end{align}
    For the second term in the advantage,
    \begin{align}
        \label{eq:appendix_a_adv_second}
        \Eop_{y\sim \bp}\left[\phi_i(y_{:i}, y_{i:i+k})\right]
        =\Eop_{s\sim \bp_{i-1}}\left[
        \sum_{w\in \Sigma^k} 
        \phi_i(s,w)\cdot
        p^{(k)}(w\mid s)
        \right].
    \end{align}
    Combining Equation~\eqref{eq:appendix_a_adv_first} and Equation~\eqref{eq:appendix_a_adv_second},
    we can rewrite the advantage as
    \begin{align}
        a(d,\bp,\bq)
        =&\frac{1}{n}\sum_{i=1}^{n}
        \Eop_{s\sim \bp_{i-1}}
        \left[
        \sum_{w\in\Sigma^k}\left( \phi_i(s,w)\cdot
        \left(
        q^{(k)}(w\mid s)-
        p^{(k)}(w\mid s)
        \right)
        \right)
        \right]\\
        =&\frac{1}{n}\sum_{i=0}^{n-1}
        \Eop_{s\sim \bp_{i}}
        \left[
        \sum_{w\in\Sigma^k}\left( \phi_i(s,w)\cdot
        \left(
        q^{(k)}(w\mid s)-
        p^{(k)}(w\mid s)
        \right)
        \right)
        \right]
        \label{eq:appendixa_adv2sum}.
    \end{align}
    Since $\phi_i$ is a binary function, for any fix $i\in[0,n-1]$ and prefix $s\in\Sigma^i$,
    \begin{align*}
        \sum_{w\in\Sigma^k} \left( \phi_i(s,w)\cdot
        \left(
        q^{(k)}(w\mid s)-
        p^{(k)}(w\mid s)
        \right)
        \right)
        \leq& \sup_{A\subseteq \Sigma^k}
        \sum_{w\in A}
        \left(
        q^{(k)}(w\mid s)-
        p^{(k)}(w\mid s)
        \right)\\
        \leq & \TV\left(q^{(k)}(\cdot \mid s), p^{(k)}(\cdot \mid s)\right)\\
        \leq & \sqrt{\frac{1}{2}\KL\left(q^{(k)}(\cdot \mid s), p^{(k)}(\cdot \mid s)\right)}
    \end{align*}
    The last step comes from Pinsker's Inequality.
    By Equation\eqref{eq:appendixa_adv2sum}, we can bound the advantage as
    \begin{align*}
        a(d,\bp,\bq)
        =&\frac{1}{n}\sum_{i=0}^{n-1}
        \Eop_{s\sim \bp_i}
        \left[
        \sum_{w\in\Sigma^k}\left( \phi_i(s,w)\cdot
        \left(
        q^{(k)}(w\mid s)-
        p^{(k)}(w\mid s)
        \right)
        \right)
        \right]\\
        \leq& \frac{1}{n}\sum_{i=0}^{n-1} \Eop_{s\sim \bp_i}
        \sqrt{\frac{1}{2}\KL\left(q^{(k)}(\cdot \mid s), p^{(k)}(\cdot \mid s)\right)}
    \end{align*}
    Since the function $f(z)=\sqrt{z}$ is concave, by Jensen's inequality, 
    \[
    \frac{1}{n}\sum_{i=0}^{n-1} \Eop_{s\sim \bp_i}
        \sqrt{\frac{1}{2}\KL\left(q^{(k)}(\cdot \mid s), p^{(k)}(\cdot \mid s)\right)}
    \leq \sqrt{\frac{1}{2n}\sum_{i=0}^{n-1} \Eop_{s\sim\bp_i}  
    \KL\left(q^{(k)}(\cdot \mid s), p^{(k)}(\cdot \mid s)\right)}.
    \]
\end{proof}

Based on Lemma~\ref{lemma:append_a_adv2kl}, we can prove Theorem~\ref{thm:indistinguishability} using the chain rule of KL divergence.
\begin{proof}[Proof of Theorem~\ref{thm:indistinguishability}]
    By Lemma~\ref{lemma:append_a_adv2kl}, the advantage 
    \begin{align}
    \label{eq:appendix_a_eq_lemma2}
        a(d,\bp,\bq)
    \leq \sqrt{
    \frac{1}{2n}\sum_{i=0}^{n-1} \Eop_{s\sim \bp_i}
    \KL\left(
    p^{(k)}(\cdot \mid s)\| q^{(k)}(\cdot \mid s)
    \right)
    }.
    \end{align}
    For any fixed prefix $s\in\Sigma^i$, by the definition of KL divergence,
    \begin{align*}
        \KL(p^{(k)}(\cdot \mid s) \| q^{(k)}(\cdot \mid s))
    =& \sum_{w\in\Sigma^k}p^{(k)}(w\mid s)\log \frac{p^{(k)}(w\mid s)}{p^{(k)}(w\mid s)}
    \\
    =& \sum_{w\in\Sigma^k}p^{(k)}(w\mid s)\sum_{t=1}^k\log \frac{p(w_t\mid sw_{:t})}{p(w\mid sw_{:t})}
    \\
    =& \sum_{t=1}^k\sum_{w\in\Sigma^k} p^{(k)}(w\mid s)\log \frac{p(w_t\mid sw_{:t})}{q(w_t\mid sw_{:t})}
    \end{align*}
    For each $t\in [1,k]$, we sum over $w_{t+1:}$ first to integrate them out, which gives
    \begin{align*}
        \sum_{w\in\Sigma^k}p^{(k)}(w\mid s)\log\frac{p(w_t\mid s\cdot w_{:t})}{q(w_t\mid s\cdot w_{:t})}
    =&
    \Eop_{u\sim p^{(t-1)}(\cdot \mid s)}\sum_{a\in\Sigma}p(a\mid s\cdot u)\log\frac{p(a\mid s\cdot u)}{q(a\mid s\cdot u)}\\
    =&\Eop_{u\sim p^{(t-1)}(\cdot \mid s)}\KL\left(
    p(\cdot \mid s\cdot u)\| q(\cdot \mid s\cdot u)
    \right)
    \end{align*}
    Thus, we have
    \begin{align}
        \label{eq:appendix_a_kl_cond}
        \KL(p^{(k)}(\cdot \mid s) \| q^{(k)}(\cdot \mid s))
        =\sum_{t=1}^{k}\Eop_{u\sim p^{(t-1)}(\cdot \mid s)}\KL\left(
        p(\cdot \mid s\cdot u)\| q(\cdot \mid s\cdot u)
        \right).
    \end{align}
    Taking $\Eop_{s\sim \bp_i}$ in Equation~\eqref{eq:appendix_a_kl_cond}, we have
    \begin{align}
        \Eop_{s\sim \bp_i}\KL(p^{(k)}(\cdot \mid s) \| q^{(k)}(\cdot \mid s))
        =&\sum_{t=1}^{k}\Eop_{s\sim \bp_i}\Eop_{u\sim p^{(t-1)}(\cdot \mid s)}\KL\left(
        p(\cdot \mid s\cdot u)\| q(\cdot \mid s\cdot u)
        \right)\\
        =&\sum_{t=1}^{k} \Eop_{x\sim \bp}
        \KL\left(
        p(\cdot \mid x_{:i+t})\| q(\cdot \mid x_{:i+t})
        \right)\label{eq:appendix_a_kl_it1}
    \end{align}
    Next we will sum over the prefix length $i$ and count how many times each position appears. Define
    \[
    D_j:=\Eop_{x\sim \bp} \KL\left(
        p(\cdot \mid x_{:j})\| q(\cdot \mid x_{:j})
        \right),\quad
        \text{for }j=1,\cdots,n.
    \]
    Summing up Equation \eqref{eq:appendix_a_kl_it1} over $i=0,1,2,\cdots,n-1$ gives
    \begin{align*}
        \sum_{i=0}^{n-1}\Eop_{s\sim \bp_i}\KL(p^{(k)}(\cdot \mid s) \| q^{(k)}(\cdot \mid s))
        =\sum_{i=0}^{n-1} \sum_{t=1}^{k}D_{i+t}
        =\sum_{j=1}^{n}c_jD_j,
    \end{align*}
    where $c_j:=|\{(i,t)\in[0,n-1]\times[1,k],i+t=j\}|$. We know for each $j$, $c_j\leq k$. Hence
    \[
    \sum_{i=0}^{n-1}\Eop_{s\sim \bp_i}\KL(p^{(k)}(\cdot \mid s) \| q^{(k)}(\cdot \mid s))
    \leq k\sum_{j=1}^{n}D_j
    =k\sum_{j=1}^{n} \Eop_{x\sim \bp} \KL\left(
        p(\cdot \mid x_{:j})\| q(\cdot \mid x_{:j})\right).
    \]
    By the chain rule,
    \[
    \KL(\bp\|\bq)
    =\sum_{i=1}^{n}\Eop_{x\sim \bp}\KL\left(
        p(\cdot \mid x_{:j})\| q(\cdot \mid x_{:j})\right).
    \]
    Combining with Equation\eqref{eq:appendix_a_eq_lemma2}, and we have
    \[
    a(d,\bp,\bq)\leq \sqrt{\frac{k}{2n} \KL(\bp\|\bq)}.
    \]

\end{proof}

\section{Computational Limitations of Autoregressive LLMs}
\label{sec:appendix_autoregressive}

In this section, we will assume that numbers are represented, as is common, in base 10. For simplicity, we further assume that each character in a string (including digits 0-9) is a single token, though the observation below can easily be extended to any constant-sized set of tokens.  

\begin{defn}[Factorization distribution]
    The \emph{Factorization Distribution} is the distribution over strings of length $n$ of strings of the form ``The prime factors of $m$ are $p_1 \times p_2$.'' where $p_1 \le p_2$ are random $n/4$-digit numbers and $m = p_1 \le p_2 \le \ldots p_i$ are its prime factorization. The string is padded with spaces to make its length $n$.
\end{defn}
We now state a hardness assumption for the average-case hardness of factorization, which is a common assumption in cryptography.
\begin{assumption}\label{assm:factor}
For every polynomial-time algorithm $A$, there is an $n_0$ such that for any $n\ge n_0$, given the product of two independent, uniformly random $n$-bit prime numbers, the probability that $A$ can output the prime factors with probability is no greater than 1\%.
\end{assumption}

\begin{defn}[Non-autoregressive language model]\label{def:lm_nonautoreg}
    A Non-Autoregressive Language Model $A$ is a randomized algorithm that outputs a document $\Sigma^n$ of length $n$, given an input $n$.
\end{defn}

\begin{observation}\label{obs:hardness}
    There is a Non-Autoregressive Language Model whose output distribution is exactly the Factorization Distribution for any $n$, which runs in time polynomial in $n$. For any ordinary (autoregressive) language model, there is an $n_0$ such that for all $n \ge n_0$, the statistical distance between its output distribution and that of the Factorization Distribution is $\ge 0.99$.
\end{observation}
\begin{proof}
It is possible to efficiently sample documents: simply multiply together two random $n/4$-digit random prime numbers (these can be sampled by simply repeatedly sampling random numbers and testing primality, which can be done in time in their length, until primes are found). However, if one could efficiently sample from the conditional next-token distribution, then one could clearly solve the factorization problem by completing the prompt ``The prime factors of $m$ are'' for the number $m$ to be factored. This would give an efficient algorithm for factorizing products of two random large primes more often than 99\% of the time.  
\end{proof}

\section{Common Transition Functions}
\label{subsec:tool_lemmas}
We show that the following common functions are transition functions.

\begin{lemma}[Transition Functions]
\label{lemma:transition_function}
Let $\eps>0$ be the machine precision. The following functions can be implemented as transition functions.
\begin{enumerate}
    \item \textbf{Indicator function.} For a constant $c \in \R$, and variable $x \in \R$,
    the indicator functions $\ind{x=c}, \ind{x\leq c}, \ind{x\geq c}$ are transition functions.
    
    \item \textbf{If-Else function.} For a constant $c\in \R$ and variables $b, v_1,v_2,\cdots,v_k \in R$, and let $f_1,f_2:\R^k\times \R$ be transition functions. Then the if-else function
    \[
    g(v_1,v_2,\cdots,v_k) = \begin{cases}
        f_1(v_1,v_2,\cdots,v_k) & \text{ if } b = c,\\
        f_2(v_1,v_2,\cdots,v_k) &  \text{ if } b \neq c
    \end{cases}
    \]
    is a transition function. The condition $b=c$ can be replaced by $b\leq c$ or $b\geq c$ to maintain a transition function.

    \item \textbf{Boolean logic operations: OR, AND, NOT.} For boolean variables $x_1,x_2, \cdots, x_k \in \{0,1\}$, the logical operations \textbf{OR},\textbf{ AND}, and \textbf{NOT} are transition functions. Specifically, the following three functions are transition functions.
    \begin{gather*}
        f_1(x_1,\cdots, x_k)=\bigvee\limits_{1\leq i\leq k}x_i,\\
        f_2(x_1,\cdots, x_k)=\bigwedge\limits_{1\leq i\leq k}x_i,\\
        f_3(x_i)=\neg x_i,\quad 1\leq i\leq k.
    \end{gather*}

    \item \textbf{Increment $k$-digit base-$c$ number.}
    Let $c,k \in N$.
    Let variables $x_1,x_2,\cdots, x_k \in [0,c-1]$. Then there exists $k$ transition functions $f_i:\{0,1,\cdots,c-1\}^k\rightarrow \{0,1,\cdots,c-1\}$ for $ 1\leq i\leq k$, such that
    \begin{align}
            \label{eq:count_rule_i}
            f_i(x_1,\cdots,x_k) = 
        \begin{cases}
            x_i & \exists j\leq i-1 \text{ s.t. } 0\leq x_j \leq c-2\\
            x_i + 1 & \text{if } x_j=c-1,\forall j\leq i-1,\text{ and } 0\leq x_i \leq c-2\\
            0& \text{if } x_j=c-1,\forall j\leq i
        \end{cases}
    \end{align}
    To interpret this operation, 
    define the vectors $x = (x_k,x_{k-1},\cdots,x_1)$ and $x'=(x_k',x_{k+1}',\cdots,x_1')$, where $x_i'=f_i(x_1,x_2,\cdots,x_k)$. Then $x'=x+1$ in base-$c$.

    \item \textbf{Exponential function on binary input.} Let $\alpha \in\R$, Then the exponential function $f(x) = \exp(\alpha x)$ is a transition function.
    
\end{enumerate}
    
\end{lemma}
\begin{proof}
\begin{enumerate}
    \item 
    Firstly for the indicator function $\ind{x^t=c}$, we have
    \begin{align*}
        &\frac{1}{\eps}  \sigma (\eps - \sigma(x-c) - \sigma(c-x))\\
        =& \begin{cases}
                \frac{1}{\eps} \sigma(1-(x-c))& \text{ if }x \geq c+\eps\\
                1  & \text{ if }x^t = c\\
                \frac{1}{\eps} \sigma(1-(c-x)) &\text{ if } x \leq c-\eps
            \end{cases}\\
            =& \begin{cases}
                1 & \text{ if } x=c\\
                0 & \text{ if } x\neq c
            \end{cases}
        \end{align*}
        Similarly, the indicator functions $\ind{x\leq c}$ and $\ind{x\geq c}$ can be computed as
        \begin{align*}
            \ind{x\leq c} =& \frac{1}{\eps} \sigma(\sigma(c+\eps-x)-\sigma(c-x))\\
            \ind{x\geq c} =& \frac{1}{\eps} \sigma( \sigma(x+\eps-c) - \sigma(x-c)) 
        \end{align*}

    \item 
    Since the indicator function $\ind{b=c}$ is a transition function, we have
    \[
    g(v_1,v_2,\cdots,v_k) = f_1(v_1,v_2,\cdots,v_k)\cdot \ind{b=c} + f_2(v_1,v_2,\cdots,v_k)\cdot (1-\ind{b=c}).
    \]
    This also holds by replacing $\ind{b=c}$ with $\ind{b\leq c}$ and $\ind{b\geq c}$ since they are transition functions as well.

    \item 
    We can construct the following transition functions.
    \begin{itemize}
        \item OR: Let
        \[
        f_1(x_1,\cdots,x_k) = \ind{
        \sum_{i=1}^k x_i \geq 0
        }
        = \bigvee\limits_{1\leq i\leq k}x_i.
        \]
        \item AND: Let
        \[
        f_1(x_1,\cdots,x_k) = \ind{
        \sum_{i=1}^k x_i \geq k
        }
        = \bigwedge\limits_{1\leq i\leq k}x_i.
        \]
        \item NOT: Let
        \[
        f_3(x_i) = 1-x_i = \neg x_i.
        \]
    \end{itemize}
    Since the indicator function is a transition function, the OR, AND,  NOT operations are also transition functions.

    \item 
    We construct $f_i$ for $1\leq i\leq k$ as follows.
    \begin{align}
        f_i(x_1,x_2,\cdots,x_j) =& \sigma\left(x_i+1-h_i^1+h_i^2-c h_i^3\right), \text{ where}\label{eq:count_rule_finalstep}\\
        h^1_i =& \sigma\left((i-1)(c-1)-\sum_{j=1}^{i-1} x_j\right),\label{eq:count_rule_h1}\\
        h^2_i =& \sigma\left(
        (i-1)(c-1)-1-\sum_{j=1}^{i-1}x_j
        \right),\label{eq:count_rule_h2}\\
        h^3_i =&\sigma \left(
        \sum_{j=1}^{i}x_j - i(c-1)+1
        \right).\label{eq:count_rule_h3}
    \end{align}

    We now verify that Equation~\eqref{eq:count_rule_finalstep} realizes the update rules Equation~\eqref{eq:count_rule_i}. There are three cases to consider.
    
    Firstly, if $\exists j \leq i-1$ s.t. $0\leq x_j \leq c-2$, then we have $\sum_{j=1}^{i-1}x_j \leq (i-1)(c-1)-1$. So by Equations~\eqref{eq:count_rule_h1},\eqref{eq:count_rule_h2}, we know $h^1_i=(i-1)(c-1)-\sum_{j=1}^{i-1}x_j, h^2_i = (i-1)(c-1)-1-\sum_{j=1}^{i-1}x_j$. Since $x_i\leq c-1$, then $\sum_{j=1}^i x_j \leq i(c-1)-1$. Thus, we have $h_3^i=0$ from Equation~\eqref{eq:count_rule_h3}. Combining with Equation~\eqref{eq:count_rule_finalstep}, $x_i'=\sigma(x_i+1-1)=x_i$.

    Secondly, when $x_j=c-1$ for $\forall j\leq i-1$, and $0\leq x_i\leq c-2$, we have $\sum_{j=1}^{i-1}x_j = (c-1)(i-1)$. So in this case, $h^1_i=h^2_i=0$. Furthermore, $\sum_{j=1}^i x_j \leq c-2+(i-1)(c-1) = i(c-1)-1$. So $h^3_i=0$ as well. By Equation~\eqref{eq:count_rule_finalstep}, $x_i'=\sigma(x_i+1)=x_i+1$.

    Thirdly, for the case when $x_j = c-1,\forall j\leq i$, similar to the second case, we know $h^1_i=h^2_i=0$. $h^3_i = \sigma((c-1)i-i(c-1)+1)=1$. So $x_i'=\sigma(c-1+1-c)=0$.

    Thus, the constructed functions realize Equation~\eqref{eq:count_rule_i} and are transition functions.

    \item Since the input is binary, the function $f(x)$ can be rewritten as
    \begin{align*}
        f(x)=\exp(\alpha x)
        =\begin{cases}
            1 &\text{ if }x=0\\
            e^{\alpha}&\text{ if }x=1
        \end{cases}
        =(1-e^{\alpha})\ind{x=0}+e^{\alpha}
    \end{align*}
    Since $\ind{x=0}$ is a transition function, $f(x)$ is also a transition function.

\end{enumerate}

\end{proof}

\end{document}